\documentclass[a4paper,11pt]{article}

\usepackage{xr}
\makeatletter

\newcommand*{\addFileDependency}[1]{
\typeout{(#1)}
\@addtofilelist{#1}
\IfFileExists{#1}{}{\typeout{No file #1.}}
}\makeatother

\usepackage{xr-hyper}
\usepackage{mattsstyle}
\usepackage{tcolorbox}
\usepackage{soul}
\usepackage{bm}
\usepackage{comment}
\usepackage{xfrac}
\usepackage{dirtytalk}

\usepackage{booktabs,tabularx,ragged2e}
\newcolumntype{Y}{>{\RaggedRight\arraybackslash}X}
\usepackage{adjustbox} 

\def\Lp#1{\mathrm{L}^{#1}}

\def\Ck#1{\mathrm{C}^{#1}}

\def\spaceBar{\, | \,}

\newcommand{\ub}{\mathbf{u}}

\DeclareMathOperator{\nn}{NN}

\definecolor{ETHBlue}{RGB}{0,0,0}

\def\commentOut#1{}

\allowdisplaybreaks

\usepackage{pifont}
\title{Multiple Neural Operators Achieve Near-Optimal Rates for Multi-Task Learning}
\author[1]{Adrien Weihs}
\author[1]{Hayden Schaeffer}

\affil[1]{Department of Mathematics,\protect\\ University of California Los Angeles,\protect\\ Los Angeles, CA 90095, USA. \vspace{\baselineskip}}

\date{}

\newcommand{\ScalingNetwork}{\mathrm{MNO}}

\newcommand{\ev}{\operatorname{ev}}
\newcommand{\complexity}{\mathcal{C}}

\begin{document}

\maketitle

\begin{abstract}
We study the approximation and statistical complexity of learning collections of operators in a shared multi-task setting, with a focus on the Multiple Neural Operators (MNO) architecture. For broad classes of Lipschitz multiple operator maps, we derive near-optimal upper bounds for approximation and statistical generalization. On the lower-bound side, we establish a curse of parametric complexity and prove corresponding minimax rates. Together, these results show that shared representations across tasks do not increase the overall cost: multi-task operator learning follows the same scaling laws as single operator learning. We also compare MNO with a multi-task extension of DeepONet based on concatenated task inputs and show that, from a worst-case approximation-complexity perspective, both architectures satisfy essentially the same asymptotic rates.

\end{abstract}



\section{Introduction} \label{sec:intro}

We study the problem of learning collections of operators in a shared multi-task setting, referred to as \emph{multi-operator learning} or \emph{multi-task operator learning} \cite{sun2025foundation,sun2025lemonlearninglearnmultioperator,weihs2025MOL,weihs2026generalizationboundsstatisticalguarantees}. The goal is to approximate multiple operator maps, that is, maps of the form
\begin{equation*}
    G: W \longrightarrow \{G[\alpha]:U\to V\}_{\alpha\in W},
\end{equation*}
within a single neural network model.
Here, \(W,U,V\) are function spaces, \(\alpha\in W\) identifies the operator or task, and for each fixed \(\alpha\), the corresponding map \(G[\alpha]\) sends an input function \(u\in U\) to an output in \(V\). This viewpoint extends the standard operator-learning problem \cite{KOVACHKI2024419,lu2022comprehensive}, in which one seeks to approximate a single operator between function spaces. It also arises naturally in a wide variety of applications, including parameterized kernel operators, solution operators of parameterized PDEs, and task-conditioned operator families or foundation models; see \cite[Section 2]{weihs2026generalizationboundsstatisticalguarantees} for a broader discussion. Learning such maps is particularly challenging for several reasons. First, multiple operator/multi-task learning is \say{effectively higher-dimensional} than standard operator learning. Second, the dependence on the operator descriptor \(\alpha\), the input function \(u\), and the evaluation variable \(x\) may all be highly nonlinear. Third, these variables play different roles in the problem, and this structure should be taken into account in the design of the learning mechanism. 

Given their expressivity and flexibility, neural networks are a natural approximation class in this setting. However, the associated architectural search problem remains a central challenge. Specifically, one seeks network structures that are empirically effective while also admitting theoretical guarantees. A natural way to quantify such guarantees is through their scaling laws. These describe how a target error metric, for instance approximation accuracy or generalization error, depends on a notion of complexity, such as architectural size, sparsity, or the number of training samples. In this sense, scaling laws relate achievable performance to the resources available for learning.
 
A recent architecture addressing both empirical performance and theoretical guarantees is the Multiple Neural Operators (MNO) architecture \cite{weihs2025MOL,weihs2026generalizationboundsstatisticalguarantees}. Concretely, 
for Lipschitz multiple operator maps \(G\), it was shown in \cite{weihs2025MOL} that for every target accuracy \(\varepsilon>0\), there exists a MNO with explicit \(\varepsilon\)-dependent bounds on the depth, width, sparsity, and parameter magnitude (see Table \ref{tab:mno-old-new-subcolumns}) satisfying
\(
\|\mathrm{MNO}-G\|\le \varepsilon.
\)
In particular, these constructive bounds imply an estimate of the form
\begin{equation*} 
    N_\#(\mathrm{MNO})\le \varepsilon^{-\varepsilon^{-\varepsilon^{-1/d}}},
\end{equation*}
where \(N_\#\) denotes the total number of nonzero parameters and \(d>0\) depends on the underlying function classes. These approximation-theoretic bounds were then used in \cite{weihs2026generalizationboundsstatisticalguarantees} to derive statistical learning rates. In particular, one obtains a generalization bound scaling as
\[
\left(
\frac{\log\log n_\alpha}{\log\log\log n_\alpha}
\right)^{-2/d},
\]
where \(n_\alpha\) denotes the number of sampled operators \(G[\alpha]\) used during training. To the best of our knowledge, these rates were the first of their kind for multiple operator learning. However, they also exhibit an additional exponential blow-up relative to the corresponding rates in standard operator learning (see Table \ref{tab:operator-learning-scaling-laws}), raising the possibility that multiple operator learning may be intrinsically more complex and might suffer from a specific \textit{curse of parametric complexity}.

In this paper, we demonstrate that the complexity of multi-task operator learning matches the qualitative exponential rates one expects for single-task operator learning. Specifically, our first main result is a substantial improvement of the previously known approximation and generalization scaling laws for multi-task learning using MNO. In particular, we show that the additional constructive blow-up is not intrinsic: one can derive complexity upper bounds for MNO that match the scale of the corresponding operator-learning rates. This immediately yields a stronger bound in the statistical learning rates. Our second main contribution is a lower-bound theory for multiple operator learning. We prove lower complexity bounds for broad classes of Lipschitz and differentiable multiple operator maps, showing that the upper bounds obtained in this work are close to sharp and that some form of parametric complexity barrier is indeed unavoidable. 
Lastly, we show that the same minimax scaling laws also apply to alternative architectural approaches to multiple operator learning. In this way, the paper clarifies which aspects of the previously observed complexity growth are artifacts of the constructive analysis and which reflect intrinsic barriers of the target class.

\subsection{Main Contributions}

Our main contributions are summarized as follows.
\begin{enumerate}
    \item \textbf{Near-optimal constructive approximation rates for MNO}
We derive near-optimal approximation-theoretic upper bounds for MNO on classes of Lipschitz multiple operator maps in Theorem \ref{thm:main:improvedRates}. The rates are obtained through a refined error analysis of the constructive approximation scheme. In particular, for every target accuracy \(\varepsilon>0\), we construct a MNO approximator with uniform error at most \(\varepsilon\) and with explicit \(\varepsilon\)-dependent bounds on the depth, width, sparsity, and parameter magnitude of its subnetworks. The precise scalings are given in Table \ref{tab:mno-old-new-subcolumns}. The resulting total approximation complexity satisfies
\[
N_\# \lesssim \exp\l d \log(\varepsilon^{-1})\,\varepsilon^{-\max\{d_W,d_U\}}\r,
\]
where $d_W, d_U$ denote the dimension of the domain of functions in $W, U$  respectively and $d > 0$ depends on $U$ and $W$. Inverting the relationship, we obtain the following scaling of error $\eps$ as a function of $N_\#$:
\[
\eps \lesssim
\left(
\frac{\log N_\#}{\log\log N_\#}
\right)^{-1/\max\{d_W,d_U\}}.
\]
These rates are of iterated exponential-type, qualitatively matching the known rates for operator learning \cite{liu2024neuralscalinglawsdeep,lanthalerStuart,SchwabStein,MarcatiSchwabPolytopes,herrmann,Lanthaler2022,marcati2023,lanthalerPCAnet}.

    \item \textbf{Refined statistical learning rates for MNO}
    We show that the refined approximation-theoretic bounds propagate directly to the statistical setting. More precisely, by combining the refined approximation construction with the generalization framework of \cite{weihs2026generalizationboundsstatisticalguarantees}, we obtain stronger learning rates for MNO on Lipschitz multiple operator maps in Theorem \ref{thm:scalingLawsGeneralizationError}. In particular, the resulting generalization bound scales as
    \[
    \left(
    \frac{\log n_\alpha}{\log\log n_\alpha}
    \right)^{-2/\max\{d_W,d_U\}},
    \]
    thus improving the previously known rates and reducing them to the corresponding operator-learning scale.

\item \textbf{Lower complexity bounds for multiple operator learning}
We extend the lower-complexity framework of \cite{lanthalerStuart} to the multiple operator setting. We introduce the notion of multiple operator maps of neural network type adapted to separable architectures such as MNO. As a consequence, in Theorem \ref{thm:codMNO} and Lemma \ref{lem:codMNO:symmetric}, we obtain lower complexity bounds for broad classes of Lipschitz and differentiable multiple operator maps, proving the same curse of parametric complexity as in operator learning for the multiple operator case. Specifically, we show that there exists a multiple operator map $G$ so that for any suitable neural network $\nn$ satisfying $\Vert \nn - G \Vert \leq \eps$, we have \[
N_\#(\nn) \geq \exp\l c\,\varepsilon^{-1/\eta}\r
\] 
where $\eta$ depends on $G$, $W$ and $U$.

\item \textbf{Minimax approximation-complexity rates for multiple operator learning} Adapting the lower bounds to the precise setting of the improved constructive upper bounds, in Theorem \ref{thm:minimaxMNO}, we derive minimax approximation-complexity rates for MNO on the Lipschitz class \(\mathcal H\). In particular, for the worst-case approximation complexity \(\mathfrak C(\varepsilon;\mathcal H)\), we obtain bounds of the form
\[
\exp\l c\,\varepsilon^{-1/\eta}\r
\lesssim
\mathfrak C(\varepsilon;\mathcal H)
\lesssim
\exp\l d \log(\varepsilon^{-1})\,\varepsilon^{-\max\{d_W,d_U\}}\r,
\]
showing that the upper bounds obtained in this work are close to sharp at the level of the overall exponential complexity regime.

    \item \textbf{Comparison with a concatenated DeepONet baseline.}
    We analyze an alternative multiple operator architecture obtained by concatenating the operator descriptor and the input function, and prove corresponding upper and lower complexity bounds for this model in Theorem \ref{thm:minimaxDeepONet}. In particular, we show that this concatenated DeepONet-type baseline obeys essentially the same minimax approximation-complexity scaling as MNO. Thus, from the viewpoint of worst-case approximation complexity, the present theory does not separate the two architectures, even though MNO exhibits markedly stronger empirical performance in previously reported experiments \cite{weihs2025MOL}. 
\end{enumerate}

\subsection{Related Works} \label{sec:related}

\paragraph{Multi-task and multiple operator learning}

There are at least two broad motivations for learning operator families rather than isolated operators. In some applications, the problem itself is naturally described by a collection of related operators, for instance through variations in physical parameters, geometry, boundary conditions, or governing equations. In other settings, learning several operators jointly is primarily a modeling strategy: by sharing structure across tasks, one may improve data efficiency, robustness, and transfer. These perspectives have motivated a growing recent literature on multiple operator learning and closely related frameworks; see, for example, \cite{sun2025foundation,liu2024prose,mccabe2023multiple,yang2023incontext,yang2023prompting,cao2024vicon,zhang2024modno,zhang2024d2no,liu2025bcat,ye2025pdeformer,zhang2025probabilistic,Jollie_2025,herde2024poseidon,bacho2025operatorlearningmachineprecision,weihs2025MOL,weihs2026generalizationboundsstatisticalguarantees,wang2026opinfllmparametricpdesolving}. 

At a coarse level, two modeling paradigms are common. One option is to train separate operator learners independently, one for each task or operator instance. Another is to regard the target as a parameterized family \(\{G[\alpha]\}\), where a discrete or continuous descriptor \(\alpha\) specifies the operator identity. The former viewpoint avoids introducing an explicit operator descriptor, but is correspondingly limited in its ability to exploit shared structure and to extrapolate to unseen operators. The latter instead augments operator learning with an explicit encoding of operator information \cite{sun2025foundation,liu2024prose,yang2023prompting,negrini2025multimodal,liu2024prosefd,weihs2025MOL,weihs2026generalizationboundsstatisticalguarantees}, such as a task label, symbolic expression, governing equation, or textual description. This conditioning mechanism has become central in recent work on PDE foundation models and task-adaptive operator learning, where it often improves transfer and enables zero-shot or out-of-distribution generalization without retraining.

\paragraph{Theoretical analyses of approximation and statistical generalization.}

A basic theoretical question in operator learning is expressivity: can a given architecture approximate large classes of operators, and under what assumptions? Early foundational work on neural network approximation of maps between spaces of scalar-valued functions was developed in \cite{ChenChen1993,ChenChen1995}. Since then, universal approximation and related expressivity results have been established for a range of operator-learning architectures, including DeepONet \cite{Lanthaler2022,liu2024neuralscalinglawsdeep}, the Fourier Neural Operator \cite{Kovachki2021}, and PCA-Net \cite{Bhattacharya}. Further developments concerning operator approximation, the effect of discretization, and architectural refinements may be found in \cite{mionet,CASTRO2023127413,castro2022,Huang2025,Kovachki2023,zhangBelnet,zhang2025discretization}.

A second line of work seeks to also quantify how approximation and learning performance scale with available resources. In this direction, scaling laws relate error to quantities such as model complexity, data size, and computational budget, and thereby provide a theoretical basis for performance prediction and generalization analysis \cite{kaplan2020scalinglawsneurallanguage}.  They offer a principled way to compare the efficiency of different architectures by quantifying the amount of complexity required to attain a prescribed error level. Also, when constructive upper bounds are compared with lower bounds, they help separate inefficiencies of a particular architecture or proof strategy from barriers that are intrinsic to the target class itself. This distinction is especially important in multiple operator learning, where complexity may arise both from the structure of the multiple operator class and from the chosen architecture.

In the standard operator-learning setting, both empirical and theoretical analyses of scaling behavior have received significant attention. On the empirical side, \cite{dehoop2022costaccuracytradeoffoperatorlearning} studies cost--accuracy trade-offs across neural operator architectures and highlights the role of both network size and sampling budget. On the theoretical side, approximation scaling laws and complexity estimates for deep ReLU networks, DeepONet, and related operator-learning architectures have been developed in \cite{liu2024neuralscalinglawsdeep,Lanthaler2022,lanthalerPCAnet,marcati2023,herrmann,lanthalerStuart,furuya2023globally,MarcatiSchwabPolytopes,SchwabStein}. The precise form of these scaling laws depends on several ingredients, including the target class of operators, the regularity and geometry of the input and output spaces, and the model class used for approximation. Existing theoretical results in operator learning can be organized, very broadly, into the categories shown in Table \ref{tab:operator-learning-scaling-laws}.

\begin{table}[H]
\centering
\small
\renewcommand{\arraystretch}{1.25}
\begin{tabularx}{\linewidth}{
>{\RaggedRight\arraybackslash}p{5cm}
>{\RaggedRight\arraybackslash}p{5cm}
>{\RaggedRight\arraybackslash}p{3.2cm}
Y}
\toprule
\textbf{Setting} & \textbf{Neural architecture} & \textbf{Complexity bound} & \textbf{Representative references} \\
\midrule
Lipschitz operators (upper bounds) 
& DeepONet
& $\exp \l c \log(\eps^{-1}) \eps^{-d} \r$ 
& \cite{liu2024neuralscalinglawsdeep} \\

Lipschitz/Differentiable operators (lower bounds) 
& DeepONet, FNO, PCANet, etc.
& $\exp \l c\varepsilon^{-d} \r$ 
& \cite{lanthalerStuart} \\

Lipschitz operators with non-standard architectures 
& DeepONet
& $\varepsilon^{-d}$ 
& \cite{SchwabStein} \\

Holomorphic operators 
& DeepONet
& substantially improved rates 
& \cite{MarcatiSchwabPolytopes,herrmann} \\

PDE operators (problem-specific analyses) 
& DeepONet, FNO, PCANet, etc.
& problem-dependent
& \cite{Lanthaler2022,lanthalerStuart,marcati2023,lanthalerPCAnet} \\
\bottomrule
\end{tabularx}
\caption{Schematic overview of representative scaling-law regimes in operator learning. The precise exponents, assumptions, and architecture classes vary across settings.}
\label{tab:operator-learning-scaling-laws}
\end{table}

Statistical guarantees have also become an active topic of study. For DeepONet-type models, \cite{liu2024neuralscalinglawsdeep} derives generalization bounds of the form
\[
\left(
\frac{\log n_u}{\log\log n_u}
\right)^{-1/d},
\]
for some \(d>0\), with \(n_u\) denoting the number of sampled input functions available during training. Closely related statistical rates for DeepONet and related models are obtained in \cite{liu2024,liu2024neuralscalinglawsdeep,benitez}. Complementary sample-complexity analyses for operator learning are developed in \cite{kovachki2024datacomplexityestimatesoperator,grohs2025theorytopracticegapneuralnetworks,adcock2025samplecomplexitylearninglipschitz}.
By contrast, the corresponding theory in multiple operator learning remains much less developed. Empirical evidence on multi-task and operator-family learning can be found in \cite{sun2025lemonlearninglearnmultioperator,Jollie_2025}. For the Multiple Neural Operators architecture, universal approximation results and the first explicit approximation scaling laws were established in \cite{weihs2025MOL}, and the corresponding statistical generalization rates were derived in \cite{weihs2026generalizationboundsstatisticalguarantees}. The present paper builds on this line by substantially improving these complexity estimates and by complementing them with lower bounds and minimax approximation-complexity rates. The theoretical framework transfers to other neural operator frameworks that are built from \cite{ChenChen1995}.

\paragraph{Organization of the paper.}
The remainder of the paper is organized as follows. In Section~\ref{sec:background}, we introduce the assumptions and the mathematical framework used throughout the paper. Section~\ref{sec:main} presents the main results of the paper. Proofs of all results are collected in the \nameref{sec:proofs}.

\section{Background} \label{sec:background}

In this section, we review the mathematical framework underlying our analysis. We begin by introducing the general notation and collecting the assumptions on the function spaces, target multiple operator maps, and product norms used throughout the paper. We then summarize the existing approximation and generalization scaling laws for MNO, which serve as a benchmark for our improved upper bounds. Finally, we recall the lower-complexity framework from operator learning on which our lower-bound analysis is based.

\subsection{General Notation}
Throughout the paper, we take \(\mathbb{N}=\{1,2,3,\dots\}\). For operators $T:U \to V$, we define the norm
\[
\|T\|_{\op} := \sup_{u\in U} \|T(u)\|_V.
\]
We write $\Omega_A$ for the domain of functions in the function set $A$. We denote the ball centered at $x$ with radius $\delta$ with respect to the norm $\Vert \cdot \Vert$ as $\cB_{\delta,\Vert \cdot \Vert}(x)$. When the norm is omitted, the norm is understood to be the $\ell^2$/Euclidean norm. When considering balls of functions mapping from $\Omega_A$ into $\bbR$, we write $\cB_{\delta,\Vert \cdot \Vert,\Omega_A}(x)$.
We denote by \(\ev\) the evaluation operator; for instance, \(\ev_x(f)=f(x)\) for a function \(f\) and a point \(x\) in its domain. 
For a neural network \(\Phi\), we write \(N_{\#}(\Phi)\) for the total number of nonzero parameters of \(\Phi\), that is, the total number of nonzero entries in all weight matrices and bias vectors. We use $\lesssim$ in scaling discussions to indicate the dominant asymptotic order, possibly suppressing multiplicative constants and lower-order terms. Lastly, we recall the definition of covering numbers. 

\begin{mydef}[Covering Number]
Let \( (X, d) \) be a metric space and let \( \eta > 0 \). A finite subset \( \mathcal{C} \subset X \) is called a \( \eta \)-cover of \( X \) if for every \( x \in X \), there exists \( c \in \mathcal{C} \) such that
\[
d(x, c) \leq \eta.
\]
The covering number of \( X \) at scale \( \eta \) with respect to the metric \( d \) is defined as
\[
\mathcal{N}(\eta, X, d) := \min \left\{ |\mathcal{C}| : \mathcal{C} \subset X \text{ is a } \eta\text{-cover of } X \right\}.
\]
\end{mydef}

\subsection{Assumptions}

\begin{assumptions} We make the following assumption on our spaces.
\begin{enumerate}[label=\textbf{S.\arabic*}]
\item The space $U(d_U,\gamma_U,L_U,\beta_U)$ is a function set such that \begin{enumerate}
    \item any function $u \in U$ is defined on $\Omega_U := [-\gamma_U,\gamma_U]^{d_U}$;
    \item for all functions $u \in U$ and $x,y \in \Omega_U$, we have \[
    \vert u(x) - u(y) \vert \leq L_U \vert x - y \vert;
    \]
    \item for all functions $u \in U$, we have $\Vert u \Vert_{\Lp{\infty}} \leq \beta_U$.
\end{enumerate}\label{assumption:Main:assumptions:S4}
\end{enumerate}
\end{assumptions}

The space defined in Assumption \ref{assumption:Main:assumptions:S4} yields a general approximation class for approximation theory.

\begin{assumptions} We make the following assumptions on our multiple operator map $G$. 
    \begin{enumerate}[label=\textbf{O.\arabic*}]
        \item  Assume that  $W(d_W,\gamma_W,L_W,\beta_W)$,
 $U(d_U,\gamma_U,L_U,\beta_U)$ and $V(d_V,\gamma_V,L_V,\beta_V)$ satisfy Assumption \ref{assumption:Main:assumptions:S4}. For $L_\mathcal{G} >1$, $r_\mathcal{G} \geq 1$, a multiple operator map is a function $G$ such that \begin{align*}
 	&G:\cB_{\beta_W,\Vert \cdot \Vert_{\Lp{\infty}},\Omega_W}(0) \mapsto \mathcal{G} 
    \end{align*}
    where $\mathcal{G} = \Big\{ G[\alpha]\spaceBar G[\alpha]:\cB_{\beta_U,\Vert \cdot \Vert_{\Lp{\infty}},\Omega_U}(0) \mapsto V \text{ and } \Vert G[\alpha][u_1] - G[\alpha][u_2]\Vert_{\Lp{\infty}(\Omega_V)} \leq L_{\mathcal{G}}\Vert u_1 - u_2 \Vert_{\Lp{r_\mathcal{G}}(\Omega_U)} \Big\}$ \label{assumption:Main:assumptions:O1}
    \item For $r_G \geq 1$, the multiple operator map $G$ satisfies \begin{equation*} 
   \Vert G[\alpha_1] - G[\alpha_2] \Vert_{\Lp{\infty}(\cB_{\beta_U,\Vert \cdot \Vert_{\Lp{\infty}},\Omega_U}(0) \times \Omega_V)} \leq L_G \Vert \alpha_1 - \alpha_2 \Vert_{\Lp{r_G}(\Omega_W)}  
\end{equation*}
for $\alpha_1,\alpha_2 \in \cB_{\beta_W,\Vert \cdot \Vert_{\Lp{\infty}},\Omega_W}(0)$. \label{assumption:Main:assumptions:O2}
 \end{enumerate}
\end{assumptions}

Assumptions \ref{assumption:Main:assumptions:O1} and \ref{assumption:Main:assumptions:O2} reflect minimal Lipschitz regularity assumptions on multiple operator maps required for our approximation-theoretical results.

\begin{assumptions} We make the following assumptions on our product norm.
    \begin{enumerate}[label=\textbf{N.\arabic*}]
        \item For Banach spaces $X$ and $Y$, we equip $X \times Y$ with a norm $\Vert \cdot \Vert_{X \times Y}$ satisfying $\Vert (x,y) \Vert_{X\times Y} \geq C_{\textrm{prod}} \max\{\Vert x \Vert_X, \Vert y \Vert_Y\}$ for some $C_{\textrm{prod}} > 0$. \label{assumption:Main:assumptions:N1}
        \item For Banach spaces $X$ and $Y$, we equip $X \times Y$ with a norm $\Vert \cdot \Vert_{X \times Y}$ satisfying $\Vert (x,y) \Vert_{X\times Y} \leq C_{\textrm{prod}} \max\{\Vert x \Vert_X, \Vert y \Vert_Y\}$ for some $C_{\textrm{prod}} > 0$. \label{assumption:Main:assumptions:N2}
        \item For Banach spaces $X$ and $Y$, we equip $X \times Y$ with a norm $\Vert \cdot \Vert_{X \times Y}$ satisfying $\Vert (x,0) \Vert_{X\times Y} = \Vert x \Vert_{X}$ and $\Vert (0,y) \Vert_{X\times Y} = \Vert y \Vert_{Y}$. \label{assumption:Main:assumptions:N3}
    \end{enumerate}
\end{assumptions}

The following product Banach space norms satisfy all Assumptions \ref{assumption:Main:assumptions:N1}, \ref{assumption:Main:assumptions:N2} and \ref{assumption:Main:assumptions:N3}: $\Vert (x,y) \Vert_{X\times Y} = (\Vert x \Vert_X^p + \Vert y \Vert_Y^p)^{1/p}$ with $p \geq 1$ or $\Vert (x,y) \Vert_{X\times Y} = \max\{\Vert x \Vert_X, \Vert y \Vert_Y\}$.

\subsection{Complexity Scaling Laws for Multiple Operator Learning}

We begin by introducing the class of neural networks used throughout our analysis. This class covers a wide range of architectures, is widely used \cite{liu2024neuralscalinglawsdeep,LIU2025101717,liu20252,zhangCoefficient,liu2024}, and will serve as the basic building block for all of our constructions.

\begin{mydef}[Feedforward ReLU Network Class] \label{def:networkClass}
Let \( q : \mathbb{R}^{d_1} \to \mathbb{R} \) be a feedforward ReLU network defined as
\begin{equation} \label{eq:back:reluNN}
   q(x) = W_L \cdot \mathrm{ReLU}\left(W_{L-1} \cdots \mathrm{ReLU}(W_1 x + b_1) + \cdots + b_{L-1} \right) + b_L, 
\end{equation}

where \( W_\ell \) are weight matrices, \( b_\ell \) are bias vectors, and \( \mathrm{ReLU}(a) = \max\{a, 0\} \) is applied element-wise.

We define the class of such feedforward networks with ReLU activations:
\[
\cF_{\rm NN}(d_1, d_2, L, p, K, \kappa, R) = \left\{ [q_1, q_2, \dots, q_{d_2}]^\top \in \mathbb{R}^{d_2} \; \middle| \;
\begin{array}{l}
\text{each } q_k : \mathbb{R}^{d_1} \to \mathbb{R} \text{ has the above form with} \\
L \text{ layers, width bounded by } p, \\
\|q_k\|_{\Lp{\infty}} \leq R, \quad \|W_\ell\|_{\infty,\infty} \leq \kappa, \quad \|b_\ell\|_\infty \leq \kappa, \\
\sum_{\ell=1}^L \left( \|W_\ell\|_0 + \|b_\ell\|_0 \right) \leq K
\end{array}
\right\},
\]
where
\begin{itemize}
    \item \( \|q\|_{\Lp{\infty}} = \sup_{x \in \Omega} |q(x)| \),
    \item \( \|W_\ell\|_{\infty,\infty} = \max_{i,j} |[W_\ell]_{ij}| \),
    \item \( \|b_\ell\|_\infty = \max_i |[b_\ell]_i| \),
    \item \( \|\cdot\|_0 \) denotes the number of nonzero elements.
\end{itemize}
This network class consists of vector-valued functions with input dimension \( d_1 \), output dimension \( d_2 \), depth \( L \), width at most \( p \), at most \( K \) nonzero parameters, all bounded in magnitude by \( \kappa \), and uniformly bounded output norm by \( R \).
\end{mydef}

To approximate multiple operator maps, the following \(\ScalingNetwork\) architecture was introduced in \cite{weihs2025MOL}. This architecture has proven effective in practice while remaining mathematically tractable.

\begin{mydef}[$\ScalingNetwork$ Architecture]
Let \(P, H \in\mathbb N\). We define a $\ScalingNetwork$ to be a map of the form
\[
\mathrm{MNO}[\alpha][u](x)
=
\sum_{p=1}^{P}\sum_{k=1}^{H}
l_p(M_W(\alpha))\, b_{pk}(M_U(u))\, \tau_{pk}(x),
\]
where \(\alpha\in W\), \(u\in U\), \(x\in\Omega_V\), the functions \(l_p\), \(b_{pk}\), and \(\tau_{pk}\) are chosen from suitable neural-network classes \(\cF_{\rm NN}\), and $M_W:W \mapsto \bbR^m$ and $M_U:U \mapsto \bbR^{q}$ are linear maps. 
\end{mydef}

The architecture above fits into the broader fully separable family
\begin{equation} \label{eq:separableArchitecture}
    \sum_{p=1}^{P}\sum_{k=1}^{H}\sum_{\ell=1}^{N}
\theta_{pk\ell}\, l_p(M_W(\alpha))\, b_k(M_U(u))\, \tau_\ell(x),
\qquad \theta_{pk\ell}\in\mathbb R.
\end{equation}
This more general form is convenient for analysis, and the original MNO architecture is recovered by a direct relabelling of indices or by grouping terms into suitably chosen subnetworks.

A first fundamental question is the expressive capacity of the \(\ScalingNetwork\) architecture. In \cite{weihs2025MOL}, both a universal approximation theorem for continuous multiple operator maps and quantitative scaling laws for Lipschitz multiple operator maps satisfying Assumptions \ref{assumption:Main:assumptions:O1} and \ref{assumption:Main:assumptions:O2} were established. In particular, for the latter class, \cite[Theorem 3.16]{weihs2025MOL} shows that for every target accuracy \(\varepsilon>0\), there exists a network \(\nn\) of the form \eqref{eq:separableArchitecture} such that
\[
\sup_{\alpha\in W}\sup_{u\in U}\sup_{x\in\Omega_V}
\left|
G[\alpha][u](x)-\nn[\alpha][u](x)
\right|
\le \varepsilon,
\]
with architectural scalings summarized in Table \ref{tab:mno-old-new-subcolumns}. These scalings imply the following total parameter count and, equivalently, the following approximation rate expressed in terms of the total number of nonzero parameters \(N_\#\):
\[
N_\# \lesssim \eps^{-\eps^{-\eps^{-d_W}}}
\qquad \Longleftrightarrow \qquad
\eps \lesssim \left( \frac{\log \log N_\#}{\log \log \log N_\#} \right)^{-1/d_W}.
\]
In Section \ref{sec:improvedResults}, we show that a different organization of the approximation argument leads to substantially improved rates.

\subsection{Generalization Scaling Laws for Multiple Operator Learning}

Beyond approximation capabilities, a second fundamental question in learning concerns generalization, namely, how well an algorithm trained on a given dataset performs on previously unseen samples. We next introduce the learning setup in \cite{weihs2026generalizationboundsstatisticalguarantees} that is used to train a network of the form \eqref{eq:separableArchitecture}.

Given a function \(f:\mathbb{R}\to\mathbb{R}\) and a constant \(a\geq 0\), we denote by
\[
\operatorname{Clip}_a(f):=\min\{\max(f,-a),a\}
\]
the clipping operation onto the interval \([-a,a]\). This map can be realized using only affine transformations and ReLU activations. In particular, it belongs to the class \(\cF_{\rm NN}(1,1,2,1,6,2a,a)\). 

For the study of generalization, it is convenient to work with a clipped version of the network class \eqref{eq:separableArchitecture}. This does not impose a significant restriction in our setting: indeed, the approximation results recalled above produce approximating networks with uniformly bounded outputs, so clipping at a sufficiently large level does not alter the constructions. Moreover, in practical implementations, network outputs are typically bounded as well, either explicitly or through the scale of the training data and targets.

\begin{mydef}[Clipped Multiple Operator Network Class] 
    Let $\cF_i$ for $1 \leq i \leq 3$ be network classes defined in Definition \ref{def:networkClass}. For $a,I > 0$, $P,H,N \in \bbN$ and fixed sampling points $\{y_s\}_{s=1}^{n_{c_W}} \subset \Omega_W$ and $\{c_s\}_{s=1}^{n_{c_U}}$, we define $ \mathrm{Cl}_a(I,\cF_1,\cF_2,\cF_3,\{y_s\},\{c_s\},P,H,N)$, the set of $a$-clipped multiple operator networks, as \begin{align}
    &\mathrm{Cl}_a(I,\cF_1,\cF_2,\cF_3,\{y_s\},\{c_s\},P,H,N) \notag \\
    &= \left\{ \operatorname{Clip}_a\l \sum_{p=1}^{P} \sum_{k=1}^H \sum_{\ell=1}^{N} \theta_{pk\ell} l_p(\bm{\alpha}) b_k(\mathbf{u}) \tau_\ell(x)\r \mid \theta_{pk\ell} \in [-I,I],\, \tau_\ell \in \cF_1,\, b_k \in \cF_2, \, l_p \in \cF_3  \right\} \notag
    \end{align}
    where $\bm{\alpha}=(\alpha(y_1), \alpha(y_2),...,\alpha(y_{n_{c_W}}))^\top$ and $\ub=(u(c_1), u(c_2),...,u(c_{n_{c_U}}))^\top$.
\end{mydef}

The following definition specifies the probabilistic framework used to model the training data. In the multiple operator setting, observations are generated through a hierarchical sampling procedure: one first draws parameter instances \(\alpha\), then draws input functions \(u\), and finally samples evaluation points \(x\) where noisy measurements of the output \(G[\alpha][u]\) are taken.

\begin{mydef}[Training Set] \label{def:trainingSet}
Let $G : W \mapsto \{G[\alpha] : U \to V\}$ be a multiple operator map. Let $\mu_\alpha$ be a probability measure on $W$, $\mu_u$ a probability measure on $U$, and $\mu_x$ a probability measure on $\Omega_V$. Given fixed sampling points $\{y_s\}_{s=1}^{n_{c_W}} \subset \Omega_W$ and $\{c_s\}_{s=1}^{n_{c_U}} \subset \Omega_U$, we define the training set:
\[
S_{G, \{y_s\}, \{c_s\}} =
\left\{
\boldsymbol{\alpha}_\ell,
\left\{
\boldsymbol{u}_{\ell i},
\left\{ (x_{\ell ij}, w_{\ell ij}) \right\}_{j=1}^{n_x}
\right\}_{i=1}^{n_u}
\right\}_{\ell = 1}^{n_\alpha}
\]
where
\begin{itemize}
    \item \( \alpha_\ell \iid \mu_\alpha \) and \( \boldsymbol{\alpha}_\ell = \big(\alpha_\ell(y_1), \dots, \alpha_\ell(y_{n_{c_W}})\big) \in \mathbb{R}^{n_{c_W}} \);
    \item \( u_{\ell i} \iid \mu_u \) and \( \boldsymbol{u}_{\ell i} = \big(u_{\ell i}(c_1), \dots, u_{\ell i}(c_{n_{c_U}})\big) \in \mathbb{R}^{n_{c_U}} \);
    \item \( x_{\ell ij} \iid \mu_x \) drawn from \( \Omega_V \);
    \item \( w_{\ell ij} = G[\alpha_\ell][u_{\ell i}](x_{\ell ij}) + \zeta_{\ell ij} \), where \( \zeta_{\ell ij} \) are i.i.d. sub-Gaussian noise variables with mean 0 and variance proxy \( \sigma^2 \). 
\end{itemize}
\end{mydef}

For a given training set \(S_{G,\{y_s\},\{c_s\}}\), we introduce the corresponding learned operator below. More precisely, this is a neural network selected from the class \(\mathrm{Cl}_a(I,\cF_1,\cF_2,\cF_3,\{y_s\},\{c_s\},P,H,N)\) by minimizing the empirical \(\Lp{2}\)-loss over the observed data.

\begin{mydef}[Trained Operator] 
    Let $\cF_i$ for $1 \leq i \leq 3$ be network classes defined in Definition \ref{def:networkClass}. Let $G : W \mapsto \{G[\alpha] : U \to V\}$ be a multiple operator map. Let $\mu_\alpha$ be a probability measure on $W$, $\mu_u$ a probability measure on $U$, and $\mu_x$ a probability measure on $\Omega_V$. Given fixed sampling points $\{y_s\}_{s=1}^{n_{c_W}} \subset \Omega_W$ and $\{c_s\}_{s=1}^{n_{c_U}} \subset \Omega_U$, let $S_{G, \{y_s\}, \{c_s\}}$ be the training set defined in Definition \ref{def:trainingSet}. For $a,I >0$, the trained $a$-clipped operator $G_{a,I,\cF_1,\cF_2,\cF_3,S}$ is defined as \[
    G_{a,I,\cF_1,\cF_2,\cF_3,S} = \argmin_{\nn \in \mathrm{Cl}_a(I, \cF_1,\cF_2,\cF_3,\{y_s\}, \{c_s\}, P,H,N)} \frac{1}{n_\alpha n_u n_x} \sum_{\ell=1}^{n_\alpha} \sum_{i=1}^{n_u} \sum_{j = 1}^{n_x} \l \nn[\bm{\alpha}_\ell][\bm{u}_{\ell i}](x_{\ell ij}) - w_{\ell i j}  \r^2.
    \]
\end{mydef}

Next, we define the expected generalization error associated with the learned operator. This quantity evaluates the performance of \(G_{a,I,\cF_1,\cF_2,\cF_3,S}\) after averaging both over the randomness of the training set \(S\) and over previously unseen test samples \((\alpha,u,x)\). In this way, it captures how well a model trained on one realization of the data transfers to new observations.

\begin{mydef}[Expected Generalization Error]
Let $\cF_i$ for $1 \leq i \leq 3$ be network classes defined in Definition \ref{def:networkClass}. Let $G : W \mapsto \{G[\alpha] : U \to V\}$ be a map. Let $\mu_\alpha$ be a probability measure on $W$, $\mu_u$ a probability measure on $U$, and $\mu_x$ a probability measure on $\Omega_V$. Let $\{y_s\}_{s=1}^{n_{c_W}} \subset \Omega_W$ and $\{c_s\}_{s=1}^{n_{c_U}} \subset \Omega_U$ be fixed sampling points. We define the expected generalization error as
\[
\mathbb{E}_{S_{G, \{y_s\}, \{c_s\}}} \underbrace{\mathbb{E}_{\alpha \sim \mu_\alpha} \, \mathbb{E}_{u \sim \mu_u} \, \mathbb{E}_{\{x_j\}_{j=1}^{n_x} \sim \mu_x^{\otimes n_x}}}_{\textnormal{test sampling}} \underbrace{\left[
\frac{1}{n_x} \sum_{j=1}^{n_x} \left( G_{a,I,\cF_1,\cF_2,\cF_3,S}[\bm{\alpha}][\ub](x_j) - G[\alpha][u](x_j) \right)^2
\right]}_{\textnormal{empirical approximation of the squared $\Lp{2}(\mu_x)$ error}},
\]
where \(\mathbb{E}_{S_{G, \{y_s\}, \{c_s\}}}\) denotes the expectation over the full training dataset \(S_{G, \{y_s\}, \{c_s\}}\). 

\end{mydef}

We now state the main scaling law for the expected generalization error \cite[Theorem 3.5 and Corollary 3.8]{weihs2026generalizationboundsstatisticalguarantees} for MNOs approximating multiple operator maps satisfying Assumptions \ref{assumption:Main:assumptions:O1} and \ref{assumption:Main:assumptions:O2}. To this end, let \(\varepsilon>0\) be a target approximation accuracy and consider the hypothesis class
\[
\mathrm{Cl}_a(I,\cF_1,\cF_2,\cF_3,\{y_s\},\{c_s\},P,H,N),
\]
where the architectural parameters of $\cF_i$ are chosen according to the \(\varepsilon\)-dependent scalings provided by the approximation theory (i.e. \cite[Theorem 3.16]{weihs2025MOL}), so that the class contains an \(\varepsilon\)-accurate approximation of \(G\). Also, let \(\eta>0\) be the covering radius used to estimate the size of this class through the metric entropy
\[
\log \mathcal N\bigl(\eta,\mathrm{Cl}_a(I,\cF_1,\cF_2,\cF_3,\{y_s\},\{c_s\},P,H,N),\|\cdot\|_{\Lp{\infty}(W\times U\times\Omega_V)}\bigr).
\]
Then, we obtain the following approximation--estimation tradeoff, relating the target accuracy, the architectural complexity, and the sampling budgets \((n_\alpha,n_u,n_x)\):
\begin{align}
    &\mathbb{E}_{S_{G, \{y_s\}, \{c_s\}}} \mathbb{E}_{\alpha \sim \mu_\alpha} \, \mathbb{E}_{u \sim \mu_u} \, \mathbb{E}_{\{x_j\}_{j=1}^{n_x} \sim \mu_x^{\otimes n_x}} \left[
\frac{1}{n_x} \sum_{j=1}^{n_x} \left( G_{a,I,\cF_1,\cF_2,\cF_3,S}[\bm{\alpha}][\ub](x_j) - G[\alpha][u](x_j) \right)^2
\right] \notag  \\
        &\leq 4 \eps^2 + \eta (8 \sigma + 6) \notag \\
    &+ \frac{8\sigma \eta}{\sqrt{n_\alpha n_u n_x}} \sqrt{\log \l \cN\l\eta,\mathrm{Cl}_a(I,\cF_1,\cF_2,\cF_3,\{y_s\},\{c_s\},P,H,N),\Vert \cdot \Vert_{\Lp{\infty}(W \times U \times \Omega_V)}\r \r + \log(2)} \notag \\
    &+ \frac{16\sigma^2}{n_\alpha n_u n_x} \l \log \l \cN\l\eta,\mathrm{Cl}_a(I,\cF_1,\cF_2,\cF_3,\{y_s\},\{c_s\},P,H,N),\Vert \cdot \Vert_{\Lp{\infty}(W \times U \times \Omega_V)}\r \r + \log(2) \r \notag \\
    &+ \frac{112 \beta_V^2}{3 n_\alpha} \log\l   \cN\l\eta/(4\beta_V),\mathrm{Cl}_a(I,\cF_1,\cF_2,\cF_3,\{y_s\},\{c_s\},P,H,N),\Vert \cdot \Vert_{\Lp{\infty}(W \times U \times \Omega_V)}\r \r \notag 
    \end{align}
where $\bm{\alpha}=(\alpha(y_1), \alpha(y_2),...,\alpha(y_{n_{c_W}}))^\top$ and $\ub=(u(c_1), u(c_2),...,u(c_{n_{c_U}}))^\top$. Estimating the metric entropy as a function of $\eps$, picking 
\[
\varepsilon \asymp 
    \left(
        \frac{\log\log n_\alpha}{\log\log\log n_\alpha}
    \right)^{-\frac{1}{d_W}} \qquad \text{and} \qquad \eta = 4\beta_V n_\alpha^{-1},
    \] 
the above reduces to the learning rate:
\begin{align}
        &\mathbb{E}_{S_{G, \{y_s\}, \{c_s\}}} \mathbb{E}_{\alpha \sim \mu_\alpha} \, \mathbb{E}_{u \sim \mu_u} \, \mathbb{E}_{\{x_j\}_{j=1}^{n_x} \sim \mu_x^{\otimes n_x}} \left[
\frac{1}{n_x} \sum_{j=1}^{n_x} \left( G_{a,I,\cF_1,\cF_2,\cF_3,S}[\bm{\alpha}][\ub](x_j) - G[\alpha][u](x_j) \right)^2
\right] \notag \\
&= \mathcal{O} \l \left(
        \frac{\log\log n_\alpha}{\log\log\log n_\alpha}
    \right)^{-\frac{2}{d_W}} \r. \notag
    \end{align}
In Section \ref{sec:improvedResults}, we show that our improved approximation-theoretic rates lead, in the same way, to a corresponding improvement in the learning rate, reducing it to the rate dictated by operator learning.

\subsection{Lower Complexity Bounds in Operator Learning} \label{sec:back:lowerBound}

We now recall the operator-learning lower-complexity framework from \cite{lanthalerStuart}, which will serve as the starting point for our lower-bound analysis in the multiple operator setting. 

We begin by introducing the notion of an infinite-dimensional cube \cite{lanthalerStuart,Donoho,DahlkeDeMariGrohsLabate2015}. These cubes provide a way to describe intrinsic geometric features of compact sets in infinite-dimensional spaces, and thereby furnish a quantitative measure of approximation-theoretic complexity. They also arise naturally in many function classes of interest: \cite[Lemma 2.7]{lanthalerStuart} establishes their existence for suitable compact classes of differentiable functions, while Lemma \ref{lem:boundedLipCube} proves the same for the bounded Lipschitz classes considered in this paper.

\begin{mydef}[Infinite-dimensional Cube] \label{def:cube}
Let $X$ be a Banach space and let $e_1,e_2,\dots \in X$ be a sequence of linearly independent elements so that
\(
\|e_j\|_X = 1
\)
for all $j\in\mathbb \bbN$.
Given constants $A>0$ and $\eta>1$, we say that a set $K\subset X$ contains an infinite-dimensional cube
\(
Q_\eta = Q_\eta\l A;\{e_j\}_{j \in \bbN} \r
\)
if the following two conditions hold:
\begin{enumerate}
\item $K$ contains every element of the form
\[
u = A\sum_{j=1}^\infty j^{-\eta} y_j e_j,
\qquad y_j\in[0,1]
\quad \text{for all } j\in\mathbb N.
\]

\item The family $\{e_j\}_{j\in\mathbb N}$ admits a bounded biorthogonal sequence in the dual space $X^*$, that is, there exist functionals
\(
e_1^*,e_2^*,\dots \in X^*
\)
such that
\[
e_k^*(e_j)=\delta_{jk}
\qquad \text{for all } j,k\in\mathbb N, \quad \text{and} \quad \|e_k^*\|_{X^*}\le M
\qquad \text{for all } k\in\mathbb N
\]
for some constant $M>0$.
\end{enumerate}
\end{mydef}

A cube as in Definition \ref{def:cube} may be viewed as an infinite-dimensional coordinate system inside \(K\), with coefficients \(y_j\in[0,1]\) that are damped by a decay rate \(j^{-\eta}\). The associated bounded biorthogonal sequence guarantees that these coefficients can be recovered from the embedded element, so that the cube genuinely carries a coordinate structure. The exponent \(\eta\) measures the asymptotic size of this embedded cube, and will directly control the lower-complexity bounds stated below.

The next step is to specify the class of neural approximants to which the lower-bound framework applies. This is captured by the notions of functionals and operators of neural network type, which combine a finite-dimensional linear encoding with a ReLU network. Many operator-learning architectures encountered in practice (e.g. DeepONet \cite{deepOnet}, PCANet \cite{Bhattacharya}) fall within this framework, in the sense that they can be interpreted as operators of neural network type after a suitable finite-dimensional discretization of the input. We begin by introducing functionals of neural network type. 

\begin{mydef}[Functional of Neural Network Type] \label{def:functionalType}
    We say that $\nn:W \to \bbR$ is a functional of neural network type if
\begin{equation} \label{eq:functionalType}
\nn[\alpha]
=
\Phi(M_W(\alpha)) \qquad \text{for all $\alpha \in W$} 
\end{equation}
where $M_W:W\to \mathbb{R}^m$ is linear, and $\Phi$ is a ReLU neural network. We define the complexity of $\nn$ as \[
\complexity(\nn) = \min_{\Phi \text{ satisfying \eqref{eq:functionalType}}} N_\#\l \Phi \r.
\]
\end{mydef}

We now extend this notion to operators. 

\begin{mydef}[Operator of Neural Network Type] \label{def:operatorType}
    We say that $\nn:U \to V$ is an operator of neural network type if for all $x\in \Omega_V$,
\begin{equation} \label{eq:operatorType}
\ev_x \circ \nn[u]
=
\Phi_x(M_U(u)) \qquad \text{for all $u \in U$} 
\end{equation}
where $M_U:U\to \mathbb{R}^q$ is linear, and $\Phi_x$ is a ReLU neural network. We define the complexity of $\nn$ as \[
\complexity(\nn) = \sup_{x \in \Omega_V} \min_{\Phi_x \text{ satisfying \eqref{eq:operatorType}}} N_\#\l \Phi_x \r.
\]
\end{mydef}

Equipped with the above geometric and representational notions, \cite[Corollary 2.12]{lanthalerStuart} establishes a curse of parametric complexity for operator learning. More precisely, for every \(r\in\mathbb N\), there exists an \(r\)-times Fréchet differentiable pathological operator \(G\) defined on any compact set containing an infinite-dimensional cube such that, for every operator of neural network type \(\nn_\varepsilon\) satisfying
\[
\sup_{u\in K}\|G[u]-\nn_\varepsilon[u]\|\le \varepsilon,
\]
the complexity of \(\nn_\varepsilon\) must obey a lower bound of the form
\[
\complexity(\nn_\varepsilon)\ge \exp\l c\,\varepsilon^{-q} \r
\]
for suitable constants \(c,q>0\). In Section \ref{sec:lowerBounds}, our objectives are twofold: first, we prove an analogous result for multiple operator learning with the architecture \eqref{eq:separableArchitecture}; second, we deduce minimax approximation-complexity rates for the class of Lipschitz multiple operator maps.

\section{Main Results} \label{sec:main}

In this section, we present the main results of the paper (proofs are given in the \nameref{sec:proofs}). We begin by deriving near-optimal scaling laws for both approximation and generalization. We then establish a lower complexity bound for approximation, which in turn yields minimax complexity rates for approximating Lipschitz multiple operator maps with MNO. Finally, we compare these minimax approximation-complexity rates with those of the simpler concatenated DeepONet architecture, which is directly inspired by single-operator learning.

\subsection{Near-Optimal Approximation Rates} \label{sec:improvedResults}

The first result of this section provides the near-optimal upper bound underlying our minimax rates in Theorem \ref{thm:minimaxMNO}.
The key point is that the new construction in Theorem \ref{thm:main:improvedRates} reduces the upper bound on approximation complexity in multiple operator learning to the scale of single operator learning. As explained in Section \ref{sec:intro}, in other words, the multiple operator structure no longer causes an additional exponential deterioration in the number of parameters required to achieve a prescribed accuracy $0<\varepsilon\ll 1$. The stronger bounds come from a different organization of the error analysis. Table \ref{tab:error-propagation-nested-parallel} schematically compares the nested proof strategy of \cite{weihs2025MOL} with the new one. In both constructions, the first stage is identical: one approximates the dependence of the target multiple operator map on the variable \(\alpha\), which produces a representation involving on the order of \(\varepsilon^{-\gamma}\) terms. The distinction arises only at the second stage, in how the resulting approximation error is aggregated. More precisely, in the nested argument, the second-stage error accumulates term by term over the first-stage representation. Hence, if the second-stage approximation error for each term is bounded by $\delta$, the total contribution is controlled by
\[
\sum_{p=1}^{\varepsilon^{-\gamma}} \delta
\;\asymp\;
\varepsilon^{-\gamma}\delta.
\]
To ensure that this remains of order \(\varepsilon\), one must therefore choose
\[
\varepsilon^{-\gamma}\delta \asymp \varepsilon,
\qquad\text{that is,}\qquad
\delta \asymp \varepsilon^{\gamma+1}.
\]
Thus the second-stage approximation must be carried out at a much smaller accuracy than the final target, which is precisely the source of the complexity inflation.

In the present construction, the second-stage error instead appears with coefficients \(\theta_p\) and takes the form
\[
\sum_{p=1}^{\varepsilon^{-\gamma}} \theta_p\,\delta
=
\delta \sum_{p=1}^{\varepsilon^{-\gamma}} \theta_p.
\]
The key point is that the coefficients are constructed so as to form a partition of unity, that is
\(
\sum_{p=1}^{\varepsilon^{-\gamma}} \theta_p \approx 1
\)
with a controllable error. Hence the total second-stage contribution is simply of order
\(
\delta,
\)
rather than \(\varepsilon^{-\gamma}\delta\). To keep the global error of order \(\varepsilon\), it is therefore enough to choose
\[
\delta \asymp \varepsilon.
\]
This implies that the second-stage approximation no longer suffers an additional amplification through the number of first-stage terms. This is the key mechanism behind the improved rates proved below.

The terms \emph{nested} and \emph{parallel} describe the interaction between the two approximation stages. In the nested case, the Stage~1 representation propagates into the Stage~2 error estimate and forces a smaller Stage~2 tolerance. In the parallel case, the Stage~2 error estimate is decoupled from the number of Stage~1 terms, so its required accuracy remains at the scale of the final target error.

\begin{table}[H]
\centering
\setlength{\tabcolsep}{5pt}
\renewcommand{\arraystretch}{1.45}
\begin{tabular}{p{0.28\textwidth} | p{0.24\textwidth} | p{0.4\textwidth}}
\midrule
 & \textbf{Nested strategy} & \textbf{Parallel strategy} \\
\midrule
\textit{Stage 1 output} 
& $\varepsilon^{-\gamma}$ terms 
& $\varepsilon^{-\gamma}$ terms with coefficients $\theta_p$ \\
\textit{Stage 2 error contribution} 
& $\displaystyle \sum_{p=1}^{\varepsilon^{-\gamma}} \delta$
& $\displaystyle \sum_{p=1}^{\varepsilon^{-\gamma}} \theta_p\,\delta
= \delta \sum_{p=1}^{\varepsilon^{-\gamma}} \theta_p$ \\
\textit{Coefficient structure} 
& no partition of unity
& approximate partition of unity,
$\displaystyle \sum_{p=1}^{\varepsilon^{-\gamma}} \theta_p \approx 1$ \\
\textit{Total Stage 2 error} 
& $\varepsilon^{-\gamma}\delta$
& $\delta$ \\
\textit{Condition for global error $\asymp \varepsilon$} 
& $\varepsilon^{-\gamma}\delta \asymp \varepsilon$
& $\delta \asymp \varepsilon$ \\
\textit{Required Stage 2 accuracy} 
& $\delta \asymp \varepsilon^{\gamma+1}$
& $\delta \asymp \varepsilon$ \\
\textit{Effect on complexity} 
& amplification of Stage 2
& no amplification through the number of terms \\
\midrule
\end{tabular}
\caption{Schematic comparison of how the second-stage approximation error enters in the nested construction of \cite{weihs2025MOL} and in the parallel construction introduced in Theorem \ref{thm:main:improvedRates}. In the nested case, the Stage~2 error accumulates over \(\varepsilon^{-\gamma}\) terms, forcing the per-term accuracy to be of order \(\varepsilon^{\gamma+1}\) in order to maintain a global error of order \(\varepsilon\). In the parallel case, the coefficients form an approximate partition of unity, so that the Stage~2 error factors as \(\delta \sum_p \theta_p \approx \delta\), and it therefore suffices to choose \(\delta \asymp \varepsilon\).}
\label{tab:error-propagation-nested-parallel}
\end{table}

\begin{theorem}[Scaling Laws for Multiple Operator Learning with the General Separable Architecture] \label{thm:main:improvedRates}
Consider integers $d_W,d_U,d_V>0$,  \[
\gamma_W, \gamma_U, \gamma_V, \beta_W, \beta_U,\beta_V,L_W,L_U,L_V,L_G,L_{\mathcal{G}} > 0 \qquad \text{and} \qquad r_G, r_{\mathcal{G}} \geq 1\]
and assume that $W(d_W,\gamma_W,L_W,\beta_W)$,
 $U(d_U,\gamma_U,L_U,\beta_U)$ and $V(d_V,\gamma_V,L_V,\beta_V)$ satisfy Assumption \ref{assumption:Main:assumptions:S4}. 
Let $G$ be a map satisfying Assumptions \ref{assumption:Main:assumptions:O1} and \ref{assumption:Main:assumptions:O2}.
There exist constants \begin{itemize}
    \item $C$ depending on $\gamma_V$, $L_V$
    \item $C'$ depending on $\beta_U$, $L_\mathcal{G}$, $\gamma_U$, $r_\mathcal{G}$
    \item $C_U$ depending on $L_\mathcal{G}$, $\gamma_U$, $r_\mathcal{G}$, $L_U$
    \item $C''$ depending on $\beta_W$, $L_G$, $\gamma_W$, $r_G$
    \item $C_{\zeta}$ depending on $L_G$, $\gamma_W$, $r_G$, $L_W$
\end{itemize}
such that the following holds. For $\eps > 0$ sufficiently small, \begin{itemize}
    \item let $N := 4C\sqrt{d_V}\eps^{-1}$ and consider the network class $\cF_1 := \cF_{\rm NN}(d_V,1,L_1,p_1,K_1,\kappa_1,R_1)$ whose parameters scale as \begin{align*}
&L_1 = \mathcal{O}\left(d_V^2\log d_V+d_V^2(\log(\varepsilon^{-1})+2\log(2))\right),\qquad p_1 = \mathcal{O}(1),\\
&K_1 = \mathcal{O}\left(d_V^2\log d_V+d_V^2(\log(\varepsilon^{-1})+2\log(2))\right),\qquad \kappa_1=\mathcal{O}(d_V^{d_V/2+1}\varepsilon^{-(d_V+1)}2^{2(d_V+1)}),\\
&R_1=1
    \end{align*}
    where the constants hidden in $\mathcal{O}$ depend on $\gamma_V$ and $L_V$;
    \item let $\delta_U=C_{U}\varepsilon/2$ and let $\{c_i\}_{i=1}^{n_{c_U}}\subset \Omega_U$ be points so that $\{\mathcal{B}_{\delta_U}(c_i) \}_{ i  = 1}^{n_{c_U}}$ is a cover of $\Omega_U$ for some $n_{c_U}$;
        \item let $H := 8C' \sqrt{n_{c_U}} \eps^{-1}$ and consider the network class $\cF_2 := \cF_{\rm NN}(n_{c_U}, 1 , L_2, p_2, K_2, \kappa_2, R_2)$ 
     with parameters scaling as
     \begin{align*}
    &L_2=\mathcal{O}\left(n_{c_U}^2\log(n_{c_U})+( n_{c_U}^2[\log(\varepsilon^{-1}) + 3\log(2)]\right), \qquad p_2 = \mathcal{O}(1),\\
    &K_2 = \mathcal{O}\left( n_{c_U}^2\log( n_{c_U})+ n_{c_U}^2[\log(\varepsilon^{-1}) + 3\log(2)]\right), \\ 
    &\kappa_2=\mathcal{O}( n_{c_U}^{n_{c_U}/2+1}\varepsilon^{-(n_{c_U} +1)} 2^{3(n_{c_U} +1)}),\qquad R_2=1
\end{align*}
    where the constants hidden in $\mathcal{O}$ depend on $\beta_U$, $L_\mathcal{G}$, $\gamma_U$, $r_\mathcal{G}$;
    \item let $\zeta:=C_{\zeta}\varepsilon/2$ and let $\{y_m\}_{m=1}^{n_{c_W}}\subset \Omega_W$ be points so that $\{\mathcal{B}_{\zeta}(y_m) \}_{ m = 1}^{n_{c_W}}$ is a cover of $\Omega_W$ for some $n_{c_W}$
    \item let $P = 2C'' \sqrt{n_{c_W}} \eps^{-1}$ and consider the network class $\cF_3 = \cF_{\rm NN}(n_{c_W},1,L_3,p_3,K_3,\kappa_3,R_3)$ whose parameters scale as \begin{align*}
&L_3=\mathcal{O}\left(n_{c_W}^2\log(n_{c_W})+n_{c_W}^2(\log(\varepsilon^{-1}) + 2\log(2)) \right),\quad  p_3 = \mathcal{O}(1),\\
    &K_3 = \mathcal{O}\left(n_{c_W}^2\log (n_{c_W})+n_{c_W}^2(\log(\varepsilon^{-1}) + 2\log(2))\right), \\ &\kappa_3=\mathcal{O}(n_{c_W}^{n_{c_W}/2+1}2^{2(n_{c_W}+1)}\varepsilon^{-n_{c_W}-1}),\qquad \, R_3=1
    \end{align*}
    where the constants hidden in $\mathcal{O}$ depend on $\beta_W$, $L_G$, $\gamma_W$, $r_G$.
\end{itemize}
Then, there exists:
\begin{itemize}
    \item networks $\{\tau_\ell\}_{k=1}^{N^{d_V}} \subset \cF_1$, $\{b_k\}_{k=1}^{H^{n_{c_U}}} \subset  \cF_2$ and $\{l_p\}_{p=1}^{P^{n_{c_W}}} \subset \cF_3$
    \item functions $\{u_k\}_{k=1}^{H^{n_{c_U}}} \subset \cB_{\beta_U,\Vert \cdot \Vert_{\Lp{\infty}},\Omega_U}(0)$ and $\{\alpha_p\}_{p=1}^{P^{n_{c_W}}} \subset \cB_{\beta_W,\Vert \cdot \Vert_{\Lp{\infty}},\Omega_W}(0)$
    \item points $\{v_\ell\}_{\ell=1}^{N^{d_V}} \subset \Omega_V$
\end{itemize}
such that 
\begin{equation} \label{thm:main:improvedRates:approx}
    \sup_{\alpha \in W} \sup_{u \in U} \sup_{x \in \Omega_V} \left\vert G[\alpha][u](x) - \sum_{p=1}^{P^{n_{c_W}}} \sum_{k=1}^{H^{n_{c_U}}} \sum_{\ell=1}^{N^{d_V}} G[\alpha_p][u_k](v_\ell) l_p(\bm{\alpha}) b_k(\bm{u}) \tau_\ell(x) \right\vert \leq \eps
\end{equation}
where $\bm{u}=(u(c_1), u(c_2),\dots,u(c_{n_{c_U}}))^\top$ and $\bm{\alpha} = (\alpha(y_1),\dots,\alpha(y_{n_{c_W}}))^\top$.

\end{theorem}

\begin{remark}[Total Number of Parameters for Multiple Operator Learning with the General Separable Architecture] \label{rem:countMNO}
    Let $\nn_\eps$ be a network that satisfies \eqref{thm:main:improvedRates:approx}. We want to estimate $\Vert \Theta \Vert_0 + K_1 N^{d_V} + K_2 H^{n_{c_U}} + K_3 P^{n_{c_W}} \geq N_\#(\nn_\eps)$ where $\Theta_0 = \{\theta_{pk\ell}\}_{p=1, k=1,\ell=1}^{P^{n_{c_W}},H^{n_{c_U}},N^{d_V}}$.

    First, we note that the computation in Remark \ref{rem:count} applies to our case (with $n_{c_W} = 0$ and $d_U = 0$ since we considered $W \times \{\emptyset\} \cong W$) to compute $K_1 N^{d_V} + K_2 H^{n_{c_U}}$. Indeed, the scalings of $\cF_1$ and $\cF_2$ only differ by constants between Theorem \ref{thm:main:improvedRates} and Proposition \ref{prop:multipleProductSpace}.
    Specifically, we therefore have \begin{equation} \label{eq:countMNO:eq1}
        H^{n_{c_U}} \lesssim \eps^{-d \eps^{-d_U}} \qquad \text{and} \qquad K_1 N^{d_V} + K_2 H^{n_{c_U}} \lesssim \eps^{-d \eps^{-d_U}}
    \end{equation}
    for some $d > 0$ depending on $d_U$.
    
    Next, by \cite[Lemma 2]{liu2024neuralscalinglawsdeep}, we have that \(
    n_{c_W} \lesssim \zeta^{-d_W} \lesssim \eps^{-d_W}.
    \)
    We estimate
    \begin{align}
        K_3 P^{n_{c_W}} &\lesssim \l \eps^{-2d_W} \log(\eps^{-1}) \r \l \eps^{-d_W/2+1} \r^{\eps^{-d_W}} \lesssim \eps^{-d' \eps^{-d_W}} \label{eq:rem:countMNO:eq2}
    \end{align}
    where $d'>0$ depends on $d_W$. Therefore, using \eqref{eq:countMNO:eq1} and \eqref{eq:rem:countMNO:eq2}, we have \begin{align}
        \Vert \Theta \Vert_0 &\leq N^{d_V} H^{n_{c_U}} P^{n_{c_W}} \lesssim  \eps^{-d' \eps^{-d_W}} \eps^{-d \eps^{-d_U}} \lesssim \eps^{-d'' \eps^{-\max\{d_W,d_U\}}} \label{eq:rem:countMNO:eq3}.
    \end{align}
    for some $d''$ depending on $d_U$ and $d_W$.
    Combining \eqref{eq:countMNO:eq1}, \eqref{eq:rem:countMNO:eq2} and \eqref{eq:rem:countMNO:eq3}, we obtain:
    \[
    \Vert \Theta \Vert_0 + K_1 N^{d_V} + K_2 H^{n_{c_U}} + K_3 P^{n_{c_W}} \lesssim \eps^{-d'' \eps^{-\max\{d_W,d_U\}}}.
    \]
    Repeating the computation in \cite[Remark 3.14]{weihs2025MOL}, we can also deduce that the approximation error $\eps$ scales as follows:
    \[
    \eps \lesssim \l \frac{\log N_\#}{\log \log N_\#} \r^{-1/\max\{d_W,d_U\}}
    \]
\end{remark}

\setlength{\tabcolsep}{5pt}
\renewcommand{\arraystretch}{1.35}

\begin{table}[H]
\centering
\small
\begin{tabularx}{\linewidth}{
>{\bfseries}l
*{2}{>{\centering\arraybackslash}X}
*{2}{>{\centering\arraybackslash}X}
*{2}{>{\centering\arraybackslash}X}
*{2}{>{\centering\arraybackslash}X}
*{2}{>{\centering\arraybackslash}X}
}
\toprule
& \multicolumn{2}{c}{\# networks}
& \multicolumn{2}{c}{width}
& \multicolumn{2}{c}{depth}
& \multicolumn{2}{c}{sparsity}
& \multicolumn{2}{c}{parameter magnitude} \\
\cmidrule(lr){2-3}\cmidrule(lr){4-5}\cmidrule(lr){6-7}\cmidrule(lr){8-9}\cmidrule(lr){10-11}
& Existing & \textcolor{ETHBlue}{Proposed}
& Existing & \textcolor{ETHBlue}{Proposed}
& Existing & \textcolor{ETHBlue}{Proposed}
& Existing & \textcolor{ETHBlue}{Proposed}
& Existing & \textcolor{ETHBlue}{Proposed} \\
\midrule
$l_p$
& $\eps^{-\eps^{-d_W}}$
& \textcolor{ETHBlue}{$\eps^{-\eps^{-d_W}}$}
& $\mathcal{O}(1)$
& \textcolor{ETHBlue}{$\mathcal{O}(1)$}
& $\eps^{-d_W}$
& \textcolor{ETHBlue}{$\eps^{-2d_W}$}
& $\eps^{-d_W}$
& \textcolor{ETHBlue}{$\eps^{-2d_W}$}
& $\eps^{-\eps^{-d_W}}$
& \textcolor{ETHBlue}{$\eps^{-\eps^{-d_W}}$}
\\
$b_k$
& $\eps^{-\eps^{-\eps^{-d_W}}}$
& \textcolor{ETHBlue}{$\eps^{-\eps^{-d_U}}$}
& $\mathcal{O}(1)$
& \textcolor{ETHBlue}{$\mathcal{O}(1)$}
& $\eps^{-\eps^{-d_W}}$
& \textcolor{ETHBlue}{$\eps^{-2d_U}$}
& $\eps^{-\eps^{-d_W}}$
& \textcolor{ETHBlue}{$\eps^{-2d_U}$}
& $\eps^{-\eps^{-\eps^{-d_W}}}$
& \textcolor{ETHBlue}{$\eps^{-\eps^{-d_U}}$}
\\
$\tau_\ell$
& $\eps^{-\eps^{-d_W}}$
& \textcolor{ETHBlue}{$\eps^{-d_V}$}
& $\mathcal{O}(1)$
& \textcolor{ETHBlue}{$\mathcal{O}(1)$}
& $\eps^{-d_W}$
& \textcolor{ETHBlue}{$\log(\eps^{-1})$}
& $\eps^{-d_W}$
& \textcolor{ETHBlue}{$\log(\eps^{-1})$}
& $\eps^{-\eps^{-d_W}}$
& \textcolor{ETHBlue}{$\eps^{-(1+d_V)}$}
\\
\bottomrule
\end{tabularx}
\caption{Comparison of the existing MNO constructive scalings from \cite{weihs2025MOL} and the proposed scalings derived in this work. Constants, lower-order and poly-logarithmic terms are suppressed; entries in the columns labelled \textcolor{ETHBlue}{Proposed} correspond to the present work.} 
\label{tab:mno-old-new-subcolumns}
\end{table}

\begin{remark}[Distribution of Complexity across Subnetworks]
We recall that the subnetworks \(l_p\), \(b_k\), and \(\tau_\ell\) in \eqref{eq:separableArchitecture} encode, respectively, the infinite-dimensional dependence on the parameter variable \(\alpha\), the infinite-dimensional dependence on the input function \(u\), and the remaining finite-dimensional dependence on the output variable \(x\). Beyond the improvement in total parameter scaling (see Remark \ref{rem:countMNO}), Table \ref{tab:mno-old-new-subcolumns} also reveals a markedly different distribution of complexity across the subnetworks.

In the existing construction of \cite{weihs2025MOL}, the complexity is highly unbalanced: the subnetworks \(l_p\) and \(\tau_\ell\) have comparable scalings, while the subnetworks \(b_k\) are dramatically more complex and therefore carry the main burden of the approximation. In this sense, the architecture is used in a strongly hierarchical manner, with the subnetworks appearing in the second-stage approximation, namely \(b_k\) and \(\tau_\ell\), absorbing most of the complexity. We also note that the space-approximation subnetworks \(\tau_\ell\) have the same complexity scaling as the subnetworks \(l_p\). This appears intuitively suboptimal, since one would expect the finite-dimensional spatial approximation step to be substantially simpler than the infinite-dimensional approximation of functional dependence.

By contrast, the proposed construction leads to a much more balanced allocation of complexity. The subnetworks \(l_p\) and \(b_k\) now have comparable scalings (both of the same general order as the previous \(l_p\) and \(\tau_\ell\)) while the subnetworks \(\tau_\ell\) become significantly simpler. Thus, the improved rates are also reflected in the internal organization of the architecture: rather than concentrating complexity in a single dominant block, the new proof distributes it more evenly across the infinite- and finite-dimensional approximation components.
\end{remark}

The next result yields substantially improved generalization scaling laws. Its proof is driven by the rates established in Theorem \ref{thm:main:improvedRates}: once the estimate \eqref{eq:generalization} is established, the remaining step is to control the metric entropy of the associated hypothesis class. Since this metric entropy is governed by the architectural scalings of the network classes, the improvement in approximation complexity propagates directly to the statistical error bound. In particular, because the new approximation theory removes the additional parametric curse specific to multiple operator learning, the resulting learning rate is likewise reduced to the scale dictated by operator learning.

\begin{theorem}[Scaling Laws for the Expected Generalization Error] \label{thm:scalingLawsGeneralizationError}
Let $d_W,d_U,d_V>0$ be integers, \[
\gamma_W, \gamma_U, \gamma_V, \beta_W, \beta_U,\beta_V,L_W,L_U,L_V,L_G,L_{\mathcal{G}} > 0 \qquad \text{and} \qquad r_G, r_{\mathcal{G}} \geq 1\]
and assume that $W(d_W,\gamma_W,L_W,\beta_W)$,
 $U(d_U,\gamma_U,L_U,\beta_U)$ and $V(d_V,\gamma_V,L_V,\beta_V)$ satisfy Assumption \ref{assumption:Main:assumptions:S4}.
Let $G$ be a map satisfying Assumptions \ref{assumption:Main:assumptions:O1} and \ref{assumption:Main:assumptions:O2}.

There exist constants \(C\), \(C_{\delta}\), \(C'\), \(C_{\zeta}\), and \(C''\), depending on the same quantities as in Theorem \ref{thm:main:improvedRates}, such that the following holds.
For any $\varepsilon>0$, use the latter constants to define \(N\), \(\delta\), \(H\), \(\zeta\), \(P\), the network classes \(\cF_1,\cF_2,\cF_3\), and the sampling points \(\{c_m\}_{m=1}^{n_{c_U}},\), \(\{y_m\}_{m=1}^{n_{c_W}} \)
as in Theorem \ref{thm:main:improvedRates}, with \(\varepsilon\) replaced everywhere by \(\varepsilon/2\).

Let $a = \beta_V$, $I \geq \beta_V$, $n_\alpha,n_u,n_x \in \bbN$,  $\mu_\alpha$ a probability measure on $W$, $\mu_u$ a probability measure on $U$, and $\mu_x$ a probability measure on $\Omega_V$. Consider the clipped network class \[
\mathrm{Cl}_a(I,\cF_1,\cF_2,\cF_3,\{y_s\},\{c_s\},P^{n_{c_W}},H^{n_{c_U}},N^{d_V}).
\] 

For $\eta > 0$, the expected generalization error is bounded as follows: \begin{align}
    &\mathbb{E}_{S_{G, \{y_s\}, \{c_s\}}} \mathbb{E}_{\alpha \sim \mu_\alpha} \, \mathbb{E}_{u \sim \mu_u} \, \mathbb{E}_{\{x_j\}_{j=1}^{n_x} \sim \mu_x^{\otimes n_x}} \left[
\frac{1}{n_x} \sum_{j=1}^{n_x} \left( G_{a,I,\cF_1,\cF_2,\cF_3,S}[\bm{\alpha}][\ub](x_j) - G[\alpha][u](x_j) \right)^2
\right] \notag  \\
        &\leq 4 \eps^2 + \eta (8 \sigma + 6) \notag \\
    &+ \frac{8\sigma \eta}{\sqrt{n_\alpha n_u n_x}} \sqrt{\log \l \cN\l\eta,\mathrm{Cl}_a(I,\cF_1,\cF_2,\cF_3,\{y_s\},\{c_s\},P^{n_{c_W}},H^{n_{c_U}},N^{d_V}),\Vert \cdot \Vert_{\Lp{\infty}(W \times U \times \Omega_V)}\r \r + \log(2)} \notag \\
    &+ \frac{16\sigma^2}{n_\alpha n_u n_x} \l \log \l \cN\l\eta,\mathrm{Cl}_a(I,\cF_1,\cF_2,\cF_3,\{y_s\},\{c_s\},P^{n_{c_W}},H^{n_{c_U}},N^{d_V}),\Vert \cdot \Vert_{\Lp{\infty}(W \times U \times \Omega_V)}\r \r + \log(2) \r \notag \\
    &+ \frac{112 \beta_V^2}{3 n_\alpha} \log\l   \cN\l\eta/(4\beta_V),\mathrm{Cl}_a(I,\cF_1,\cF_2,\cF_3,\{y_s\},\{c_s\},P^{n_{c_W}},H^{n_{c_U}},N^{d_V}),\Vert \cdot \Vert_{\Lp{\infty}(W \times U \times \Omega_V)}\r \r \label{eq:generalization}
    \end{align}
where $\bm{\alpha}=(\alpha(y_1), \alpha(y_2),...,\alpha(y_{n_{c_W}}))^\top$ and $\ub=(u(c_1), u(c_2),...,u(c_{n_{c_U}}))^\top$.

Furthermore, if we pick \[
    \varepsilon = 
    \l \frac{\max\{d_W,d_U\}}{4(1+\max\{d_W,d_U\}/2)} \frac{\log n_\alpha}{ \log \log n_\alpha} \r^{-1/\max\{d_W,d_U\}} \qquad \text{and} \qquad \eta = 4\beta_V n_\alpha^{-1},
    \] then the expected generalization error scales as follows:
    \begin{align}
        &\mathbb{E}_{S_{G, \{y_s\}, \{c_s\}}} \mathbb{E}_{\alpha \sim \mu_\alpha} \, \mathbb{E}_{u \sim \mu_u} \, \mathbb{E}_{\{x_j\}_{j=1}^{n_x} \sim \mu_x^{\otimes n_x}} \left[
\frac{1}{n_x} \sum_{j=1}^{n_x} \left( G_{a,I,\cF_1,\cF_2,\cF_3,S}[\bm{\alpha}][\ub](x_j) - G[\alpha][u](x_j) \right)^2
\right] \notag \\
&= \mathcal{O} \l \l \frac{\log n_\alpha}{ \log \log n_\alpha} \r^{-2/\max\{d_W,d_U\}} \r \notag
    \end{align}
where the constants hidden in $\mathcal{O}$ are independent of $n_\alpha, n_u, n_x$.
\end{theorem}

\begin{remark}[Exchanging the sampling hierarchy]
If one modifies the sampling procedure in Definition \ref{def:trainingSet} by exchanging the roles of \(\alpha\) and \(u\), then one expects the corresponding learning rate to take the same form as above, but with the parameter-side quantities replaced by their input-side counterparts. In particular, one would expect a learning rate of the form
\[
\mathcal{O}\!\left(
\left(
\frac{\log n_u}{\log\log n_u}
\right)^{-2/\max\{d_W,d_U\}}
\right).
\]
\end{remark}

\subsection{Approximation Complexity Lower Bounds and Minimax Rates for MNO} \label{sec:lowerBounds}

To extend the lower-complexity framework of \cite{lanthalerStuart} to the multiple operator setting, we first introduce the natural analogue of an operator of neural network type. The key idea is to require that, after fixing an input function \(u\) and an evaluation point \(x\), the resulting dependence on the parameter variable \(\alpha\) is represented by a finite-dimensional linear encoding followed by a ReLU network.

\begin{mydef}[Multiple Operator Map of Neural Network Type] \label{def:multipleOperatorType}
    We say that $\nn:W \to \{U \to V\}$ is a multiple operator map of neural network type if for every $u\in U$ and $x\in \Omega_V$,
\begin{equation} \label{eq:multipleOperatorType}
\ev_x \circ \ev_u \circ \nn[\alpha]
=
\Phi_{x,M_U(u)}(M_W(\alpha)) \qquad \text{for all $\alpha \in W$} 
\end{equation}
where $M_W:W\to \mathbb{R}^m$ and $M_U:U\to \mathbb{R}^q$ are linear, and $\Phi_{x,M_U(u)}$ is a ReLU neural network. We define the complexity of $\nn$ as $$
\complexity(\nn) = \sup_{x \in \Omega_V} \sup_{u \in U} \min_{\Phi_{x,M_U(u)} \text{ satisfying \eqref{eq:multipleOperatorType}}} N_{\#}\l \Phi_{x,M_U(u)} \r.$$ 
\end{mydef}

\begin{remark}[Limitations of the Classical Neural-network-type Framework]
The notion of operator of neural network type from Definitions \ref{def:operatorType} is formulated for maps that can be written in the form
\(
u \mapsto \Phi(M(u)),
\)
where \(M\) is linear and \(\Phi\) is a ReLU network. As recalled in Section \ref{sec:back:lowerBound}, this framework is well suited to a range of classical operator-learning architectures, as well as to concatenation-based models, where all inputs are first assembled into a single finite-dimensional vector and then processed jointly by one network (see Section \ref{sec:concatenated}).

By contrast, the separable architecture \eqref{eq:separableArchitecture} has the form
\[
\sum_{p=1}^{P}\sum_{k=1}^{H}\sum_{\ell=1}^{N}
\theta_{pk\ell}\, l_p(M_W(\alpha))\, b_k(M_U(u))\, \tau_\ell(x),
\]
and is therefore built from products of subnetworks. To cast such an expression into the form \(\Phi(M(\alpha,u))\), one would need to absorb terms such as
\[
l_p(M_W(\alpha))\, b_k(M_U(u))
\]
into a single ReLU network acting on a linear encoding of \((\alpha,u)\). Such closure under multiplication fails already for simple non-constant subnetworks. Indeed, the identity map \(x\mapsto x\) is exactly representable by a finite ReLU network, for instance via
\(
x=\mathrm{ReLU}(x)-\mathrm{ReLU}(-x),
\)
whereas multiplying this representation by itself yields the quadratic map \(x\mapsto x^2\), which is not piecewise affine. This is precisely why Definition \ref{def:multipleOperatorType} is needed: it is tailored to the separable structure of MNO-type architectures.
\end{remark}

The abstract lower-bound argument from \cite{lanthalerStuart} carries over directly to the multiple operator setting and yields the following curse of parametric complexity. The proof is based on lifting a pathological functional \(F:W\to\mathbb{R}\), known to exhibit the curse of parametric complexity, to a rank-one multiple operator map of the form
\[
G[\alpha][u](x)=F(\alpha)\phi(x),
\]
where \(\phi\in V\) is fixed and nontrivial. Evaluating \(G\) at \((u_0,x_0)\) recovers \(F(\alpha)\), so any low-complexity neural approximation of \(G\) would in particular yield a low-complexity neural approximation of \(F\). The desired lower bound therefore follows from the corresponding functional lower bound of \cite[Theorem 2.11]{lanthalerStuart}.

\begin{theorem}[Curse of Parametric Complexity for Multiple Operator Learning] \label{thm:codMNO}
Let $K$ be a compact subset of the Banach space $W$ containing an infinite-dimensional cube $Q_\eta$ for some $\eta > 1$.
Let $V$ be a Banach space continuously embedded in $C(\Omega_V)$, and let $U$ be a set of functions.

Then, for any $r \in \bbN$ and $\delta > 0$, there exists an $r$-times Frechet differentiable map $G:W \to \{U \to V\}$ and $\overline{\eps} := \overline{\eps}(\eta,\delta,r) > 0$ such that for any $\eps \leq \overline{\eps}$ and multiple operator map of neural network type $\nn_\varepsilon: W \to \{U \to V\}$ satisfying
\begin{equation} \label{eq:thm:codMNO:approximation}
    \sup_{\alpha\in K} \|\nn_\varepsilon[\alpha]-G[\alpha]\|_{\op} \leq \varepsilon,
\end{equation}
we have $\complexity(\nn_\eps) \geq \exp \l c \eps^{-1/[(\eta + 1 + \delta)r]} \r$ for some $c:=c(\eta,\delta,r) > 0$.
\end{theorem}

Theorem \ref{thm:codMNO} gives a lower-complexity principle for multiple operator learning in a fairly general setting. However, the separable architecture \eqref{eq:separableArchitecture} has an additional structural feature: the parameter variable \(\alpha\) and the input function \(u\) play analogous roles (from the analytical viewpoint). It is therefore natural to strengthen Definition \ref{def:multipleOperatorType} by requiring a neural-network-type representation in both variables. This leads to the following symmetric variant.

\begin{mydef}[Symmetric Multiple Operator Map of Neural Network Type] \label{def:multipleOperatorTypeSymmetric}
    We say that $\nn:W \times U \to V$ is a symmetric multiple operator map of neural network type if there exists linear maps $M_W:W\to \mathbb{R}^m$ and $M_U:U\to \mathbb{R}^q$ such that the following holds for every $x\in \Omega_V$: \begin{itemize}
        \item for every $u \in U$, we have the representation \begin{equation} \label{eq:symmetricW}
            \ev_x \circ \ev_u \circ \nn_W[\alpha] = \Phi_{x,M_U(u)}(M_W(\alpha)) \qquad \text{for all $\alpha \in W$}
        \end{equation}
        where $\Phi_{x,M_U(u)}$ is a ReLU neural network and $\nn_W:W \to \{U \to V\}$ is defined as $\nn_W[\alpha] = \nn[\alpha][\cdot](\cdot)$.
        \item for every $\alpha \in W$, we have the representation \begin{equation} \label{eq:symmetricU}
            \ev_x \circ \ev_\alpha \circ \nn_U[u] = \Psi_{x,M_W(\alpha)}(M_U(u)) \qquad \text{for all $u \in U$}
        \end{equation}
        where $\Psi_{x,M_W(\alpha)}$ is a ReLU neural network and $\nn_U:U \to \{W \to V\}$ is defined as $\nn_U[u] = \nn[\cdot][u](\cdot)$.
    \end{itemize}
We define the complexity of $\nn$ as $$
\complexity(\nn) = \max \left\{ \sup_{x \in \Omega_V} \sup_{u \in U} \min_{\substack{\Phi_{x,M_U(u)} \\ \text{satisfying \eqref{eq:symmetricW}}}} N_{\#}\l \Phi_{x,M_U(u)} \r, \sup_{x \in \Omega_V} \sup_{\alpha \in W} \min_{\substack{
\Psi_{x,M_W(\alpha)} \\
\text{satisfying \eqref{eq:symmetricU}}
}} N_{\#}\l \Psi_{x,M_W(\alpha)} \r \right\}.$$ 
\end{mydef}

This directly leads to the following curse of parametric complexity for symmetric multiple operator map of neural network type.

\begin{lemma}[Symmetric Curse of Parametric Complexity for Multiple Operator Learning]
\label{lem:codMNO:symmetric}
Let $W, \, U$ be Banach spaces and $K_W\subset W, \, K_U\subset U$ be compact subsets. Assume that $K_W$ and $K_U$ contain infinite-dimensional cubes $Q_{\eta_W}$ and $Q_{\eta_U}$ for some $\eta_W>1$ and $\eta_U > 1$, respectively.
Let $V$ be a Banach space continuously embedded in $C(\Omega_V)$. We equip the product Banach space $W\times U$ with a norm $\Vert \cdot \Vert_{W\times U}$ satisfying Assumption \ref{assumption:Main:assumptions:N1}.

Then, for any $r \in \bbN$ and $\delta > 0$, there exists an $r$-times Frechet differentiable map $G:W \times U \to V$ and $\overline{\eps} := \overline{\eps}(\eta_W,\eta_U,\delta,r) > 0$ such that for any $\eps \leq \overline{\eps}$ and symmetric multiple operator map of neural network type $\nn_\varepsilon: W \times U \to V$ satisfying
\begin{equation} \label{eq:thm:codMNOSymmetric:approximation}
    \sup_{\alpha\in K_W} \sup_{u \in K_U} \|\nn_\varepsilon[\alpha][u]-G[\alpha][u]\|_{V} = \sup_{u \in K_U} \sup_{\alpha \in K_W} \|\nn_\varepsilon[\alpha][u]-G[\alpha][u]\|_{V} \leq \varepsilon,
\end{equation}
we have $\complexity(\nn_\eps) \geq \exp \l c \eps^{-1/[(\min\{\eta_W,\eta_U\} + 1 + \delta)r]} \r$ for some $c:=c(\eta_W,\eta_U,\delta,r) > 0$.
\end{lemma}

\begin{remark}[Choice of Operator Norm and Banach Space Assumption for $V$] \label{rem:banachSpace}
The proof of Theorem \ref{thm:codMNO} does not rely on the specific choice
\[
\|T\|_{\op} = \sup_{u\in U} \|T(u)\|_V
\]
except through the estimate
\(
|T(u_0)(x_0)| \le C \|T\|_{\op}
\)
for some fixed $u_0\in U$, $x_0\in \Omega_V$, and some constant $C>0$; see \eqref{eq:thm:codMNO:eq2}.
Hence, the same argument applies to any norm on the operator space for which the slice-evaluation map
\(
T \mapsto T(u_0)(x_0)
\)
is continuous.

In particular, one may replace $\|\cdot\|_{\op}$ by the uniform pointwise norm
\[
\|T\|_{\op,\infty} := \sup_{u\in U}\sup_{x\in \Omega_V} |T(u)(x)|.
\]
With this choice, the proof becomes slightly simpler, since
\(
|T(u_0)(x_0)| \le \|T\|_{\op,\infty}
\)
holds with constant $1$, and the continuous embedding $V \hookrightarrow C(\Omega_V)$ is no longer needed.
Accordingly, in the proof of Theorem \ref{thm:codMNO}, the constant $C_V$ disappears and one may take $\overline{\eps}=\eps_0$.
Moreover, the approximation condition \eqref{eq:thm:codMNO:approximation} becomes
\(
\sup_{\alpha\in K} \|\nn_\varepsilon[\alpha]-G[\alpha]\|_{\op,\infty} \le \varepsilon,
\)
and one obtains
\(
\sup_{\alpha\in K} |F(\alpha)-\ev_{x_0}\circ \ev_{u_0} \circ \nn_{\eps}[\alpha]| \le \eps
\)
directly.

A further advantage of this formulation is that the proof no longer requires $V$ to be a Banach space. Indeed, in the proof based on $\|\cdot\|_{\op}$, the estimate \eqref{eq:thm:codMNO:eq2} requires that
\[
\bigl(G[\alpha]-\nn_{\eps}[\alpha]\bigr)(u_0)\in V
\]
so that its $V$-norm is well-defined. By contrast, for the norm $\|\cdot\|_{\op,\infty}$ it suffices that the outputs are actual bounded functions on $\Omega_V$, since the proof only uses pointwise differences. Therefore, the same argument applies if V satisfies Assumption \ref{assumption:Main:assumptions:S4}, rather than being Banach space.

The same remark applies to Lemma \ref{lem:codMNO:symmetric} with the norm choice of $\sup_{\alpha\in K_W} \sup_{u \in K_U} \| \cdot \|_{\Lp{\infty}}$.

\end{remark}

\begin{remark}[Role of the Symmetry Assumptions]
The symmetric approximation norm $$\sup_{\alpha\in K_W} \sup_{u \in K_U} \|\nn_\varepsilon[\alpha][u]-G[\alpha][u]\|_{V}$$ and the two symmetric neural network representations in \eqref{eq:symmetricW} and \eqref{eq:symmetricU} are essential in Lemma \ref{lem:codMNO:symmetric}, since the proof may be carried out starting either from the parameter side $W$ or from the input side $U$, depending on which cube exponent is smaller.
\end{remark}

Now that we have established the curse of parametric complexity for multiple operator learning in an abstract setting, we turn to the specific function classes used in our approximation-theoretical result, Theorem \ref{thm:main:improvedRates}. To apply the lower-bound framework in that setting, the first step is to verify that the bounded Lipschitz classes from Assumption \ref{assumption:Main:assumptions:S4} contain an infinite-dimensional hypercube. The next result addresses precisely this point. 

\begin{lemma}[Infinite-dimensional Cube in Bounded Lipschitz Classes]
\label{lem:boundedLipCube}
Let $d_U \in \bbN$, $\gamma_U,L_U,\beta_U > 0$ and $U = U(d_U,\gamma_U,L_U,\beta_U)$ satisfy Assumption \ref{assumption:Main:assumptions:S4}. 
Then, viewed as a subset of the Banach space $\Lp{r}(\Omega_U)$ with $r\geq 1$, 
$U$ contains an infinite-dimensional hypercube $Q_\eta$ for every
\(
\eta > 1+\frac{1}{d_U}.
\)
\end{lemma}

Combining the latter result and Lemma \ref{lem:codMNO:symmetric}, we obtain the following.

\begin{corollary}[Curse of Parametric Complexity for Multiple Operator Learning on Bounded Lipschitz Classes] \label{cor:codMNO}
Let $d_W,d_U,d_V>0$ be integers, \[
\gamma_W, \gamma_U, \gamma_V, \beta_W, \beta_U,\beta_V,L_W,L_U,L_V> 0 \quad \text{and} \quad r_G \geq 1\]
and assume that $W(d_W,\gamma_W,L_W,\beta_W)$,
 $U(d_U,\gamma_U,L_U,\beta_U)$ and $V(d_V,\gamma_V,L_V,\beta_V)$ satisfy Assumption \ref{assumption:Main:assumptions:S4}. We equip the product Banach space $\Lp{r_G}(\Omega_W) \times \Lp{r_G}(\Omega_U)$ with a norm $\Vert \cdot \Vert_{\Lp{r_G}(\Omega_W) \times \Lp{r_G}(\Omega_U)}$ that satisfies Assumption \ref{assumption:Main:assumptions:N1}. 

Then, for any $\eta > \min\left\{1 + \frac{1}{d_W},1+\frac{1}{d_U}\right\}$, $r \in \bbN$ and $\delta > 0$, there exists an $r$-times Frechet differentiable map $G:\Lp{r_G}(\Omega_W) \times \Lp{r_G}(\Omega_U) \to V$ and $\overline{\eps} := \overline{\eps}(\eta,\delta,r) > 0$ such that for any $\eps \leq \overline{\eps}$ and symmetric multiple operator map of neural network type $\nn_\varepsilon$ satisfying
\begin{equation*} 
    \sup_{\alpha\in W} \sup_{u \in U} \sup_{x \in \Omega_V} \vert \nn_\varepsilon[\alpha][u](x)-G[\alpha][u](x)\vert \leq \varepsilon,
\end{equation*}
we have $\complexity(\nn_\eps) \geq \exp \l c \eps^{-1/[(\eta + 1 + \delta)r]} \r$ for some $c:=c(\eta,\delta,r) > 0$.
\end{corollary}

\begin{remark}[Asymmetric Structure]
If one were to rely only on Theorem \ref{thm:codMNO}, rather than on Lemma \ref{lem:codMNO:symmetric}, in the proof of Corollary \ref{cor:codMNO}, then one would not fully exploit the symmetric role played by the parameter space and the input-function space in the bounded Lipschitz setting. Specifically, one would obtain an analogous result for (asymmetric) multiple operator maps of neural network type under the more restrictive condition
\(
\eta > 1 + \frac{1}{d_W}.
\)
\end{remark}

Corollary \ref{cor:codMNO} provides a lower bound on approximation complexity for multiple operator maps satisfying Assumptions \ref{assumption:Main:assumptions:O1} and \ref{assumption:Main:assumptions:O2} using general symmetric multiple operator maps of neural network type defined on bounded Lipschitz classes. It only remains to connect this result directly to the architecture of interest \eqref{eq:separableArchitecture}. To do so, we estimate the complexity of the corresponding network class in the next result.

\begin{lemma}[Complexity of the General Separable Architecture] \label{lem:complexitySeparable}
The network given by
    \[\nn[\alpha][u](x) = \sum_{p=1}^{P}\sum_{k=1}^{H}\sum_{\ell=1}^{N}
    \theta_{pk\ell}\, l_p(M_W(\alpha))\, b_k(M_U(u))\, \tau_\ell(x) \qquad \theta_{pk\ell} \in \bbR\] 
with $l_p\in \cF_{\rm NN}(d_3, 1, L_3, p_3, K_3, \kappa_3, 1)$, $b_k \in \cF_{\rm NN}(d_2, 1, L_2, p_2, K_2, \kappa_2, 1)$ and $\tau_\ell \in \cF_{\rm NN}(d_1, 1, L_1, p_1, K_1, \kappa_1, 1)$
is a symmetric multiple operator map of neural network type. Furthermore, we have \[
\cC(\nn) \leq 2 (\Vert \Theta \Vert_0 + HK_2 + N K_1 + P K_3)
\]
where $\Theta = \{\theta_{pk\ell}\}$.
\end{lemma}

\begin{remark}[Asymmetric Complexity of the General Separable Architecture]
From Definitions \ref{def:multipleOperatorType} and \ref{def:multipleOperatorTypeSymmetric}, we note that any symmetric multiple operator map of neural network type is also an asymmetric one. Also, its complexity measured under the symmetric setting is an upper bound to its complexity in the asymmetric case. This observation  specifically applies to the general separable architecture \eqref{eq:separableArchitecture} and, given Lemma \ref{lem:complexitySeparable}, it is unnecessary to compute the complexity separately under the asymmetric notion.
\end{remark}

\begin{remark}[Complexity of the MNO architecture]
The same argument as in Lemma \ref{lem:complexitySeparable} shows that the MNO architecture
\[
\nn[\alpha][u](x)
=
\sum_{p=1}^P \sum_{k=1}^H l_p(M_W(\alpha))\, b_{pk}(M_U(u))\, \tau_{pk}(x)
\]
is a symmetric multiple operator map of neural network type. Moreover,
\[
\cC(\nn)
\leq
2(HPK_2 + HPK_1 + PK_3)
=
2P\bigl(H(K_1+K_2)+K_3\bigr).
\]
\end{remark}

We conclude this section with minimax rates for Lipschitz multiple operator maps using the architecture \eqref{eq:separableArchitecture}, as a combination of Corollary \ref{cor:codMNO}, Lemma \ref{lem:complexitySeparable} and Theorem \ref{thm:main:improvedRates}.

\begin{theorem}[Minimax Bounds for Lipschitz Multiple Operator Maps] \label{thm:minimaxMNO}
Let $d_W,d_U,d_V>0$ be integers, \[
\gamma_W, \gamma_U, \gamma_V, \beta_W, \beta_U,\beta_V,L_W,L_U,L_V> 0 \quad \text{and} \quad r_G \geq 1\]
and assume that $W(d_W,\gamma_W,L_W,\beta_W)$,
 $U(d_U,\gamma_U,L_U,\beta_U)$ and $V(d_V,\gamma_V,L_V,\beta_V)$ satisfy Assumption \ref{assumption:Main:assumptions:S4}. We equip the product Banach space $\Lp{r_G}(\Omega_W) \times \Lp{r_G}(\Omega_U)$ with a norm $\Vert \cdot \Vert_{\Lp{r_G}(\Omega_W) \times \Lp{r_G}(\Omega_U)}$ that satisfies Assumption \ref{assumption:Main:assumptions:N1}.
\begin{enumerate}
\item For any $\eta > \min\left\{1 + \frac{1}{d_W},1+\frac{1}{d_U}\right\}$, $r \in \bbN$ and $\delta > 0$, there exists an $r$-times Frechet differentiable map $G:\Lp{r_G}(\Omega_W) \times \Lp{r_G}(\Omega_U) \to V$ and $\overline{\eps} := \overline{\eps}(\eta,\delta,r) > 0$ such that the following holds:
for any $\eps \leq \overline{\eps}$ and $\nn_\eps$ of the form \eqref{eq:separableArchitecture}
satisfying
\begin{equation} \label{eq:thm:minimax:approximation}
    \sup_{\alpha\in W} \sup_{u \in U} \sup_{x \in \Omega_V} \vert \nn_\varepsilon[\alpha][u](x)-G[\alpha][u](x)\vert \leq \varepsilon,
\end{equation}
we have \[
\Vert \Theta \Vert_0 + HK_2 + N K_1 + P K_3 \gtrsim \exp \l c \eps^{-1/[(\eta + 1 + \delta)r]} \r
\] for some $c:=c(\eta,\delta,r) > 0$ and where $\Theta = \{\theta_{pk\ell}\}$.
\item With $r = 1$, the map $G$ in part 1. can be chosen so that it satisfies Assumptions \ref{assumption:Main:assumptions:O1} and \ref{assumption:Main:assumptions:O2}. 
\item Let $\mathcal H$ denote the class of all maps $G$ satisfying Assumptions \ref{assumption:Main:assumptions:O1} and \ref{assumption:Main:assumptions:O2}. Define the worst-case/minimax approximation complexity
\begin{align*}
\mathfrak C(\varepsilon;\mathcal H)
:=
\inf\Bigl\{&
M\in\mathbb N \,\Big|\, \forall G\in\mathcal H,\ \exists \nn_\varepsilon
\text{ of the form \eqref{eq:separableArchitecture} satisfying \eqref{eq:thm:minimax:approximation}} \\
&\text{and }\|\Theta\|_0 + HK_2 + NK_1 + PK_3 \le M
\Bigr\}.
\end{align*}
Let $r=1$, $\eta > \min\left\{1 + \frac{1}{d_W},1+\frac{1}{d_U}\right\}$, and $\delta > 0$.  Then, for all $0 < \varepsilon \leq \overline{\eps}(\eta,\delta,1)$,
\begin{equation} \label{eq:thm:minimax:minimax}
\exp\l c\,\varepsilon^{-1/[(\eta+1+\delta)]}\r
\lesssim
\mathfrak C(\varepsilon;\mathcal H)
\lesssim
\exp \l d \log(\eps^{-1}) \eps^{-\max\{d_W,d_U\}} \r
\end{equation}
for some $d >0$ depending on $d_W$ and $d_U$.
\end{enumerate}
 
\end{theorem}

\begin{remark}[Upper Bounds under Higher Regularity]
The lower bound above is obtained from a general \(C^r\)-differentiable construction, and therefore applies more broadly than the Lipschitz class \(\mathcal H\) considered in the present minimax formulation. This suggests that, for classes of \(C^r\) multiple operator maps, one should likewise expect constructive upper bounds of a similar double-exponential type.
\end{remark}

The lower and upper bounds in \eqref{eq:thm:minimax:minimax} are both of exponential type in a diverging function of $\varepsilon^{-1}$, and therefore capture the same qualitative curse of parametric complexity on the class $\mathcal H$. The upper bound involves a faster-growing exponent, reflecting both a larger power of $\varepsilon^{-1}$ and an additional logarithmic factor. Thus, although the two estimates are not tight at the level of the exponent, they are consistent in identifying the same overall double-exponential complexity regime.

This comparison also shows that, as in operator learning \cite{lanthalerStuart}, Lipschitz or, more generally, \(C^r\)-regularity alone is not sufficient to overcome the curse of parametric complexity in multiple operator learning. Rather, one must exploit additional structure in the target multiple operator maps. Two important mechanisms in this direction are holomorphic dependence on the inputs or parameters, and architectures designed to emulate underlying numerical schemes \cite{marcati2023,lanthalerStuart}. In operator learning, both have already been explored as ways of overcoming the curse of parametric complexity (see Section \ref{sec:related}). In the multiple operator setting, however, they appear to remain largely unexplored, and investigating their potential to improve approximation complexity would therefore be of clear interest.

\subsection{An Extension of DeepONet to Multi-Task Learning} \label{sec:concatenated}

In this final section, we examine an alternative approach to multiple operator learning. A particularly direct strategy is to start from a standard single-operator learning architecture and incorporate the parametric descriptor $\alpha$ simply by concatenating it with the input function $u$. It is therefore natural to ask how this concatenation affects approximation complexity. To address this question, we derive bounds for the concatenated DeepONet architecture \eqref{eq:concatenatedDeepONet} introduced below.

As in our analysis of the MNO architecture, instead of working directly with the concatenated DeepONet architecture
\begin{equation} \label{eq:concatenatedDeepONet}
\mathrm{DeepONet}_{\mathrm{C}}[\alpha][u](x) = \sum_{k=1}^H b_k(M_W(\alpha),M_U(u)) \tau_k(x),
\end{equation}
we consider the more general separable form
\begin{equation} \label{eq:separable2}
\sum_{k=1}^{H}\sum_{\ell=1}^{N}
\theta_{k\ell}\, b_{k}(M_W(\alpha),M_U(u))\, \tau_\ell(x),
\qquad \theta_{k\ell}\in\mathbb R.
\end{equation}

We begin by showing that the concatenated DeepONet architecture \eqref{eq:concatenatedDeepONet} is an operator of neural network type, thereby placing it directly within the lower-complexity framework of \cite{lanthalerStuart}. This is to be expected, since the concatenated model can be interpreted as an ordinary operator-learning problem posed on the product space \(W\times U\). We state the result, however, for the more general separable architecture \eqref{eq:separable2}. The proof therefore is a combination of the strategy used in Lemma \ref{lem:complexitySeparable} and the argument underlying \cite[Lemma 2.20]{lanthalerStuart}.

\begin{lemma}[Complexity of the General Separable Architecture with Concatenated Inputs] \label{lem:complexityConcatenated}

    The network given by \[
    \nn[\alpha][u](x) = \sum_{k=1}^{H}\sum_{\ell=1}^{N}
\theta_{k\ell}\, b_{k}(M_W(\alpha),M_U(u))\, \tau_\ell(x),
\qquad \theta_{k\ell}\in\mathbb R
    \]
    with $b_k \in \cF_{\rm NN}(d_2, 1, L_2, p_2, K_2, \kappa_2, 1)$ and $\tau_\ell \in \cF_{\rm NN}(d_1, 1, L_1, p_1, K_1, \kappa_1, 1)$
is an operator map of neural network type from $W \times U$ to $V$. Furthermore, we have \[
\cC(\nn) \leq 2\bigl(\|\Theta\|_0 + N K_1 + H K_2\bigr).
\]
\end{lemma}

\begin{remark}[Symmetric Multiple Operator Map Interpretation of the Concatenated Separable Architecture]
One may also view the architecture
\[
\nn[\alpha][u](x)
=
\sum_{k=1}^{H}\sum_{\ell=1}^{N}
\theta_{k\ell}\, b_k(M_W(\alpha),M_U(u))\, \tau_\ell(x)
\]
as a symmetric multiple operator map of neural network type, rather than as an operator map of neural network type on the product space \(W\times U\). In that case, however, the resulting complexity estimate is slightly less sharp. More precisely, one obtains the upper bound
\[
\cC(\nn)
\le
2(\|\Theta\|_0 + N K_1 + H(K_2+p_2)).
\]
The additional \(2Hp_2\) term arises from the fact that, in order to verify the symmetric definition, one must show that if \(b(x,y)\) is a ReLU network in the concatenated variables \((x,y)\), then freezing one block of variables still yields a ReLU network in the remaining variables. This operation modifies the first affine layer and increases the corresponding complexity estimate by at most the width $D$ of that layer. Since $D \leq p_2$, we obtain the above.
\end{remark}

As in Section \ref{sec:lowerBounds}, our goal is to obtain minimax rates in the same setting as our approximation theory, and in particular for input classes given by bounded Lipschitz spaces as in Assumption \ref{assumption:Main:assumptions:S4}. For the architecture \eqref{eq:separable2}, however, the relevant input space is a product of two such classes. The next result explains how to construct an infinite-dimensional cube in this product space by embedding a single cube into one factor.

\begin{lemma}[Embedding a Single Cube into a Product Space]
\label{lem:productCubeFromSingleCube}
Let \(W\) and \(U\) be Banach spaces, and equip \(W\times U\) with a norm \(\|\cdot\|_{W\times U}\) satisfying Assumptions \ref{assumption:Main:assumptions:N1} and \ref{assumption:Main:assumptions:N3}.
Let \(K_W\subset W\) and suppose that \(K_W\) contains an infinite-dimensional cube
\(
Q_\eta(A;\{e_j\}_{j\in\mathbb N})
\)
for some \(\eta>1\). Then the subset \(K_W\times\{0\}\subset W\times U\) contains an infinite-dimensional cube with the same exponent \(\eta\).
\end{lemma}

\begin{remark}[Optimality of infinite-dimensional product cube]
    Although the construction of the infinite-dimensional product cube in Lemma \ref{lem:productCubeFromSingleCube} is straightforward, it also raises a natural question of optimality. Namely, one may ask whether there exists a genuinely product-space cube construction that would lead to a better lower-bound exponent than the one obtained here from a single-factor embedding.
\end{remark}

We now turn to the upper bound in the corresponding minimax approximation-complexity rates. In particular, the following result shows that the concatenated architecture is sufficiently expressive to approximate operator-learning problems posed on product spaces.

\begin{proposition}[Scaling Laws for Multiple Operator Learning through Product Spaces] \label{prop:multipleProductSpace}
    Let $d_W,d_U,d_V>0$ be integers, \[
\gamma_W, \gamma_U, \gamma_V, \beta_W, \beta_U,\beta_V,L_W,L_U,L_V,L_G,L_{\mathcal{G}} > 0 \qquad \text{and} \qquad r_G, r_{\mathcal{G}} \geq 1\]
and assume that $W(d_W,\gamma_W,L_W,\beta_W)$,
 $U(d_U,\gamma_U,L_U,\beta_U)$ and $V(d_V,\gamma_V,L_V,\beta_V)$ satisfy Assumption \ref{assumption:Main:assumptions:S4}. 
Let $G$ be a map satisfying Assumptions \ref{assumption:Main:assumptions:O1} and \ref{assumption:Main:assumptions:O2}.
There exist constants \begin{itemize}
    \item $C$ depending on $\gamma_V$, $L_V$
    \item $C'$ depending on $\beta_W$, $\beta_U$, $L_\mathcal{G}$, $\gamma_U$, $r_\mathcal{G}$, $L_G$, $\gamma_W$, $r_G$
    \item  $C_{W}$ depending on $L_\mathcal{G}$, $\gamma_U$, $r_\mathcal{G}$, $L_G$, $\gamma_W$, $r_G$, $L_W$
    \item $C_U$ depending on $L_\mathcal{G}$, $\gamma_U$, $r_\mathcal{G}$, $L_G$, $\gamma_W$, $r_G$, $L_U$
\end{itemize}
such that the following holds. For $0 < \eps$ sufficiently small,  \begin{itemize}
    
    \item let $N := 2C\sqrt{d_V}\eps^{-1}$ and consider the network class $\cF_1 := \cF_{\rm NN}(d_V,1,L_1,p_1,K_1,\kappa_1,R_1)$ whose parameters scale as \begin{align*}
&L_1 = \mathcal{O}\left(d_V^2\log d_V+d_V^2(\log(\varepsilon^{-1})+\log(2))\right),\qquad p_1 = \mathcal{O}(1),\\
&K_1 = \mathcal{O}\left(d_V^2\log d_V+d_V^2(\log(\varepsilon^{-1})+\log(2))\right),\qquad \kappa_1=\mathcal{O}(d_V^{d_V/2+1}\varepsilon^{-(d_V+1)}2^{d_V+1}),\\
&R_1=1
    \end{align*}
    where the constants hidden in $\mathcal{O}$ depend on $\gamma_V$ and $L_V$
    \item let $\delta_W=C_{W}\varepsilon/2$ and let $\{a_i\}_{i=1}^{n_{c_W}}\subset \Omega_W$ be points so that $\{\mathcal{B}_{\delta_W}(a_i) \}_{ i  = 1}^{n_{c_W}}$ is a cover of $\Omega_W$ for some $n_{c_W}$;
        \item let $\delta_U=C_{U}\varepsilon/2$ and let $\{c_i\}_{i=1}^{n_{c_U}}\subset \Omega_U$ be points so that $\{\mathcal{B}_{\delta_U}(c_i) \}_{ i  = 1}^{n_{c_U}}$ is a cover of $\Omega_U$ for some $n_{c_U}$;
        \item let $H := 4C' \sqrt{n_{c_W} + n_{c_U}} \eps^{-1}$ and consider the network class $\cF_2 := \cF_{\rm NN}(n_{c_W} + n_{c_U}, 1 , L_2, p_2, K_2, \kappa_2, R_2)$ 
     with parameters scaling as
     \begin{align*}
    &L_2=\mathcal{O}\left((n_{c_W} + n_{c_U})^2\log(n_{c_W} + n_{c_U})+(n_{c_W} + n_{c_U})^2[\log(\varepsilon^{-1}) + 2\log(2)]\right), \qquad p_2 = \mathcal{O}(1),\\
    &K_2 = \mathcal{O}\left((n_{c_W} + n_{c_U})^2\log(n_{c_W} + n_{c_U})+(n_{c_W} + n_{c_U})^2[\log(\varepsilon^{-1}) + 2\log(2)]\right), \\ 
    &\kappa_2=\mathcal{O}((n_{c_W} + n_{c_U})^{(n_{c_W} + n_{c_U})/2+1}\varepsilon^{-(n_{c_W} + n_{c_U} +1)} 2^{2(n_{c_W} + n_{c_U} +1)}),\qquad R_2=1
\end{align*}
    where the constants hidden in $\mathcal{O}$ depend on $\beta_W, \beta_U$, $L_\mathcal{G}$, $\gamma_U$, $r_\mathcal{G}$, $L_G$, $\gamma_W$, $r_G$.
\end{itemize}
Then, there exists
\begin{itemize}
    \item networks $\{\tau_\ell\}_{k=1}^{N^{d_V}} \subset \cF_1$ and $\{b_k\}_{k=1}^{H^{n_{c_W} + n_{c_U}}} \subset  \cF_2$
    \item functions $\{\alpha_k\}_{k=1}^{H^{n_{c_W} + n_{c_U}}} \subset \cB_{\beta_W,\Vert \cdot \Vert_{\Lp{\infty}},\Omega_W}(0)$ and $\{u_k\}_{k=1}^{H^{n_{c_W} + n_{c_U}}} \subset \cB_{\beta_U,\Vert \cdot \Vert_{\Lp{\infty}},\Omega_U}(0)$
    \item points $\{v_\ell\}_{\ell=1}^{N^{d_V}} \subset \Omega_V$
\end{itemize}
such that \begin{equation} \label{eq:thm:minmax2}
    \sup_{\alpha \in W} \sup_{u \in U} \sup_{x \in \Omega_V} \left\vert G[\alpha][u](x) - \sum_{\ell=1}^{N^{d_V}} \sum_{k=1}^{H^{n_{c_W}+n_{c_U}}} G[\alpha_k][u_k](v_\ell) b_k(\bm{\alpha},\bm{u}) \tau_\ell(x) \right\vert \leq \eps
\end{equation}
and where $\bm{\alpha}=\frac{\max\{\beta_W,\beta_U\}}{\beta_W}(\alpha(a_1), \alpha(a_2),\dots,\alpha(a_{n_{c_W}}))^\top$, $\bm{u}=\frac{\max\{\beta_W,\beta_U\}}{\beta_U}(u(c_1), u(c_2),\dots,u(c_{n_{c_U}}))^\top$.
\end{proposition}

\begin{remark}[Total Number of Parameters for Multiple Operator Learning through Product Spaces] \label{rem:count}
    Let $\nn_\eps$ be a network that satisfies \eqref{eq:thm:minmax2}. We want to estimate $\Vert \Theta \Vert_0 + K_1 N^{d_V} + K_2 H^{n_{c_W}+n_{c_U}} \geq N_\#(\nn_\eps)$ where $\Theta_0 = \{\theta_{k\ell}\}_{k=1,\ell=1}^{H^{n_{c_W} + n_{c_U}},N^{d_V}}$.

    First, by \cite[Lemma 2]{liu2024neuralscalinglawsdeep}, we have that \[
    n_{c_W} \lesssim \delta_W^{-d_W} \lesssim \eps^{-d_W} \qquad \text{and} \qquad n_{c_U} \lesssim \delta_U^{-d_U} \lesssim \eps^{-d_U},
    \]
    implying that $n_{c_W} + n_{c_U} \lesssim \eps^{-\max\{d_W,d_U\}}$.
    Next, we consider \begin{align}
        K_1 N^{d_V} \lesssim \log(\eps^{-1}) \eps^{-d_V}, \label{eq:rem:count:eq1}
    \end{align}
    \begin{align}
        K_2 H^{n_{c_W}+n_{c_U}} &\lesssim \l \eps^{-2\max\{d_W,d_U\}} \log(\eps^{-1}) \r \l \eps^{-\max\{d_W,d_U\}/2+1} \r^{\eps^{-\max\{d_W,d_U\}}} \notag \\
        &\lesssim \l \eps^{-2\max\{d_W,d_U\}} \log(\eps^{-1}) \r \eps^{-(\max\{d_W,d_U\}/2-1)\eps^{-\max\{d_W,d_U\}}} \notag \\
        &\lesssim \eps^{-d \eps^{-\max\{d_W,d_U\}}} \label{eq:rem:count:eq2}
    \end{align}
    for some $d>0$ and therefore \begin{align}
        \Vert \Theta \Vert_0 &\leq N^{d_V} H^{n_{c_W}+n_{c_U}} \lesssim  \eps^{-d \eps^{-\max\{d_W,d_U\}}} \label{eq:rem:count:eq3}.
    \end{align}
    Combining \eqref{eq:rem:count:eq1}, \eqref{eq:rem:count:eq2} and \eqref{eq:rem:count:eq3}, we obtain:
    \[
    \Vert \Theta \Vert_0 + K_1 N^{d_V} + K_2 H^{n_{c_W}+n_{c_U}} \lesssim \eps^{-\gamma \eps^{-\max\{d_W,d_U\}}}.
    \]
    Repeating the computation in \cite[Remark 3.14]{weihs2025MOL}, we can also deduce that the approximation error $\eps$ scales as follows:
    \[
    \eps \lesssim \l \frac{\log N_\#}{\log \log N_\#} \r^{-1/\max\{d_W,d_U\}}
    \]
\end{remark}

\begin{remark}[Uniform Multiple Operator Approximation in Product Spaces] \label{rem:productOperator}
Similarly to Remark \ref{rem:uniformProduct}, one can show that Proposition \ref{thm:back:functionalApproximationLinftyProductSpace} admits a uniform version for families $\{G_j:W \times U \to V\}_{j \in J}$ of multiple operator maps, even when the family is uncountable. Indeed, provided that the quantities controlling the construction (in particular the Lipschitz constants, the \(\Lp{\infty}\)-bounds) are bounded uniformly over the family, the proof yields a single collection of sampling points and a single family of neural networks that work simultaneously for all maps in the family. The dependence on the particular target map then appears only through the coefficients of the resulting approximation. In particular, with the same $\eps$-dependent classes of networks as in Proposition \ref{prop:multipleProductSpace}, one obtains \[
\sup_{j \in J} \sup_{\alpha \in W} \sup_{u \in U} \sup_{x \in \Omega_V} \left\vert G_j[\alpha][u](x) - \sum_{k=1}^{H^{n_{c_u}+n_{c_W}}} \sum_{\ell=1}^{N^{d_V}} G_j[\alpha_k][u_k](v_\ell) b_k(\bm{\alpha},\bm{u}) \tau_\ell(x)  \right\vert \leq \eps
\]
This conclusion always applies for finitely many multiple operator maps. We also note that, in the special case of families \(\{G_j:U\to V\}_{j \in J}\), the present statement reduces to \cite[Remark 3.9]{weihs2025MOL}.
\end{remark}

\begin{remark}[Generalization Bounds for the Concatenated DeepONet Architecture] \label{rem:generalization}
Similarly to the derivation of Theorem \ref{thm:scalingLawsGeneralizationError}, the approximation-complexity bounds of Proposition \ref{prop:multipleProductSpace} and Remark \ref{rem:count} can be combined with the statistical-learning framework of \cite{liu2024neuralscalinglawsdeep} to obtain generalization rates for the concatenated DeepONet architecture.

Indeed, in this setting the natural training data take the form
\[
\{(\alpha_j,u_j,\{x_i\}_{i=1}^{n_x})\}_{j=1}^n,
\]
that is, one observes \(n\) samples of pairs \((\alpha_j,u_j)\), together with evaluations of the corresponding output at sampling points \(\{x_i\}_{i=1}^{n_x}\). Repeating the proof of \cite[Theorem 2]{liu2024neuralscalinglawsdeep} for the corresponding clipped hypothesis class yields a generalization bound analogous to \eqref{eq:generalization}, expressed in terms of its covering numbers. Using Remark \ref{rem:count}, one obtains the corresponding complexity estimate, and the approximation and estimation terms can then be balanced by choosing \(\varepsilon=\varepsilon(n)\) according to the approximation complexity. This leads to a generalization error of the form
\[
\mathcal{O}\!\left(
\left(
\frac{\log n}{\log\log n}
\right)^{-2/\max\{d_W,d_U\}}
\right).
\]
\end{remark}

Combining Lemmata \ref{lem:complexityConcatenated} and \ref{lem:productCubeFromSingleCube}, Proposition \ref{prop:multipleProductSpace} and \cite[Theorem 2.11]{lanthalerStuart}, we obtain the following. 

\begin{theorem}[Minimax Bounds for Lipschitz Multiple Operator Maps using the Concatenated DeepOnet Architecture] \label{thm:minimaxDeepONet}
Let $d_W,d_U,d_V>0$ be integers, \[
\gamma_W, \gamma_U, \gamma_V, \beta_W, \beta_U,\beta_V,L_W,L_U,L_V> 0 \quad \text{and} \quad r_G \geq 1\]
and assume that $W(d_W,\gamma_W,L_W,\beta_W)$,
 $U(d_U,\gamma_U,L_U,\beta_U)$ and $V(d_V,\gamma_V,L_V,\beta_V)$ satisfy Assumption \ref{assumption:Main:assumptions:S4}. We equip the product Banach space $\Lp{r_G}(\Omega_W) \times \Lp{r_G}(\Omega_U)$ with a norm $\Vert \cdot \Vert_{\Lp{r_G}(\Omega_W) \times \Lp{r_G}(\Omega_U)}$ that satisfies Assumptions \ref{assumption:Main:assumptions:N1}, \ref{assumption:Main:assumptions:N2} and \ref{assumption:Main:assumptions:N3}.
\begin{enumerate}
\item For any $\eta > \min\left\{1 + \frac{1}{d_W},1+\frac{1}{d_U}\right\}$, $r \in \bbN$ and $\delta > 0$, there exists an $r$-times Frechet differentiable map $G:\Lp{r_G}(\Omega_W) \times \Lp{r_G}(\Omega_U) \to V$ and $\overline{\eps} := \overline{\eps}(\eta,\delta,r) > 0$ such that the following holds:
for any $\eps \leq \overline{\eps}$ and $\nn_\eps$ of the form \eqref{eq:separable2}
satisfying
\begin{equation} \label{eq:thm:minimax2:approximation}
    \sup_{\alpha\in W} \sup_{u \in U} \sup_{x \in \Omega_V} \vert \nn_\varepsilon[\alpha][u](x)-G[\alpha][u](x)\vert \leq \varepsilon,
\end{equation}
we have \[
\Vert \Theta \Vert_0 + HK_2 + N K_1 \gtrsim \exp \l c \eps^{-1/[(\eta + 1 + \delta)r]} \r
\] for some $c:=c(\eta,\delta,r) > 0$ and where $\Theta = \{\theta_{k\ell}\}$.
\item With $r = 1$, the map $G$ in part 1. can be chosen so that it satisfies Assumptions \ref{assumption:Main:assumptions:O1} and \ref{assumption:Main:assumptions:O2}. 
\item Let $\mathcal H$ denote the class of all maps $G$ satisfying Assumptions \ref{assumption:Main:assumptions:O1} and \ref{assumption:Main:assumptions:O2}. Define the worst-case/minimax approximation complexity
\begin{align*}
\mathfrak C(\varepsilon;\mathcal H)
:=
\inf\Bigl\{&
M\in\mathbb N \,\Big|\, \forall G\in\mathcal H,\ \exists \nn_\varepsilon
\text{ of the form \eqref{eq:separable2} satisfying \eqref{eq:thm:minimax2:approximation}} \\
&\text{and }\|\Theta\|_0 + HK_2 + NK_1 \le M
\Bigr\}.
\end{align*}
Let $r=1$, $\eta > \min\left\{1 + \frac{1}{d_W},1+\frac{1}{d_U}\right\}$, and $\delta > 0$.  Then, for all $0 < \varepsilon \leq \overline{\eps}(\eta,\delta,1)$ sufficiently small,
\begin{equation*} 
\exp\Bigl(c\,\varepsilon^{-1/[(\eta+1+\delta)]}\Bigr)
\lesssim
\mathfrak C(\varepsilon;\mathcal H)
\lesssim
\exp\l d \log( \eps^{-1}) \eps^{-\max\{d_W,d_U\}} \r
\end{equation*}
for some $d > 0$ only depending on $d_W$ and $d_U$.
\end{enumerate}
\end{theorem}

The minimax rates in Theorems \ref{thm:minimaxMNO} and \ref{thm:minimaxDeepONet}, as well as the generalization rates in Theorem \ref{thm:scalingLawsGeneralizationError} and Remark \ref{rem:generalization}, show that MNO and concatenated DeepONet share essentially the same scaling laws. Consequently, the present analysis does not distinguish the two architectures from the viewpoint of approximation complexity or statistical generalization.

This shifts the emphasis to empirical performance. In \cite[Section 5]{weihs2025MOL}, MNO was observed to outperform concatenated DeepONet by orders of magnitude on parametric PDE tasks. This likely reflects the fact that MNO explicitly separates the roles of the parameter variable \(\alpha\) and the input function \(u\), a distinction that is fundamental in many multiple operator learning problems \cite{weihs2026generalizationboundsstatisticalguarantees}. Consequently, the absence of a clear minimax separation provides an additional argument for designing architectures specifically adapted to the multiple operator setting.

\section{Conclusion} 

We studied approximation and statistical generalization in multiple operator learning, with a particular focus on the general separable architecture, of which MNO is a special case. While prior bounds exhibited an additional exponential blow-up relative to standard operator learning, our results showed that multiple operator learning obeys a qualitatively similar scaling to single task operator learning. In particular, our first main result showed that, by refining the approximation argument, MNO has near-optimal approximation rates and that the corresponding approximation complexity can be reduced to the same overall scale as in operator learning. Through the generalization framework of \cite{weihs2026generalizationboundsstatisticalguarantees}, this also yields improved statistical learning rates, again matching the operator-learning scale. Our second main result establishes a lower-bound theory for multiple operator learning. Extending the framework of \cite{lanthalerStuart}, we proved a curse of parametric complexity for broad classes of Lipschitz and differentiable multiple operator maps, and then specialized this abstract result to the bounded Lipschitz classes considered in our approximation theory. In this way, we obtained minimax approximation-complexity bounds showing that, although the previous constructive blow-up is not intrinsic, a genuine exponential complexity barrier remains in the worst case.

Theoretically, we compared MNO with a concatenated DeepONet-type extension to multi-task learning. From the viewpoint of minimax approximation complexity, both architectures have essentially the same scaling laws on the broad Lipschitz classes studied here. This indicates that worst-case complexity alone does not explain the empirical advantage of MNO observed in \cite{weihs2025MOL}, and suggests that the practical gains of multiple-operator-specific architectures may lie in more refined structural, geometric, or data-dependent properties that are not captured by minimax rates on generic Lipschitz classes.

Several directions remain open. On the upper-bound side, it would be natural to investigate whether stronger assumptions on the target multiple operator maps, such as holomorphic structure or PDE-specific regularity, can lead to substantially better rates. On the lower-bound side, it would be interesting to obtain sharper minimax characterizations, in particular at the level of the exponent. More broadly, one would like to establish minimax approximation-complexity rates for other classes of multiple operator maps beyond the Lipschitz/differentiable setting considered here, with the longer-term goal of identifying and characterizing a meaningful class of \emph{well-approximable} multiple operator maps. 

\section*{Acknowledgment}

This work was supported by NSF 2427558.

\bibliographystyle{plain}
\bibliography{references}{}

\appendix 

\section*{Appendix} \label{sec:proofs}

In this section, we present detailed proofs of all our results.

\section{Near-Optimal Approximation Rates} 

\begin{proof}[Proof of Theorem \ref{thm:main:improvedRates}]

For $u \in U$ and $x \in \Omega_V$, we define the functional $f_{u,x}: \cB_{\beta_W,\Vert \cdot \Vert_{\Lp{\infty}},\Omega_W}(0) \mapsto \bbR$ as \(
    f_{u,x}(\alpha) = G[\alpha][u](x). 
    \)
    In particular, $\vert f_{u,x} \vert \leq \beta_V$ and $f_{u,x}$ is Lipschitz in $\Lp{\infty}(\Omega_W)$: specifically, we have \begin{align}
        \vert f_{u,x}(\alpha_1) - f_{u,x}(\alpha_2) \vert &= \vert G[\alpha_1][u](x) - G[\alpha_2][u](x) \vert \notag \\
        &\leq L_G \Vert \alpha_1 - \alpha_2 \Vert_{\Lp{r_G}(\Omega_W)} \label{eq:multipleOperatorApproximation:Lipschitz1} \\
        &\leq L_G \vert \Omega_W \vert^{1/(r_G)} \Vert \alpha_1 - \alpha_2 \Vert_{\Lp{\infty}(\Omega_U)}. \notag 
    \end{align}
where we use Assumption \ref{assumption:Main:assumptions:O2} for \eqref{eq:multipleOperatorApproximation:Lipschitz1}. Next, we want to approximate the entire family of functionals $\{f_{u,x} : \cB_{\beta_W,\Vert \cdot \Vert_{\Lp{\infty}},\Omega_W}(0) \mapsto \bbR \}_{u \in U, \, x\in \Omega_V }$ by neural networks. Since the family is both uniformly Lipschitz and bounded by $\beta_V$, we can apply Proposition \ref{thm:back:functionalApproximationLinftyProductSpace} and Remark \ref{rem:uniformProduct} with $W \times \{\emptyset\} \cong W$ (or directly \cite[Theorem 3.6 and Remark 3.7]{weihs2025MOL}). Specifically, there exist constants \begin{itemize}
    \item $C''$ depending on $\beta_W, L_G \vert \Omega_W \vert^{1/r_G}$
    \item $C_{\zeta}$ depending on $L_G \vert \Omega_W \vert^{1/r_G},L_W$
\end{itemize}
such that the following holds. 
For any $\eps_0 >0$, 
\begin{itemize}
    \item let $\zeta:=C_{\zeta}\varepsilon_0$ and let $\{y_m\}_{m=1}^{n_{c_W}}\subset \Omega_W$ be points so that $\{\mathcal{B}_{\zeta}(y_m) \}_{ m = 1}^{n_{c_W}}$ is a cover of $\Omega_W$ for some $n_{c_W}$
    \item let $P = C'' \sqrt{n_{c_W}} \eps_0^{-1}$ and consider the network class $\cF_3 = \cF_{\rm NN}(n_{c_W},1,L_3,p_3,K_3,\kappa_3,R_3)$ whose parameters scale as \begin{align*}
    &L_3=\mathcal{O}\left(n_{c_W}^2\log(n_{c_W})+n_{c_W}^2\log(\varepsilon_0^{-1}) + n_{c_W}^2\log(2) \right),\quad  p_3 = \mathcal{O}(1),\\
    &K_3 = \mathcal{O}\left(n_{c_W}^2\log n_{c_W}+n_{c_W}^2\log(\varepsilon_0^{-1}) + n_{c_W}^2\log(2)\right), \\ &\kappa_3=\mathcal{O}(n_{c_W}^{n_{c_W}/2+1}2^{n_{c_W}+1}\varepsilon_0^{-n_{c_W}-1}),\qquad \, R_3=1
    \end{align*}
    where the constants hidden in $\mathcal{O}$ depend on $\beta_W$ and $L_G \vert \Omega_W \vert^{1/r_G}$.
\end{itemize}
Then, there exists networks $\{l_p\}_{p=1}^{P^{n_{c_W}}} \subset \cF_3$ and
functions $\{\alpha_p\}_{p=1}^{P^{n_{c_W}}} \subset \cB_{\beta_W,\Vert \cdot \Vert_{\Lp{\infty}},\Omega_W}(0)$
such that 
    \begin{align}
\sup_{\alpha \in W} \left| f_{u,x}(\alpha) - \sum_{p=1}^{P^{n_{c_W}}} f_{u,x}(\alpha_p)l_p(\bm{\alpha}) \right| = \sup_{\alpha \in W} \left| G[\alpha][u](x) - \sum_{p=1}^{P^{n_{c_W}}} G[\alpha_p][u](x) l_p(\bm{\alpha}) \right| &\leq \eps_0 \label{thm:main:improvedRates:functional}
\end{align}
where $\bm{\alpha} = (\alpha(y_1),\dots,\alpha(y_{n_{c_W}}))^\top$. We also have $0 \leq l_p \leq 1$ for $1 \leq p \leq P^{n_{c_W}}$.

By Assumption \ref{assumption:Main:assumptions:O1}, $G[\alpha_p] \in \mathcal{G}$ for all $1 \leq p \leq P^{n_{c_W}}$. In particular, this is a finite family of operators, and by applying Proposition \ref{prop:multipleProductSpace} and Remark \ref{rem:productOperator} with $W \times \{\emptyset\} \cong W$, there exist constants \begin{itemize}
    \item $C$ depending on $\gamma_V$, $L_V$
    \item $C'$ depending on $\beta_U$, $L_\mathcal{G}$, $\gamma_U$, $r_\mathcal{G}$,
    \item $C_U$ depending on $L_\mathcal{G}$, $\gamma_U$, $r_\mathcal{G}$, $L_U$
\end{itemize}
such that the following holds. For any $0< \eps_1$ sufficiently small, 
\begin{itemize}
    \item let $N := 2C\sqrt{d_V}\eps_1^{-1}$ and consider the network class $\cF_1 := \cF_{\rm NN}(d_V,1,L_1,p_1,K_1,\kappa_1,R_1)$ whose parameters scale as \begin{align*}
&L_1 = \mathcal{O}\left(d_V^2\log d_V+d_V^2(\log(\varepsilon_1^{-1})+\log(2))\right),\qquad p_1 = \mathcal{O}(1),\\
&K_1 = \mathcal{O}\left(d_V^2\log d_V+d_V^2(\log(\varepsilon_1^{-1})+\log(2))\right),\qquad \kappa_1=\mathcal{O}(d_V^{d_V/2+1}\varepsilon_1^{-(d_V+1)}2^{d_V+1}),\\
&R_1=1
    \end{align*}
    where the constants hidden in $\mathcal{O}$ depend on $\gamma_V$ and $L_V$;
        \item let $\delta_U=C_{U}\varepsilon_1$ and let $\{c_i\}_{i=1}^{n_{c_U}}\subset \Omega_U$ be points so that $\{\mathcal{B}_{\delta_U}(c_i) \}_{ i  = 1}^{n_{c_U}}$ is a cover of $\Omega_U$ for some $n_{c_U}$;
        \item let $H := 4C' \sqrt{n_{c_U}} \eps_1^{-1}$ and consider the network class $\cF_2 := \cF_{\rm NN}(n_{c_U}, 1 , L_2, p_2, K_2, \kappa_2, R_2)$ 
     with parameters scaling as
     \begin{align*}
    &L_2=\mathcal{O}\left(n_{c_U}^2\log(n_{c_U})+( n_{c_U}^2[\log(\varepsilon_1^{-1}) + 2\log(2)]\right), \qquad p_2 = \mathcal{O}(1),\\
    &K_2 = \mathcal{O}\left( n_{c_U}^2\log( n_{c_U})+ n_{c_U}^2[\log(\varepsilon_1^{-1}) + 2\log(2)]\right), \\ 
    &\kappa_2=\mathcal{O}( n_{c_U}^{n_{c_U}/2+1}\varepsilon_1^{-(n_{c_U} +1)} 2^{2(n_{c_U} +1)}),\qquad R_2=1
\end{align*}
    where the constants hidden in $\mathcal{O}$ depend on $\beta_U$, $L_\mathcal{G}$, $\gamma_U$, $r_\mathcal{G}$.
\end{itemize}
Then, there exists
\begin{itemize}
    \item networks $\{\tau_\ell\}_{k=1}^{N^{d_V}} \subset \cF_1$ and $\{b_k\}_{k=1}^{H^{n_{c_U}}} \subset  \cF_2$
    \item functions $\{u_k\}_{k=1}^{H^{n_{c_U}}} \subset \cB_{\beta_U,\Vert \cdot \Vert_{\Lp{\infty}},\Omega_U}(0)$
    \item points $\{v_\ell\}_{\ell=1}^{N^{d_V}} \subset \Omega_V$
\end{itemize}
such that \begin{equation} \label{thm:main:improvedRates:operator}
 \sup_{u \in U} \sup_{x \in \Omega_V} \left\vert G[\alpha_p][u](x) - \sum_{\ell=1}^{N^{d_V}} \sum_{k=1}^{H^{n_{c_U}}} G[\alpha_p][u_k](v_\ell) b_k(\bm{u}) \tau_\ell(x) \right\vert \leq \eps_1
\end{equation}
where $\bm{u}=(u(c_1), u(c_2),\dots,u(c_{n_{c_U}}))^\top$.

We continue by estimating as follows: \begin{align}
    &\sup_{\alpha \in W} \sup_{u \in U} \sup_{x \in \Omega_V} \left\vert G[\alpha][u](x) - \sum_{p=1}^{P^{n_{c_W}}} \sum_{k=1}^{H^{n_{c_U}}} \sum_{\ell=1}^{N^{d_V}} G[\alpha_p][u_k](v_\ell) l_p(\bm{\alpha}) b_k(\bm{u}) \tau_\ell(x) \right\vert \notag \\
    &\leq \sup_{\alpha \in W} \sup_{u \in U} \sup_{x \in \Omega_V} \left\vert G[\alpha][u](x) - \sum_{p=1}^{P^{n_{c_W}}} G[\alpha_p][u](x) l_p(\bm{\alpha}) \right\vert \notag \\
    &+\sup_{\alpha \in W} \sup_{u \in U} \sup_{x \in \Omega_V} \left\vert \sum_{p=1}^{P^{n_{c_W}}} l_p(\bm{\alpha}) \ls G[\alpha_p][u](x)  -  \sum_{k=1}^{H^{n_{c_U}}} \sum_{\ell=1}^{N^{d_V}} G[\alpha_p][u_k](v_\ell) b_k(\bm{u}) \tau_\ell(x) \rs \right\vert \notag \\
    &\leq \eps_0 + \sup_{\alpha \in W} \sup_{u \in U} \sup_{x \in \Omega_V} \sum_{p=1}^{P^{n_{c_W}}} l_p(\bm{\alpha}) \left| G[\alpha_p][u](x)  -  \sum_{k=1}^{H^{n_{c_U}}} \sum_{\ell=1}^{N^{d_V}} G[\alpha_p][u_k](v_\ell) b_k(\bm{u}) \tau_\ell(x) \right| \label{thm:main:improvedRates:eq1} \\
    &\leq \eps_0 + \eps_1 \sum_{p=1}^{P^{n_{c_W}}} l_p(\bm{\alpha}) \label{thm:main:improvedRates:eq2}
\end{align}
where we used \eqref{thm:main:improvedRates:functional} and the fact that $0 \leq l_p \leq 1$ for \eqref{thm:main:improvedRates:eq1} as well as \eqref{thm:main:improvedRates:operator} for \eqref{thm:main:improvedRates:eq2}. 

For the last term, we note that for any $0 < \eta < \beta_W$, the functional $f_\eta:\cB_{\beta_W,\Vert \cdot \Vert_{\Lp{\infty}},\Omega_W}(0) \mapsto \bbR$ defined as $f(\alpha) = \eta$ has Lipschitz constant smaller that $L_G \vert \Omega_W \vert^{1/r_G}$ and is bounded by $\beta_V$. In particular, this means that $f_\eta \cup \{f_{u,x}\}_{u \in U, \, x\in \Omega_V }$ can be jointly approximated: specifically, \eqref{thm:main:improvedRates:functional} also applies to $f_\eta$. We therefore have \begin{align}
    \eps_0 \geq
\sup_{\alpha \in W} \left|
f_\eta(\alpha) - \sum_{p=1}^{P^{n_{c_W}}} f_\eta(\alpha_p) l_p(\bm{\alpha})
\right| = \eta \sup_{\alpha \in W}
\left|
1-\sum_{p=1}^{P^{n_{c_W}}} l_p(\bm{\alpha})
\right| \notag
\end{align}
which implies that \begin{equation} \label{thm:main:improvedRates:eq3}
    \sum_{p=1}^{P^{n_{c_W}}} l_p(\bm{\alpha}) \leq 1 + \frac{\eps_0}{\eta}
\end{equation}
for all $\alpha \in W$. Setting $\eps_0 = \frac{\eps}{2}$, $\eps_1 = \frac{\eps}{2(1 + \frac{\eps}{2\eta})}$ and inserting \eqref{thm:main:improvedRates:eq3} into \eqref{thm:main:improvedRates:eq2} yields: 
\[
\sup_{\alpha \in W} \sup_{u \in U} \sup_{x \in \Omega_V} \left\vert G[\alpha][u](x) - \sum_{p=1}^{P^{n_{c_W}}} \sum_{k=1}^{H^{n_{c_U}}} \sum_{\ell=1}^{N^{d_V}} G[\alpha_p][u_k](v_\ell) l_p(\bm{\alpha}) b_k(\bm{u}) \tau_\ell(x) \right\vert \leq \eps.
\]
The final network scalings for $\cF_3$ are:
\begin{align*}
&L_3=\mathcal{O}\left(n_{c_W}^2\log(n_{c_W})+n_{c_W}^2(\log(\varepsilon^{-1}) + 2\log(2)) \right),\quad  p_3 = \mathcal{O}(1),\\
    &K_3 = \mathcal{O}\left(n_{c_W}^2\log (n_{c_W})+n_{c_W}^2(\log(\varepsilon^{-1}) + 2\log(2))\right), \\ &\kappa_3=\mathcal{O}(n_{c_W}^{n_{c_W}/2+1}2^{2(n_{c_W}+1)}\varepsilon^{-n_{c_W}-1}),\qquad \, R_3=1, \qquad P = 2C'' \sqrt{n_{c_W}} \eps^{-1}.
    \end{align*}
Noting that $\eps_1 = \frac{\eps}{2(1 + \frac{\eps}{2\eta})} = \frac{\eps \eta}{2\eta + \eps} \asymp \frac{\eps}{2}$ for sufficiently small $\eps$, the final network scalings for $\cF_2$ are:
\begin{align*}
    &L_2=\mathcal{O}\left(n_{c_U}^2\log(n_{c_U})+( n_{c_U}^2[\log(\varepsilon^{-1}) + 3\log(2)]\right), \qquad p_2 = \mathcal{O}(1),\\
    &K_2 = \mathcal{O}\left( n_{c_U}^2\log( n_{c_U})+ n_{c_U}^2[\log(\varepsilon^{-1}) + 3\log(2)]\right), \\ 
    &\kappa_2=\mathcal{O}( n_{c_U}^{n_{c_U}/2+1}\varepsilon^{-(n_{c_U} +1)} 2^{3(n_{c_U} +1)}),\qquad R_2=1, \qquad H = 8C' \sqrt{n_{c_U}} \eps^{-1}.
\end{align*}
Similarly, the final scalings for $\cF_1$ are: 
\begin{align*}
&L_1 = \mathcal{O}\left(d_V^2\log d_V+d_V^2(\log(\varepsilon^{-1})+2\log(2))\right),\qquad p_1 = \mathcal{O}(1),\\
&K_1 = \mathcal{O}\left(d_V^2\log d_V+d_V^2(\log(\varepsilon^{-1})+2\log(2))\right),\qquad \kappa_1=\mathcal{O}(d_V^{d_V/2+1}\varepsilon^{-(d_V+1)}2^{2(d_V+1)}),\\
&R_1=1, \qquad N = 4C\sqrt{d_V}\eps^{-1}.
    \end{align*}
\end{proof}

\begin{proof}[Proof of Theorem \ref{thm:scalingLawsGeneralizationError}]
In the proof, $C>0$ will denote a constant independent of $\eps$, $n_\alpha$, $n_u$, $n_x$ and $\eta$ that may change from line to line.

    The proof of \eqref{eq:generalization} follows the same argument as \cite[Theorem 3.5]{weihs2026generalizationboundsstatisticalguarantees}. Indeed, that proof does not rely on the specific architectural scalings of the network classes \(\cF_i\), but only on the fact that the chosen hypothesis class $\mathrm{Cl}_a(I,\cF_1,\cF_2,\cF_3,\{y_s\},\{c_s\},P^{n_{c_W}},H^{n_{c_U}},N^{d_V})$ contains an approximation of \(G\) with accuracy \(\varepsilon\). In the present setting, this approximation property is provided by Theorem \ref{thm:main:improvedRates} and \cite[Corollary 2.9]{weihs2026generalizationboundsstatisticalguarantees}.

    It therefore only remains to prove the learning rate. For ease of notation, we write 
\[
\cN(\eta) := \cN\l\eta,\mathrm{Cl}_a(I,\cF_1,\cF_2,\cF_3,\{y_s\},\{c_s\},P^{n_{c_W}},H^{n_{c_U}},N^{d_V}),\Vert \cdot \Vert_{\Lp{\infty}(W \times U \times \Omega_V)}\r.
\]
We start from \eqref{eq:generalization} and estimate as follows \begin{align}
&\mathbb{E}_{S_{G, \{y_s\}, \{c_s\}}} \mathbb{E}_{\alpha \sim \mu_\alpha} \, \mathbb{E}_{u \sim \mu_u} \, \mathbb{E}_{\{x_j\}_{j=1}^{n_x} \sim \mu_x^{\otimes n_x}} \left[
\frac{1}{n_x} \sum_{j=1}^{n_x} \left( G_{a,I,\cF_1,\cF_2,\cF_3,S}[\bm{\alpha}][\ub](x_j) - G[\alpha][u](x_j) \right)^2
\right] \notag  \\
    &\leq 4 \eps^2 + \eta (8 \sigma + 6) + \frac{8\sigma \eta}{\sqrt{n_\alpha n_u n_x}} \sqrt{\log \l \cN(\eta/(4\beta_V)) \r + \log(2)} + \frac{16\sigma^2}{n_\alpha n_u n_x} \l \log \l \cN\l \eta/(4\beta_V)\r \r + \log(2) \r \notag \\
    &+ \frac{112 \beta_V^2}{3 n_\alpha} \log(\cN(\eta/(4\beta_V))) \label{eq:cor:boundEps:eq1} \\
    &\lesssim 4 \eps^2 + \eta (8 \sigma + 6) + \frac{8\sigma \eta}{\sqrt{n_\alpha n_u n_x}} \sqrt{\log \l \cN(\eta/(4\beta_V)) \r} + \frac{16\sigma^2}{n_\alpha n_u n_x}  \log \l \cN\l \eta/(4\beta_V)\r \r  \label{eq:cor:boundEps:eq2}  \\
    &+ \frac{112 \beta_V^2}{3 n_\alpha} \log(\cN(\eta/(4\beta_V))) \notag
\end{align}
where we used the fact that $\cN(\eta) \leq \cN(\tilde{\eta})$ if $\tilde{\eta} \leq \eta$ for \eqref{eq:cor:boundEps:eq1}. 

We next estimate the metric entropy as a function of $\eps$. To this end, we recall \cite[Equations 99 and 100]{weihs2026generalizationboundsstatisticalguarantees}, namely \begin{equation} \label{eq:metricEntropy}
    \log(\cN(\eta)) \lesssim P^{n_{c_W}} H^{n_{c_U}} N^{d_V} \ls \log \l \frac{T}{\eta} \r + K_3 \log \l \frac{L_3 \kappa_3 T}{\eta} \r + K_2 \log \l \frac{L_2 \kappa_2 T}{\eta} \r +K_1 \log \l \frac{L_1 \kappa_1 T}{\eta} \r \rs
\end{equation}
for some $T$ satisfying \begin{equation} \label{eq:T}
    T \lesssim P^{n_{c_W}} H^{n_{c_U}} N^{d_V}  \Bigg[  L_1\kappa_1^{L_1-1} + L_2\kappa_2^{L_2-1} +  L_3\kappa_3^{L_3-1} \Bigg]
\end{equation}
Using the scaling from Theorem \ref{thm:main:improvedRates} with $\eps/2$ and recalling that $n_{c_W} \lesssim \delta_W^{-d_W} \lesssim \eps^{-d_W}$ and $n_{c_U} \lesssim \delta_U^{-d_U} \lesssim \eps^{-d_U}$ by \cite[Lemma 2]{liu2024neuralscalinglawsdeep}, we have
\begin{itemize}
    \item $N^{d_V} \lesssim \eps^{-d_V}$
    \item $H^{n_{c_U}} \lesssim \eps^{-(1+d_U/2)\eps^{-d_U}}$
    \item $P^{n_{c_W}} \lesssim \eps^{-(1+d_W/2)\eps^{-d_W}}$
    \item $\kappa_1^{L_1 - 1} \lesssim \eps^{-(d_V+1)\log(\eps^{-1})}$
    \item $\kappa_2 \lesssim \eps^{-\eps^{-d_U}(1+d_U/2)}$, $K_2\asymp L_2 \lesssim \eps^{-2d_U}\log(\eps^{-1})$ and hence $\kappa_2^{L_2 -1} \lesssim \eps^{-\eps^{-3d_U}(1+d_U/2)\log(\eps^{-1})}$
    \item $\kappa_3 \lesssim \eps^{-\eps^{-d_W}(1+d_W/2)}$, $K_3 \asymp L_3 \lesssim \eps^{-2d_W}\log(\eps^{-1})$ and hence $\kappa_3^{L_3 -1} \lesssim \eps^{-\eps^{-3d_W}(1+d_W/2)\log(\eps^{-1})}$.
\end{itemize}
Let $d = \max\{d_U,d_W\}$. We estimate as follows, starting from \eqref{eq:T}: \begin{align}
    T &\lesssim P^{n_{c_W}} H^{n_{c_U}}  \Bigg[  L_2\kappa_2^{L_2-1} +  L_3\kappa_3^{L_3-1} \Bigg] \notag \\
    &\lesssim \eps^{-2(1+d/2)\eps^{-d}}  \eps^{-\eps^{-3d}(1+d/2)\log(\eps^{-1})} \notag \\
    &\lesssim  \eps^{-\eps^{-3d}(1+d/2)\log(\eps^{-1})}. \notag
\end{align}
Using the latter, we have
\begin{align}
    \log(L_2 \kappa_2 T) &\lesssim \log \l \eps^{-\eps^{-d_U}(1+d_U/2) - 2d_U} \log(\eps^{-1}) \eps^{-\eps^{-3d}(1+d/2)\log(\eps^{-1})} \r \notag \\
    &\lesssim \log \l \eps^{-\eps^{-3d}(1+d/2)\log(\eps^{-1})} \log(\eps^{-1}) \r \notag \\
    &= \eps^{-3d}(1+d/2)\log(\eps^{-1}) + \log(\log(\eps^{-1})) \notag \\
    &\lesssim \eps^{-3d}(1+d/2)\log(\eps^{-1}) \notag
\end{align}
as well as (analogously)
\[
\log(L_3 \kappa_3 T) \lesssim \eps^{-3d}(1+d/2)\log(\eps^{-1}).
\]
Then, continuing from \eqref{eq:metricEntropy}:\begin{align}
    \log(\cN(\eta)) &\lesssim P^{n_{c_W}} H^{n_{c_U}} \ls \log(\eta^{-1})(K_2 + K_3) + K_3 \log \l L_3 \kappa_3 T \r + K_2 \log (L_2 \kappa_2 T)  \rs \notag \\
    &\lesssim \eps^{-2(1+d/2)\eps^{-d}} \Big[ \eps^{-2d} \log(\eps^{-1}) \log(\eta^{-1})  \notag \\
    &+ \eps^{-3d}(1+d/2)\log(\eps^{-1}) \l\eps^{-2d_U}\log(\eps^{-1}) +  \eps^{-2d_W}\log(\eps^{-1}) \r \Big] \notag \\
    &\lesssim \eps^{-2(1+d/2)\eps^{-d}} \l \log(\eta^{-1}) +1 \r. \label{eq:metricEntropyEstimate}
\end{align}

Picking $\eta = 4\beta_V n_\alpha^{-1}$ to balance the $\eta$-dependent and the $n_\alpha^{-1}$-terms, we insert \eqref{eq:metricEntropyEstimate} into \eqref{eq:cor:boundEps:eq2} to obtain:
\begin{align}
    &\mathbb{E}_{S_{G, \{y_s\}, \{c_s\}}} \mathbb{E}_{\alpha \sim \mu_\alpha} \, \mathbb{E}_{u \sim \mu_u} \, \mathbb{E}_{\{x_j\}_{j=1}^{n_x} \sim \mu_x^{\otimes n_x}} \left[
\frac{1}{n_x} \sum_{j=1}^{n_x} \left( G_{a,I,\cF_1,\cF_2,\cF_3,S}[\bm{\alpha}][\ub](x_j) - G[\alpha][u](x_j) \right)^2
\right] \notag  \\
&\lesssim \eps^2 + \frac{C}{n_\alpha} + \frac{C}{n_\alpha^{3/2}\sqrt{ n_u n_x}} \sqrt{\log \l \cN(n_\alpha^{-1}) \r} + \frac{C}{n_\alpha n_u n_x}  \log \l \cN\l n_\alpha^{-1} \r \r  + \frac{C}{n_\alpha} \log(\cN(n_\alpha^{-1}))) \notag \\
&\lesssim \eps^2 + \frac{C}{n_\alpha} \log(\cN(n_\alpha^{-1}))) \notag \\
&\lesssim \eps^2 + \frac{\eps^{-2(1+d/2)\eps^{-d}} \log(n_\alpha)}{n_\alpha} \notag \\
&=: T_1 + T_2. \label{eq:balance}
\end{align}

Finally, we pick $\eps(n_\alpha)$ so that $T_2$ is at most of the same order as $T_1$. We fix 
\[
\eps = \l \frac{d}{4(1+d/2)} \frac{\log n_\alpha}{ \log \log n_\alpha} \r^{-1/d}
\]
and compute \begin{align}
    \log(\eps^{-1}) &= \frac{1}{d} \log \l \frac{d}{4(1+d/2)} \frac{\log n_\alpha}{ \log \log n_\alpha} \r \lesssim \frac{1}{d} \l \log \log n_\alpha - \log \log \log n_\alpha \r \lesssim \frac{1}{d} \log \log n_\alpha. \notag
\end{align}
From the latter, \begin{align}
    \eps^{-2(1+d/2)\eps^{-d}} &= \exp \l 2(1+d/2)\eps^{-d} \log(\eps^{-1}) \r \notag \\
    &\lesssim \exp\l 2(1+d/2) \frac{d}{4(1+d/2)} \frac{\log n_\alpha}{ \log \log n_\alpha} \frac{1}{d} \log \log n_\alpha \r \notag \\
    &= n_\alpha^{1/2} \notag
\end{align}
and therefore \(
T_2 \lesssim \frac{\log n_\alpha}{n_\alpha^{1/2}}
\)
or equivalently $\log T_2 \lesssim \log \log n_\alpha - \frac{1}{2} \log n_\alpha$. For $T_1$, we have: \begin{align}
    \log T_1 &= \frac{-2}{d} \ls \log \l \frac{d}{4(1+d/2)} \r + \log \log n_\alpha - \log \log \log n_\alpha \rs \notag.
\end{align}
We note that $\lim_{n_\alpha \to \infty} \log T_2 = -\infty$ due to the $\frac{1}{2}\log n_\alpha $ term; similarly, $\lim_{n_\alpha \to \infty} \log T_1 = -\infty$ due to the $- \log \log n_\alpha$ term. This implies that $T_2$ goes to 0 much faster than $T_1$, so $T_1$ dominates the bound \eqref{eq:balance} and we conclude \begin{align}
    &\mathbb{E}_{S_{G, \{y_s\}, \{c_s\}}} \mathbb{E}_{\alpha \sim \mu_\alpha} \, \mathbb{E}_{u \sim \mu_u} \, \mathbb{E}_{\{x_j\}_{j=1}^{n_x} \sim \mu_x^{\otimes n_x}} \left[
\frac{1}{n_x} \sum_{j=1}^{n_x} \left( G_{a,I,\cF_1,\cF_2,\cF_3,S}[\bm{\alpha}][\ub](x_j) - G[\alpha][u](x_j) \right)^2
\right] \notag \\
&\lesssim \l \frac{d}{4(1+d/2)} \frac{\log n_\alpha}{ \log \log n_\alpha} \r^{-2/\max\{d_W,d_U\}}. \notag
\end{align}

\end{proof}

\section{Approximation Complexity Lower Bounds and Minimax Rates for MNO}

\begin{proof}[Proof of Theorem \ref{thm:codMNO}]
Let $r \in \bbN$ and $\delta > 0$. We start by applying \cite[Theorem 2.11]{lanthalerStuart}: there exists a $r$-times Frechet differentiable functional \(
F:W \to \mathbb{R}
\)
and $\eps_0 := \eps_0(\eta,\delta,r)$ such that for any $\eps \leq \eps_0$ and functional of neural network type $S_\eps$ with \[
\sup_{\alpha \in K} \vert F(\alpha) - S_\eps(\alpha) \vert \leq \eps,
\]
we have \begin{equation} \label{eq:thm:codMNO:complexityFunctional}
  \complexity(S_\eps) \geq \exp\l c \eps^{-1/[(\eta + 1 + \delta)r]}\r  
\end{equation} for some $c:=c(\eta,\delta,r) > 0$.

Next, we choose a nontrivial element $\phi \in V$. Then there exists $x_0 \in \Omega_V$ such that $\phi(x_0)\neq 0$; rescaling $\phi$, we may assume
\(
\phi(x_0)=1.
\)
We define the constant operator
\[
T:U \to V, \qquad T(u)=\phi
\]
and its associated multiple operator map
\[
G:W \to \{U \to V\}, \qquad G[\alpha] = F(\alpha)\,T.
\]
Since $F$ is $r$-times Frechet differentiable and $T$ is independent of $\alpha$, $G$ is also $r$-times Frechet differentiable. Also, for any fixed $u_0\in U$ and every $\alpha\in W$, we have
\begin{equation} \label{eq:thm:codMNO:pointwise}
    \ev_{x_0} \circ \ev_{u_0} \circ \, G[\alpha]
=
G[\alpha][u_0](x_0)
=
F(\alpha)\phi(x_0)
=
F(\alpha).
\end{equation}

Let $0 < \eps \leq \overline{\eps} := \eps_0/\max\{1,C_V\}$ where $C_V>0$ is such that
\begin{equation} \label{eq:thm:codMNO:embedding}
    \|v\|_{C(\Omega_V)} \le C_V \|v\|_V
\end{equation}
for all $v \in V$ due to the fact that $V \hookrightarrow C(\Omega_V)$ continuously. For $\nn_{\eps}$ a multiple operator map of neural network type satisfying \eqref{eq:thm:codMNO:approximation} with $\eps$, we estimate as follows with $u_0 \in U$:
\begin{align}
\sup_{\alpha\in K} |F(\alpha)-\ev_{x_0}\circ \ev_{u_0} \circ \nn_{\eps}[\alpha]|
&=
\sup_{\alpha\in K}
\left|
\ev_{x_0}\circ \ev_{u_0} \circ \bigl(G[\alpha]-\nn_{\eps}[\alpha]\bigr)
\right|
\label{eq:thm:codMNO:eq1} \\
&\le
C_V \sup_{\alpha\in K}
\left\|
\bigl(G[\alpha]-\nn_{\eps}[\alpha]\bigr)(u_0)
\right\|_V
\label{eq:thm:codMNO:eq2} \\
&\le
C_V \sup_{\alpha\in K}
\|G[\alpha]-\nn_{\eps}[\alpha]\|_{\op}
\notag \\
&\le C_V \eps \label{eq:thm:codMNO:eq3} \\
&\leq \eps_0 \notag
\end{align}
where we used \eqref{eq:thm:codMNO:pointwise} for \eqref{eq:thm:codMNO:eq1}, \eqref{eq:thm:codMNO:embedding}
for \eqref{eq:thm:codMNO:eq2} and \eqref{eq:thm:codMNO:approximation} for \eqref{eq:thm:codMNO:eq3}. By Definition \ref{def:multipleOperatorType}, for the fixed pair $(u_0,x_0)$ there exists a ReLU neural network $\Phi_{x_0,M_U(u_0)}$ such that
\[
\ev_{x_0}\circ \ev_{u_0}\circ \nn_{\eps}[\alpha]
=
\Phi_{x_0,M_U(u_0)}(M_W(\alpha))
\]
for all $\alpha \in W$.
Hence, by Definition \ref{def:functionalType}, the map
\(
\ev_{x_0}\circ \ev_{u_0}\circ \nn_{\eps}:W\to \mathbb R
\)
is a functional of neural network type 
and we can therefore apply \eqref{eq:thm:codMNO:complexityFunctional} to deduce that $\complexity(\ev_{x_0}\circ \ev_{u_0}\circ \nn_{\eps}) \geq \exp \l c \eps^{-1/[(\eta + 1 + \delta)r]} \r$. Using the latter, we conclude: \[
\complexity(\nn_\eps) \geq \complexity (\ev_{x_0}\circ \ev_{u_0}\circ \nn_{\eps}) \geq \exp \l c \eps^{-1/[(\eta + 1 + \delta)r]} \r.
\]

\end{proof}

\begin{proof}[Proof of Lemma \ref{lem:codMNO:symmetric}]
In the proof, $C>0$ will denote a constant that may change from line to line.

Let $r \in \bbN$ and $\delta > 0$. Without loss of generality, assume that $\eta_W \leq \eta_U$. We apply \cite[Theorem 2.11]{lanthalerStuart} to obtain a $r$-times Frechet differentiable functional \(
F:W \to \mathbb{R}
\)
and $\eps_0 := \eps_0(\eta_W,\delta,r)$ such that for any $\eps \leq \eps_0$ and functional of neural network type $S_\eps$ with \[
\sup_{\alpha \in K_W} \vert F(\alpha) - S_\eps(\alpha) \vert \leq \eps,
\]
we have \begin{equation} \label{eq:thm:codMNOSymmetric:complexityFunctional}
  \complexity(S_\eps) \geq \exp \l c \eps^{-1/[(\eta_W + 1 + \delta)r]} \r  
\end{equation} for some $c:=c(\eta_W,\delta,r) > 0$.

The rest of the proof is similar to the one of Theorem \ref{thm:codMNO}. We choose a nontrivial element $\phi \in V$. Then there exists $x_0 \in \Omega_V$ such that $\phi(x_0)\neq 0$; rescaling $\phi$, we may assume
\(
\phi(x_0)=1.
\)
We define the multiple operator map
\[
G:W \times U \to V, \qquad G[\alpha][u] = F(\alpha)\phi.
\]
 Since $F$ is $r$-times Frechet differentiable,  $G$ is also $r$-times Frechet differentiable on $W\times U$ by Assumption \ref{assumption:Main:assumptions:N1}. Also, for any fixed $u_0\in U$ and every $\alpha\in W$, we have
\begin{equation} \label{eq:thm:codMNOSymmetric:pointwise}
    \ev_{x_0} \circ \ev_{u_0} \circ \, G[\alpha]
=
G[\alpha][u_0](x_0)
=
F(\alpha)\phi(x_0)
=
F(\alpha).
\end{equation}

Let $0 < \eps \leq \overline{\eps} := \eps_0/\max\{1,C_V\}$ where $C_V>0$ is such that
\begin{equation} \label{eq:thm:codMNOSymmetric:embedding}
    \|v\|_{C(\Omega_V)} \le C_V \|v\|_V
\end{equation}
for all $v \in V$ due to the fact that $V \hookrightarrow C(\Omega_V)$ continuously. For $\nn_{\eps}$ a symmetric multiple operator map of neural network type satisfying \eqref{eq:thm:codMNOSymmetric:approximation} with $\eps$, we estimate as follows with $u_0 \in U$:
\begin{align}
\sup_{\alpha\in K_W} |F(\alpha)-\ev_{x_0}\circ \ev_{u_0} \circ \nn_{\eps}[\alpha]|
&=
\sup_{\alpha\in K_W}
\left|
\ev_{x_0}\circ \ev_{u_0} \circ \bigl(G[\alpha]-\nn_{\eps}[\alpha]\bigr) 
\right|
\label{eq:thm:codMNOSymmetric:eq1} \\
&\leq C_V \sup_{\alpha \in K_W} \Vert G[\alpha][u_0] - \nn_\eps[\alpha][u_0] \Vert_V \label{eq:thm:codMNOSymmetric:eq2} \\
&\leq \eps_0 \label{eq:thm:codMNOSymmetric:eq3}
\end{align}
where we used \eqref{eq:thm:codMNOSymmetric:pointwise} for \eqref{eq:thm:codMNOSymmetric:eq1}, \eqref{eq:thm:codMNOSymmetric:embedding} for \eqref{eq:thm:codMNOSymmetric:eq2} and \eqref{eq:thm:codMNOSymmetric:approximation} for \eqref{eq:thm:codMNOSymmetric:eq3}. By Definition \ref{def:multipleOperatorTypeSymmetric}, for the fixed pair $(u_0,x_0)$ there exists a ReLU neural network $\Phi_{x_0,M_U(u_0)}$ such that
\[
\ev_{x_0}\circ \ev_{u_0}\circ \nn_{\eps}[\alpha]
=
\Phi_{x_0,M_U(u_0)}(M_W(\alpha))
\]
for all $\alpha \in W$.
Hence, by Definition \ref{def:functionalType}, the map
\(
\ev_{x_0}\circ \ev_{u_0}\circ \nn_{\eps}:W\to \mathbb R
\)
is a functional of neural network type
and we can therefore apply \eqref{eq:thm:codMNOSymmetric:complexityFunctional} to deduce that $\complexity(\ev_{x_0}\circ \ev_{u_0}\circ \nn_{\eps}) \geq \exp \l c \eps^{-1/[(\eta_W + 1 + \delta)r]} \r$. Using the latter, we conclude: \[
\complexity(\nn_\eps) \geq \complexity (\ev_{x_0}\circ \ev_{u_0}\circ \nn_{\eps}) \geq \exp \l c \eps^{-1/[(\eta_W + 1 + \delta)r]} \r.
\]

\end{proof}

The following result will be essential in the proof of Lemma \ref{lem:boundedLipCube}.

\begin{lemma}[$\Lp{r}$-norm of $\sin$] \label{lem:sin}
For every $1 \leq r < \infty$ and every nonzero multi-index $\kappa \in \mathbb N^d$, one has
\[
\|\sin(\kappa \cdot x)\|_{\Lp{r}([0,2\pi]^d)}
=
\left((2\pi)^{d-1}\int_0^{2\pi} |\sin t|^r \,\dd t\right)^{1/r}.
\]
In particular, the quantity $c_r := \|\sin(\kappa \cdot x)\|_{\Lp{r}([0,2\pi]^d)}$ is independent of $\kappa$.
\end{lemma}

\begin{proof}
Let $\kappa=(\kappa_1,\dots,\kappa_d)\in \mathbb N^d$ be a non-zero multi-index. Pick an index $j\in\{1,\dots,d\}$ such that $\kappa_j\neq 0$. Writing $x=(x_j,x')$, where $x'=(x_1,\dots,x_{j-1},x_{j+1},\dots,x_d)$, we have
\begin{align}
\int_{[0,2\pi]^d} |\sin(\kappa\cdot x)|^r \,\dd x
&=
\int_{[0,2\pi]^{d-1}}
\int_0^{2\pi}
\left|
\sin\!\left(\kappa_j x_j + \sum_{m\neq j}\kappa_m x_m\right)
\right|^r
\,\dd x_j\,\dd x' \notag \\
&= \int_{[0,2\pi]^{d-1}}
 \frac{1}{\kappa_j}
\int_{\sum_{m\neq j}\kappa_m x_m}^{\sum_{m\neq j}\kappa_m x_m+2\pi \kappa_j}
|\sin t|^r \,\dd t \,\dd x' \label{eq:lem:sin:eq1} \\
&= \int_{[0,2\pi]^{d-1}}
 \frac{1}{\kappa_j}
\sum_{n=0}^{\kappa_j-1} \int_{\sum_{m\neq j}\kappa_m x_m + 2\pi n }^{\sum_{m\neq j}\kappa_m x_m+ 2\pi (n+1)} 
|\sin t|^r \,\dd t \,\dd x' \notag \\
&= \int_{[0,2\pi]^{d-1}}
 \frac{\kappa_j}{\kappa_j}
\int_{0}^{2\pi}
|\sin t|^r \,\dd t \,\dd x' \label{eq:lem:sin:eq2} \\
&= (2\pi)^{d-1}\int_0^{2\pi} |\sin t|^r \,\dd t \notag
\end{align}
where we used the change of variables
\(
t=\kappa_j x_j + \sum_{m\neq j}\kappa_m x_m
\) for \eqref{eq:lem:sin:eq1} and the fact that $|\sin t|^r$ is $2\pi$-periodic for \eqref{eq:lem:sin:eq2}.
Taking the $r$-th root yields the claim.
\end{proof}

\begin{proof}[Proof of Lemma \ref{lem:boundedLipCube}]
In the proof $C>0$ denotes a constant that may change from line to line, independently of the summation indices and $\kappa$.

Up to an affine rescaling of the domain, it suffices to treat the case
\(
\Omega_U=[0,2\pi]^{d_U}.
\)
For each multi-index $\kappa\in \mathbb N^{d_U}$, define
\[
\tilde e_\kappa(x) := \sin(\kappa\cdot x),
\qquad
e_\kappa(x):=\frac{\tilde e_\kappa(x)}{\|\tilde e_\kappa\|_{\Lp{r}(\Omega_U)}} = \frac{\tilde e_\kappa(x)}{c_r}
\]
where we used Lemma \ref{lem:sin} for the last equality. 
These functions are orthogonal in $\Lp{2}(\Omega_U)$, and hence linearly independent in $\Lp{r}(\Omega_U)$, with
\(
\|e_\kappa\|_{\Lp{r}(\Omega_U)} =1.
\)
For each $\kappa \in \bbN^{d_U}$, define $e^*_\kappa \in \Lp{r}(\Omega_U)^*$ as
\[
e_\kappa^*(u)
:=
\frac{2 c_r}{(2\pi)^d}\int_{[0,2\pi]^d} u(x)\sin(\kappa\cdot x)\, \dd x.
\]
By the identity $\sin(a)\sin(b) = \frac{1}{2}(\cos(a-b) - \cos(a+b))$, it is straightforward to check that $e_\kappa^*(e_{\kappa'})=\delta_{\kappa\kappa'}$. Moreover, by H\"older's inequality with $r'$ such that $1/r + 1/r' = 1$, 
\[
|e_\kappa^*(u)|
\le
\frac{2 c_r c_{r'}}{(2\pi)^d}
\|u\|_{\Lp{r}(\Omega_U)}
\le C \|u\|_{\Lp{r}(\Omega_U)},
\]
so
\(
\|e_\kappa^*\|_{\Lp{r}(\Omega_U)^*} \le C
\)
uniformly in $\kappa$.

Next, we enumerate $\{e_\kappa\}_{\kappa \in \bbN^{d_U}}$ as $\{e_j\}_{j=1}^\infty$ in such a way that $j\mapsto \Vert \kappa(j) \Vert_{\infty}$ is non-decreasing and consider functions on $[0,2\pi]^{d_U}$ of the form
\[
u = A \sum_{j=1}^\infty j^{-\eta} y_j e_j,
\qquad y_j\in [0,1].
\]

First, since $\|e_j\|_{\Lp{\infty}(\Omega_U)} = c_r^{-1}$,
\(
\|u\|_{\Lp{\infty}(\Omega_U)}
\le
Ac_r^{-1}\sum_{j=1}^\infty j^{-\eta}
\)
and the latter converges because $\eta>1$. Choosing $A(\eta)>0$ sufficiently small furthermore ensures that
\(
\|u\|_{\Lp{\infty}(\Omega_U)} \leq \beta_U
\)
uniformly. 

Second, we estimate the Lipschitz constant. For a function $f:[0,2\pi]^{d_U} \mapsto \bbR$, we define \[
\textrm{Lip}(f) = \inf \{L > 0 \,\mid \, \vert f(x) - f(y) \vert \leq L \Vert x - y \Vert_2 \text{ for all $x,y$}\}
\]
and note that that $\textrm{Lip}(f_1 + f_2)\leq \textrm{Lip}(f_1) + \textrm{Lip}(f_2)$ as well as $\textrm{Lip}(f) \leq C \Vert \nabla f \Vert_{\Lp{\infty}(\Omega_U)}$ by the mean-value theorem.
By noting that 
\(
\Lip(e_\kappa) \leq C \Vert \nabla e_\kappa \Vert_{\Lp{\infty}(\Omega_U)} \leq C \Vert \kappa \Vert_{\infty},
\)
we estimate as follows:
\begin{equation} \label{eq:hypercudeBoundedLipschitz:eq1}
\Lip(u)
\le
A \sum_{j=1}^\infty j^{-\eta}\Lip(e_j)
\leq C
A \sum_{j=1}^\infty j^{-\eta} \Vert \kappa(j) \Vert_{\infty}.  
\end{equation}
Let us now consider the inverse enumeration $\kappa \mapsto j(\kappa)$. For a fixed $\kappa \in \bbN^{d_U}$, we write $K = \Vert \kappa \Vert_\infty$ and, due to the fact that $j \mapsto \kappa(j)$ is non-decreasing, we have 
\begin{equation} \label{eq:hypercudeBoundedLipschitz:eq2}
  (K-1)^{d_U} \leq \vert \{\kappa' \in \bbN^{d_U} \, \mid \, \Vert \kappa' \Vert_\infty < K \} \vert \leq j(\kappa) \leq \vert \{\kappa' \in \bbN^{d_U} \, \mid \, \Vert \kappa' \Vert_\infty \leq K \} \vert = K^{d_U}
\end{equation}
as well as \begin{equation} \label{eq:hypercudeBoundedLipschitz:eq5}
    \vert \{\kappa' \in \bbN^{d_U} \, \mid \, \Vert \kappa' \Vert_\infty = K \} \vert \leq d_U K^{d_U -1},
\end{equation}
since fixing one coordinate equal to $K$ leaves at most $K^{d_U-1}$ choices for the remaining coordinates, and there are $d_U$ possible coordinates that may attain the maximum.
We continue our estimation: \begin{align}
\sum_{j=1}^\infty j^{-\eta} \Vert \kappa(j) \Vert_{\infty} &=  \sum_{K=1}^\infty \sum_{\Vert \kappa(j) \Vert_\infty = K} j(\kappa)^{-\eta} K \notag \\
&= \sum_{\Vert \kappa(j) \Vert_\infty = 1} j(\kappa)^{-\eta} + \sum_{K=2}^\infty \sum_{\Vert \kappa(j) \Vert_\infty = K} j(\kappa)^{-\eta} K \notag \\
&\leq1 + \sum_{K=2}^\infty \sum_{\Vert \kappa(j) \Vert_\infty = K} K^{1}(K-1)^{-\eta d_U} \label{eq:hypercudeBoundedLipschitz:eq3}\\
&\leq C \sum_{K=2}^\infty K^{d_U} (K-1)^{-\eta d_U}\label{eq:hypercudeBoundedLipschitz:eq4} \\
&\leq C \sum_{K=2}^\infty K^{d_U(1-\eta)}\label{eq:hypercudeBoundedLipschitz:eq6}
\end{align}
where we used \eqref{eq:hypercudeBoundedLipschitz:eq2} for \eqref{eq:hypercudeBoundedLipschitz:eq3}, \eqref{eq:hypercudeBoundedLipschitz:eq5} for \eqref{eq:hypercudeBoundedLipschitz:eq4} and the fact that $K/2 \leq K-1$ for $K \geq 2$ for \eqref{eq:hypercudeBoundedLipschitz:eq6}. The latter quantity converges since $\eta > 1 + \frac{1}{d_U}$ and, from \eqref{eq:hypercudeBoundedLipschitz:eq1}, we therefore deduce that $A(\eta)$ can be chosen so that $\textrm{Lip}(u) \leq L_U$ uniformly. This concludes the proof that $U$ contains an infinite-dimensional hypercube.
\end{proof}

\begin{proof}[Proof of Corollary \ref{cor:codMNO}]

    Without loss of generality, assume that $d_W \geq d_U$ and pick $\eta > 1 + \frac{1}{d_W}$.
    
    First, we verify that $W$ satisfying Assumption \ref{assumption:Main:assumptions:S4} is a compact subset of a Banach space and contains an infinite-dimensional cube $Q_\eta$: the latter part of this claim is given by Lemma \ref{lem:boundedLipCube} when the Banach space is chosen to be $\Lp{r_G}(\Omega_W)$ and $\eta > 1 + 1/d_W$. To show the former part, we start by assuming that $\{\alpha_i\}_{i=1}^\infty \subset W$. These functions are uniformly bounded and, since
    \[
    \vert \alpha_i(x) - \alpha_i(y) \vert \leq L_W \Vert x - y \Vert,
    \]
    also equicontinuous. Indeed, for every $\eps >0$, $\Vert x - y \Vert \leq \frac{\eps}{L_W}$ implies that $\vert \alpha_i(x) - \alpha_i(y) \vert < \eps$ for every $i \in \bbN$ and $x,y \in \Omega_W$. By the Arzela-Ascoli theorem, there therefore exists a subsequence $\{\alpha_{i_k}\}_{k=1}^\infty$ converging in $\Ck{0}(\Omega_W)$ (and hence in $\Lp{r_G}(\Omega_W)$) to some $\alpha \in \Lp{r_G}(\Omega_W)$. To conclude compactness of $W$ in $\Lp{r_G}(\Omega_W)$, we need to show that $\alpha \in W$. First, we have
    \[
    \Vert \alpha \Vert_{\Ck{0}(\Omega_W)} \leq \Vert \alpha - \alpha_{i_k} \Vert_{\Ck{0}(\Omega_W)} + \Vert \alpha_{i_k} \Vert_{\Ck{0}(\Omega_W)} \leq \Vert \alpha - \alpha_{i_k} \Vert_{\Ck{0}(\Omega_W)} + \beta_W
    \]
    and taking the limit implies that $\Vert \alpha \Vert_{\Ck{0}(\Omega_W)} = \Vert \alpha \Vert_{\Lp{\infty}(\Omega_W)} \leq \beta_W$. Second, we note that \begin{align*}
        \vert \alpha(x) - \alpha(y) \vert &\leq \vert \alpha(x) - \alpha_{i_k}(x) \vert + \vert  \alpha_{i_k}(x) - \alpha_{i_k}(y) \vert + \vert \alpha(y) - \alpha_{i_k}(y) \vert \\
        &\leq L_W \Vert x - y \Vert + \vert \alpha(x) - \alpha_{i_k}(x) \vert + \vert \alpha(y) - \alpha_{i_k}(y) \vert
    \end{align*}
    and taking the limit shows that $\alpha$ is $L_W$-Lipschitz, and hence in $W$.

    Second, by Remark \ref{rem:banachSpace}, Lemma \ref{lem:codMNO:symmetric} remains valid when the norm is chosen to be 
\(
\sup_{\alpha\in K_W} \sup_{u \in K_U} \| \cdot \|_{\Lp{\infty}}
\)
and in this formulation, the output class $V$ need not be a Banach space. In particular, we can pick it to satisfy Assumption \ref{assumption:Main:assumptions:S4} and the claim of the corollary follows after an application of Lemma \ref{lem:codMNO:symmetric}.
\end{proof}

\begin{proof}[Proof of Lemma \ref{lem:complexitySeparable}]
    Let $u \in U$ and $x \in \Omega_V$. Then, \begin{align*}
            \ev_x  \circ \ev_u \circ \nn_W[\alpha] &= \sum_{p=1}^{P}\sum_{k=1}^{H}\sum_{\ell=1}^{N}
    \theta_{pk\ell}\, l_p(M_W(\alpha))\, b_k(M_U(u))\, \tau_\ell(x) \\
    &= \tilde{\nn}(M_U(u))(x)^\top l\l {M_W(\alpha)} \r \\
    &=: \Phi_{x,M_U(u)}(M_W(\alpha))
    \end{align*}
where $\nn_W[\alpha] = \nn[\alpha][\cdot](\cdot)$, $l: \bbR^m \mapsto \bbR^P$ is the parallelization of $\{l_p\}_{p=1}^P$ \cite[Definition 2.7]{PETERSEN2018296}, i.e. the network of the form \eqref{eq:back:reluNN} such that $ l(y) = (l_1(y),\dots,l_P(y))$, and $\tilde{\nn}(M_U(u))(x) \in \bbR^P$ is a vector with entries  
$$\sum_{k=1}^H \sum_{\ell=1}^N \theta_{pkl} b_k(M_U(u)) \tau_\ell(x)$$
for $1 \leq p \leq P$. Therefore, $\ev_x  \circ \ev_u \circ \nn[\alpha] = \Phi_{x,M_U(u)}(M_W(\alpha))$ is an inner product between a ReLU-neural network and fixed vector which, by \eqref{eq:back:reluNN}, can easily be checked to be a ReLU-neural network as in \eqref{eq:back:reluNN}. 

The same argument can be applied with $\alpha \in W$ and $x \in \Omega_V$ to deduce that $\ev_x  \circ \ev_\alpha \circ \nn_U[u] = \Psi_{x,M_W(\alpha)}(M_U(u))$ for some ReLU-neural network $\Psi_{x,M_W(\alpha)}$ and where $\nn_U[u] = \nn[\cdot][u](\cdot)$. We therefore conclude that $\nn$ is a symmetric multiple operator map of neural network type. 

For complexity, we have \begin{align}
    \cC(\nn) &\leq \max \left\{ \sup_{x \in \Omega_V} \sup_{u \in U} N_\# \l \Phi_{x,M_U(u)} \r, \sup_{x \in \Omega_V} \sup_{\alpha \in W} N_\# \l \Psi_{x,M_W(\alpha)} \r  \right\} \label{eq:complexityMNO:eq2}.
\end{align}
First, we assume that \[
\max \left\{ \sup_{x \in \Omega_V} \sup_{u \in U} N_\# \l \Phi_{x,M_U(u)} \r, \sup_{x \in \Omega_V} \sup_{\alpha \in W} N_\# \l \Psi_{x,M_W(\alpha)} \r  \right\} = \sup_{x \in \Omega_V} \sup_{u \in U} N_\#\l \tilde{\nn}(M_U(u))(x)^\top l \r,
\]
we can continue from \eqref{eq:complexityMNO:eq2} and estimate as follows:
\begin{align}
    \cC(\nn) &\leq \sup_{x \in \Omega_V} \sup_{u \in U} N_\#\l \tilde{\nn}(M_U(u))(x)^\top l \r \notag \\
    &\leq \sup_{x \in \Omega_V} \sup_{u \in U} 2 \Vert \tilde{\nn}(M_U(u))(x)^\top \Vert_0 + 2 N_\# (l) \label{eq:complexityMNO:eq1}
\end{align}
where we used \cite[Remark 2.6]{PETERSEN2018296} for \eqref{eq:complexityMNO:eq1}. Now, if for some $1\leq p \leq P$, $\sum_{k=1}^H \sum_{\ell=1}^N \theta_{pkl} b_k(M_U(u)) \tau_\ell(x) \neq 0$, there exists some triple $(p,k,\ell)$ such that $\theta_{pkl} \neq 0$, $b_k(M_U(u)) \neq 0$ and $\tau_\ell(x) \neq 0$. This implies that $b_k$ and $\tau_\ell$ are nonzero networks from which we deduce that at least one coefficient in each has to be different from 0. In particular, this yields \begin{align*}
\Vert \tilde{\nn}(M_U(u))(x)^\top \Vert_0 &\leq \Vert \Theta \Vert_0  + \sum_{k=1}^H N_\#(b_k) + \sum_{\ell=1}^N N_\#(\tau_\ell)
\end{align*}
where $\Theta = \{\theta_{pk\ell}\}$.
Using the latter and \cite[Definition 2.7]{PETERSEN2018296}, continuing from \eqref{eq:complexityMNO:eq1}, we obtain: \begin{align}
    \cC(\nn) &\leq 2 \l \Vert \Theta \Vert_0 + HK_2 + N K_1 + \sum_{p=1}^P N_\#(l_p) \r \leq 2 (\Vert \Theta \Vert_0 + HK_2 + N K_1 + P K_3). \notag
\end{align}
If one assumes that 
\[
\max \left\{ \sup_{x \in \Omega_V} \sup_{u \in U} N_\# \l \Phi_{x,M_U(u)} \r, \sup_{x \in \Omega_V} \sup_{\alpha \in W} N_\# \l \Psi_{x,M_W(\alpha)} \r  \right\} = \sup_{x \in \Omega_V} \sup_{\alpha \in W} N_\# \l \Psi_{x,M_W(\alpha)} \r
\]
instead, the same upper bound on complexity can be attained analogously. 

\end{proof}

\begin{proof}[Proof of Theorem \ref{thm:minimaxMNO}]
    \begin{enumerate}
        \item The first claim of the theorem is a combination of Corollary \ref{cor:codMNO} and Lemma \ref{lem:complexitySeparable}: indeed, the architecture is symmetric multiple operator map of neural network type and its complexity is upper bounded by $2(\Vert \Theta \Vert_0 + HK_2 + N K_1 + P K_3)$.
 \item From part 1, we may take $r=1$ to obtain $G:\Lp{r_G}(\Omega_W) \times \Lp{r_G}(\Omega_U) \to V$ which is Frechet differentiable on $\Lp{r_G}(\Omega_W) \times \Lp{r_G}(\Omega_U)$. Specifically, from the proof of Corollary \ref{cor:codMNO}, we know that 
\[
G[\alpha][u](x)=F(\alpha)\phi(x),
\]
where $F:\Lp{r_G}(\Omega_W)\to \mathbb R$ is the Frechet differentiable functional provided by \cite[Theorem 2.11]{lanthalerStuart}, and $\phi\in V$ is a fixed nontrivial function.

Next, the proof of \cite[Lemma A.7]{lanthalerStuart} shows that
\(
\sup_{\alpha\in \Lp{r_G}(\Omega_W)} \|DF(\alpha)\|_{\Lp{r_G}(\Omega_W)^*}<\infty.
\)
Hence, $F$ is Lipschitz on $\Lp{r_G}(\Omega_W)$ with Lipschitz constant $\textrm{Lip}(F)$.
It remains to deduce the two required Lipschitz properties. First, for fixed $\alpha$ and any $u_1,u_2$,
we have \[
\Vert G[\alpha][u_1] - G[\alpha][u_2]\Vert_{\Lp{\infty}(\Omega_V)} = \Vert F(\alpha)\phi-F(\alpha)\phi \Vert_{\Lp{\infty}(\Omega_V)} = 0 \leq L_{\mathcal{G}}\Vert u_1 - u_2 \Vert_{\Lp{r_\mathcal{G}}(\Omega_U)}
\]
for any $L_{\mathcal G}>0$ and $r_{\mathcal{G}} \geq 1$. Second, for $\alpha_1,\alpha_2$ and any $u$, we estimate as follows:
\begin{align*}
\|G[\alpha_1]-G[\alpha_2]\|_{\Lp{\infty}(\{u:\|u\|_{\Lp{\infty}(\Omega_U)}\le \beta_U\}\times \Omega_V)}
&=
|F(\alpha_1)-F(\alpha_2)|\,\|\phi\|_{\Lp{\infty}(\Omega_V)} \\
&\le
\Lip(F) \beta_V \|\alpha_1-\alpha_2\|_{\Lp{r_G}(\Omega_W)} \\
&=: L_G \|\alpha_1-\alpha_2\|_{\Lp{r_G}(\Omega_W)}
\end{align*}
where we used the fact that $F$ in Lipschitz in $\Lp{r_G}$ and the fact that $\phi \in V$ for the inequality. This concludes the proof. 
\item The lower bound in \eqref{eq:thm:minimax:minimax} is given by combining parts 1 and 2 of the theorem. The upper bound is a direct consequence of Theorem \ref{thm:main:improvedRates}.
    \end{enumerate}
\end{proof}

\section{An Extension of DeepONet to Multi-Task Learning}

\begin{proof}[Proof of Lemma \ref{lem:complexityConcatenated}]
Let \((\alpha,u)\in W\times U\) and \(x\in\Omega_V\). Then
\begin{align*}
\ev_x \circ \nn[\alpha][u]
&=
\sum_{k=1}^{H}\sum_{\ell=1}^{N}
\theta_{k\ell}\, b_k(M_W(\alpha),M_U(u))\, \tau_\ell(x) \\
&=
\sum_{k=1}^{H}
\left(\sum_{\ell=1}^{N}\theta_{k\ell}\tau_\ell(x)\right)
b_k(M_W(\alpha),M_U(u)) \\
&=: \tilde\tau(x)^\top b(M_W(\alpha),M_U(u)) \\
&=: \Phi_x(M_W(\alpha),M_U(u)),
\end{align*}
where \(b:\mathbb R^{m+q}\to\mathbb R^H\) is the parallelization of \(\{b_k\}_{k=1}^H\) \cite[Definition 2.7]{PETERSEN2018296}, i.e. the network of the form \ref{eq:back:reluNN} such that
\(
b(y,z)=\bigl(b_1(y,z),\dots,b_H(y,z)\bigr),
\)
and \(\tilde\tau(x)\in\mathbb R^H\) is the vector with entries
\[
\tilde\tau_k(x):=\sum_{\ell=1}^{N}\theta_{k\ell}\tau_\ell(x)
\]
for $1\le k\le H$. Therefore, for every \(x\in\Omega_V\), the map
\(
(\alpha,u)\mapsto \ev_x \nn[\alpha][u]
\)
is the inner product between the ReLU neural network \(b\) evaluated at the concatenated linear encoding \((M_W(\alpha),M_U(u))\) and the fixed vector \(\tilde\tau(x)\). Hence, by \eqref{eq:back:reluNN}, \(\Phi_x\) is itself a ReLU neural network, and \(\nn\) is an operator map of neural network type from \(W\times U\) to \(V\).

For complexity, by Definition \ref{def:operatorType},
\begin{align}
\cC(\nn)
&\leq \sup_{x\in\Omega_V} N_\#\bigl(\tilde\tau(x)^\top b\bigr) \notag \\
&\leq \sup_{x\in\Omega_V} \left( 2\|\tilde\tau(x)^\top\|_0 + 2N_\#(b) \right) \label{eq:complexityConcatGen:eq2}
\end{align}
where we used \cite[Remark 2.6]{PETERSEN2018296} for \eqref{eq:complexityConcatGen:eq2}. Now, if for some \(1\le k\le H\), \(\tilde\tau_k(x)\neq 0\), then
\(
\sum_{\ell=1}^{N}\theta_{k\ell}\tau_\ell(x)\neq 0.
\)
Hence there exists some \(\ell\in\{1,\dots,N\}\) such that \(\theta_{k\ell}\neq 0\) and \(\tau_\ell(x)\neq 0\). This implies that \(\tau_\ell\) is a nonzero network from which we deduce that at least one coefficient of \(\tau_\ell\) is different from \(0\). Therefore,
\[
\|\tilde\tau(x)^\top\|_0
\leq
\|\Theta\|_0 + \sum_{\ell=1}^{N} N_\#(\tau_\ell).
\]
Using the latter estimate and \cite[Definition 2.7]{PETERSEN2018296}, we continue from \eqref{eq:complexityConcatGen:eq2} and obtain
\begin{align*}
\cC(\nn)
&\leq
2\left(\|\Theta\|_0 + \sum_{\ell=1}^{N} N_\#(\tau_\ell) + N_\#(b)\right) \\
&\leq
2\left(\|\Theta\|_0 + N K_1 + \sum_{k=1}^{H} N_\#(b_k)\right) \\
&\leq
2\bigl(\|\Theta\|_0 + N K_1 + H K_2\bigr).
\end{align*}
\end{proof}

\begin{proof}[Proof of Lemma \ref{lem:productCubeFromSingleCube}]
We start by defining
\(
\tilde e_j := (e_j,0)\in W\times \{0\}
\)
for $j \in \bbN$. We verify the conditions in Definition \ref{def:cube}.

First, the family \(\{\tilde e_j\}_{j\in\mathbb N}\) is linearly independent. Indeed, if
\(
\sum_{j=1}^m a_j \tilde e_j = 0
\)
for some \(m\in\mathbb N\) and scalars \(a_1,\dots,a_m\), then
\(
\sum_{j=1}^m a_j e_j = 0
\)
in $W$. 
Since \(\{e_j\}_{j\in\mathbb N}\) is linearly independent, it follows that \(a_j=0\) for all \(j\).

Second, we have
\(
\|\tilde e_j\|_{W\times U}
=
\|(e_j,0)\|_{W\times U}
=
\|e_j\|_W
=
1.
\)

Third, we check the cube inclusion. Let \(y_j\in[0,1]\) for all \(j\), and consider
\[
\tilde w
:=
A\sum_{j=1}^\infty j^{-\eta} y_j \tilde e_j
=
A\sum_{j=1}^\infty j^{-\eta} y_j (e_j,0)
=
\left(
A\sum_{j=1}^\infty j^{-\eta} y_j e_j,\,
0
\right).
\]
Since \(Q_\eta(A;\{e_j\}_{j\in\mathbb N})\subset K_W\), the first component belongs to \(K_W\). Therefore
\(
\tilde w \in K_W\times \{0\}.
\)

Fourth, we construct a bounded biorthogonal sequence in \((W\times U)^*\). Let \(\{e_j^*\}_{j\in\mathbb N}\subset W^*\) be a bounded biorthogonal sequence for \(\{e_j\}_{j\in\mathbb N}\), and define
\(
\tilde e_j^*(w,u):=e_j^*(w)
\)
for $(w,u)\in W\times U$.
Then \(\tilde e_j^*\in (W\times U)^*\), and for all \(j,k\in\mathbb N\),
\[
\tilde e_j^*(\tilde e_k)
=
\tilde e_j^*(e_k,0)
=
e_j^*(e_k)
=
\delta_{jk}.
\]
Moreover, \(\{\tilde e_j^*\}_{j\in\mathbb N}\) is uniformly bounded, since by Assumption \ref{assumption:Main:assumptions:N1} we have
\[
|\tilde e_j^*(w,u)|
=
|e_j^*(w)|
\le
\|e_j^*\|_{W^*}\|w\|_W
\le C_{\textrm{prod}}
\|e_j^*\|_{W^*}\,\|(w,u)\|_{W\times U}
\]
which implies
\(
\|\tilde e_j^*\|_{(W\times U)^*}\le  C_{\textrm{prod}} \|e_j^*\|_{W^*}.
\)
This concludes the proof. 
\end{proof}

The following proposition is an essential building block for the proof of Proposition \ref{prop:multipleProductSpace}.

\begin{proposition}[Functional Approximation in Product Spaces]\label{thm:back:functionalApproximationLinftyProductSpace}
    Let $d_W,d_U>0$ be integers, \[
\gamma_W, \gamma_U, \beta_W, \beta_U,L_W,L_U> 0
\]
and assume that $W(d_W,\gamma_W,L_W,\beta_W)$, 
 $U(d_U,\gamma_U,L_U,\beta_U)$ satisfy Assumption \ref{assumption:Main:assumptions:S4}. We equip the product space $W \times U$ with a norm $\Vert \cdot \Vert_{W \times U}$ that satisfies Assumption \ref{assumption:Main:assumptions:N2}. 
    Let $f:\cB_{\beta_W,\Vert \cdot \Vert_{\Lp{\infty}},\Omega_W}(0) \times \cB_{\beta_U,\Vert \cdot \Vert_{\Lp{\infty}},\Omega_U}(0) \mapsto \bbR$
    be a functional such that \begin{equation} \label{eq:prop:productNorm}
            \vert f((\alpha_1,u_1)) - f((\alpha_2,u_2)) \vert \leq L_f \Vert (\alpha_1,u_1) - (\alpha_2,u_2) \Vert_{W \times U}
    \end{equation}
    for all $(\alpha_1,u_1), (\alpha_2,u_2) \in \cB_{\beta_W,\Vert \cdot \Vert_{\Lp{\infty}},\Omega_W}(0) \times \cB_{\beta_U,\Vert \cdot \Vert_{\Lp{\infty}},\Omega_U}(0)$.
There exist constants \begin{itemize}
    \item $C$ depending on $\beta_W$, $\beta_U$, $C_{\textrm{prod}}$, $L_f$
    \item $C_{W}$ depending on $C_{\mathrm{prod}}$, $L_f$, $L_W$
    \item $C_U$ depending on $C_{\mathrm{prod}}$, $L_f$, $L_U$
\end{itemize} 
such that the following holds. For any $\varepsilon>0$, \begin{itemize}
        \item let $\delta_W=C_{W}\varepsilon$ and let $\{a_i\}_{i=1}^{n_{c_W}}\subset \Omega_W$ be points so that $\{\mathcal{B}_{\delta_W}(a_i) \}_{ i  = 1}^{n_{c_W}}$ is a cover of $\Omega_W$ for some $n_{c_W}$;
        \item let $\delta_U=C_{U}\varepsilon$ and let $\{c_i\}_{i=1}^{n_{c_U}}\subset \Omega_U$ be points so that $\{\mathcal{B}_{\delta_U}(c_i) \}_{ i  = 1}^{n_{c_U}}$ is a cover of $\Omega_U$ for some $n_{c_U}$;
        \item let $H = 2C \sqrt{n_{c_W} + n_{c_U}} \eps^{-1}$ and consider the network class $\cF_{\rm NN}(n_{c_W} + n_{c_U}, 1 , L, p, K, \kappa, R)$ 
     with parameters scaling as
     \begin{align*}
    &L=\mathcal{O}\left((n_{c_W} + n_{c_U})^2\log(n_{c_W} + n_{c_U})+(n_{c_W} + n_{c_U})^2[\log(\varepsilon^{-1}) + \log(2)]\right), \qquad p = \mathcal{O}(1),\\
    &K = \mathcal{O}\left((n_{c_W} + n_{c_U})^2\log(n_{c_W} + n_{c_U})+(n_{c_W} + n_{c_U})^2[\log(\varepsilon^{-1}) + \log(2)]\right), \\ 
    &\kappa=\mathcal{O}((n_{c_W} + n_{c_U})^{(n_{c_W} + n_{c_U})/2+1}\varepsilon^{-(n_{c_W} + n_{c_U} +1)} 2^{n_{c_W} + n_{c_U} +1}),\qquad \, R=1
    \end{align*}
    where the constants hidden in $\mathcal{O}$ depend on $\beta_W, \beta_U$, $C_{\mathrm{prod}}$ and $L_f$,
    \end{itemize} 
    Then, there exists networks
     $\{b_k\}_{k=1}^{H^{n_{c_W} + n_{c_U}}} \subset  \cF_{\rm NN}(n_{c_W} + n_{c_U}, 1, L, p, K, \kappa,R)$ and functions $\{\alpha_k\}_{k=1}^{H^{n_{c_W} + n_{c_U}}} \subset \cB_{\beta_W,\Vert \cdot \Vert_{\Lp{\infty}},\Omega_W}(0)$, $\{u_k\}_{k=1}^{H^{n_{c_W} + n_{c_U}}} \subset \cB_{\beta_U,\Vert \cdot \Vert_{\Lp{\infty}},\Omega_U}(0)$ such that 
     \begin{align}
       \sup_{\alpha\in W} \sup_{u\in U}\left|f(\alpha,u)-\sum_{k=1}^{H^{n_{c_W}+n_{c_U}}} f(\alpha_k,u_k) b_k(\bm{\alpha},\bm{u})\right|\leq \varepsilon, \notag
    \end{align}
    where $\bm{\alpha}=\frac{\max\{\beta_W,\beta_U\}}{\beta_W}(\alpha(a_1), \alpha(a_2),\dots,\alpha(a_{n_{c_W}}))^\top$, $\bm{u}=\frac{\max\{\beta_W,\beta_U\}}{\beta_U}(u(c_1), u(c_2),\dots,u(c_{n_{c_U}}))^\top$. We also have that $0 \leq b_k \leq 1$ for all $1 \leq k \leq H^{n_{c_W} + n_{c_W}}$.
\end{proposition}

\begin{proof}

Let $\delta_W,\delta_U>0$, and let
\(
\mathcal C_W=\{\mathcal B_{\delta_W}(a_i)\}_{i=1}^{n_{c_W}},
\mathcal C_U=\{\mathcal B_{\delta_U}(c_m)\}_{m=1}^{n_{c_U}}
\)
be finite covers of $\Omega_W$ and $\Omega_U$ by Euclidean balls, respectively. By \cite[Lemma 1]{liu2024neuralscalinglawsdeep}, there exist partitions of unity
\(
\{\rho_i\}_{i=1}^{n_{c_W}}
\subset \Ck{\infty}(\Omega_W)\)
and \(\{\omega_m\}_{m=1}^{n_{c_U}}
\subset \Ck{\infty}(\Omega_U)
\)
subordinate to $\mathcal C_W$ and $\mathcal C_U$, respectively.

We define the discrete-to-continuum liftings
\[
I_{\mathcal C_W}:[-\beta_W,\beta_W]^{n_{c_W}}\to \Ck{\infty}(\Omega_W),
\qquad
I_{\mathcal C_U}:[-\beta_U,\beta_U]^{n_{c_U}}\to \Ck{\infty}(\Omega_U)
\]
by
\[
I_{\mathcal C_W}[z](x)=\sum_{i=1}^{n_{c_W}} [z]_i\,\rho_i(x),
\qquad
I_{\mathcal C_U}[w](x)=\sum_{m=1}^{n_{c_U}} [w]_m\,\omega_m(x),
\]
and the continuum-to-discrete projections
\[
P_{\mathcal C_W}:\Ck{0}(\Omega_W)\to[-\beta_W,\beta_W]^{n_{c_W}},
\qquad
P_{\mathcal C_U}:\Ck{0}(\Omega_U)\to[-\beta_U,\beta_U]^{n_{c_U}}
\]
by
\[
P_{\mathcal C_W}(\alpha)=\bigl(\alpha(a_1),\dots,\alpha(a_{n_{c_W}})\bigr)^\top,
\qquad
P_{\mathcal C_U}(u)=\bigl(u(c_1),\dots,u(c_{n_{c_U}})\bigr)^\top.
\]

Assume now that
\(
f:W\times U\to \mathbb R
\)
is Lipschitz with respect to the product norm
\(
\|(\alpha,u)\|_{W\times U}
\)
with Lipschitz constant $L_f$.

We first estimate the discretization error. For $\alpha\in W$ and $x\in\Omega_W$, we have
\begin{align}
|\alpha(x)-I_{\mathcal C_W}[P_{\mathcal C_W}(\alpha)](x)|
&=
\left|\sum_{i=1}^{n_{c_W}} \bigl(\alpha(x)-\alpha(a_i)\bigr)\rho_i(x)\right| \notag \\
&\le
\sum_{i:\|x-a_i\|_2\le \delta_W}
|\alpha(x)-\alpha(a_i)|\,|\rho_i(x)| \notag \\
&\le L_W\delta_W \sum_{i=1}^{n_{c_W}} \rho_i(x) \notag \\
&= L_W \delta_W \notag
\end{align}
which implies
\(
\|\alpha-I_{\mathcal C_W}[P_{\mathcal C_W}(\alpha)]\|_{\Lp{\infty}(\Omega_W)}
\le L_W\delta_W.
\)
Similarly, for every $u\in U$, we deduce that
\(
\|u-I_{\mathcal C_U}[P_{\mathcal C_U}(u)]\|_{\Lp{\infty}(\Omega_U)}
\le L_U\delta_U.
\)
Therefore, by the latter two estimates and \eqref{eq:prop:productNorm},
\[
\left\|
\bigl(\alpha,u\bigr)
-
\bigl(I_{\mathcal C_W}[P_{\mathcal C_W}(\alpha)],\, I_{\mathcal C_U}[P_{\mathcal C_U}(u)]\bigr)
\right\|_{W\times U}
\le
C_{\textrm{prod}} \max\{L_W\delta_W,\;L_U\delta_U\}.
\]
Finally, using the Lipschitz continuity of $f$, we obtain
\begin{align}
\bigl|f(\alpha,u)-f(I_{\mathcal C_W}[P_{\mathcal C_W}(\alpha)],\,I_{\mathcal C_U}[P_{\mathcal C_U}(u)])\bigr|
&\le
L_f
\left\|
\bigl(\alpha,u\bigr)
-
\bigl(I_{\mathcal C_W}[P_{\mathcal C_W}(\alpha)],\, I_{\mathcal C_U}[P_{\mathcal C_U}(u)]\bigr)
\right\|_{W\times U}
\notag\\
&\le
CL_f\max\{L_W\delta_W,\;L_U\delta_U\},
\notag 
\end{align}
and choosing
\(
\delta_W=\frac{\varepsilon}{2C_{\textrm{prod}}L_fL_W},\) \( \delta_U=\frac{\varepsilon}{2C_{\textrm{prod}}L_fL_U}
\)
therefore yields
\begin{equation} \label{eq:productApprox:LipschitzBound}
\bigl|f(\alpha,u)-f(I_{\mathcal C_W}[P_{\mathcal C_W}(\alpha)],\,I_{\mathcal C_U}[P_{\mathcal C_U}(u)])\bigr|
\le \frac{\varepsilon}{2}.
\end{equation}

Next, we define
\(
\hat f:[-\beta_W,\beta_W]^{n_{c_W}}\times[-\beta_U,\beta_U]^{n_{c_U}}\to\mathbb R
\)
by
\(
\hat f(z,w):=
f(I_{\mathcal C_W}[z],\,I_{\mathcal C_U}[w]).
\)
This function serves as a finite-dimensional surrogate of the functional $f$. Our strategy is to approximate $\hat f$ using neural-network approximation results in finite dimensions, and then transfer this approximation back to the original functional $f$.
We estimate as follows for $(z_1,w_1),(z_2,w_2) \in [-\beta_W,\beta_W]^{n_{c_W}}\times[-\beta_U,\beta_U]^{n_{c_U}}$:
\begin{align}
|\hat f(z_1,w_1)-\hat f(z_2,w_2)|
&=
|f(I_{\mathcal C_W}[z_1],I_{\mathcal C_U}[w_1])-f(I_{\mathcal C_W}[z_2],I_{\mathcal C_U}[w_2])| \notag \\
&\le
L_f\,
\|(I_{\mathcal C_W}[z_1]-I_{\mathcal C_W}[z_2],\,
I_{\mathcal C_U}[w_1]-I_{\mathcal C_U}[w_2])\|_{W\times U} \notag \\
&\leq C_{\textrm{prod}} L_f \max \{ \Vert I_{\mathcal C_W}[z_1]-I_{\mathcal C_W}[z_2] \Vert_{\Lp{\infty}(\Omega_W)}, \Vert I_{\mathcal C_U}[w_1]-I_{\mathcal C_U}[w_2] \Vert_{\Lp{\infty}(\Omega_U)} \} \label{eq:prop:eq1}
\end{align}
where we Assumption \ref{assumption:Main:assumptions:N2} for \eqref{eq:prop:eq1}. We note that \begin{align}
    \Vert I_{\mathcal C_W}[z_1]-I_{\mathcal C_W}[z_2] \Vert_{\Lp{\infty}(\Omega_W)} &\leq \sup_{x \in \Omega_W} \sum_{i=1}^{n_{c_W}} \vert [z_1]_i - [z_2]_i \vert \rho_i(x) \notag \\
    &\leq \Vert z_1 - z_2 \Vert_{\ell^2(\bbR^{n_{c_W}})}  \sup_{x \in \Omega_W} \sqrt{\sum_{i=1}^{n_{c_W}} \rho_i^2(x)} \notag \\
    &\leq \Vert z_1 - z_2 \Vert_{\ell^2(\bbR^{n_{c_W}})}  \sup_{x \in \Omega_W} \sqrt{\sum_{i=1}^{n_{c_W}} \rho_i(x)} \label{eq:prop:eq2} \\
    &\leq \Vert z_1 - z_2 \Vert_{\ell^2(\bbR^{n_{c_W}})} \label{eq:prop:eq3}
\end{align} 
where we used the fact that $0 \leq \rho_i \leq 1$ for \eqref{eq:prop:eq2}. The same argument can be repeated for $\Vert I_{\mathcal C_U}[w_1]-I_{\mathcal C_U}[w_2] \Vert_{\Lp{\infty}(\Omega_U)}$ so that inserting \eqref{eq:prop:eq3} into \eqref{eq:prop:eq1} yields \begin{align}
    |f(I_{\mathcal C_W}[z_1],I_{\mathcal C_U}[w_1])-f(I_{\mathcal C_W}[z_2],I_{\mathcal C_U}[w_2])| &\leq C_{\textrm{prod}} L_f \max \left\{ \Vert z_1 - z_2 \Vert_{\ell^2(\bbR^{n_{c_W}})}, \Vert w_1 - w_2 \Vert_{\ell^2(\bbR^{n_{c_U}})} \right\} \notag \\
    &\leq C_{\textrm{prod}} L_f \sqrt{\Vert z_1 - z_2 \Vert_{\ell^2(\bbR^{n_{c_W}})}^2 + \Vert w_1 - w_2 \Vert_{\ell^2(\bbR^{n_{c_U}})}^2} \label{eq:prop:eq4} \\
    &= C_{\textrm{prod}} L_f \|(z_1-z_2,w_1-w_2)\|_{\ell^2(\mathbb R^{n_{c_W}+n_{c_U}})} \label{eq:prop:eq5}
\end{align}
where we used $\max\{a,b\} \leq \sqrt{a^2 + b^2}$ for \eqref{eq:prop:eq4}. The latter shows that $\hat f$ is Lipschitz on the compact cube
\(
[-\beta_W,\beta_W]^{n_{c_W}}\times[-\beta_U,\beta_U]^{n_{c_U}}.
\)

With
\(
\beta:=\max\{\beta_W,\beta_U\},
\) 
let
\(
T:[-\beta,\beta]^{n_{c_W}+n_{c_U}}
\to
[-\beta_W,\beta_W]^{n_{c_W}}\times[-\beta_U,\beta_U]^{n_{c_U}},
\)
be defined by
\[
T(\xi,\zeta):=
\left(
\frac{\beta_W}{\beta}\,\xi,\,
\frac{\beta_U}{\beta}\,\zeta
\right),
\qquad
\xi\in\mathbb R^{n_{c_W}},\ \zeta\in\mathbb R^{n_{c_U}}.
\]
This map is $1$-Lipschitz with respect to the $\ell^2(\bbR^{n_{c_W}+ n_{c_U}})$-norm and consequently, the function
\[
\tilde f:[-\beta,\beta]^{n_{c_W}+n_{c_U}}\to\mathbb R,
\qquad
\tilde f(\xi,\zeta):=\hat f(T(\xi,\zeta))
\]
is $C_{\textrm{prod}}L_f$-Lipschitz with respect to the $\ell^2(\bbR^{n_{c_W}+ n_{c_U}})$-norm as a composition of Lipschitz functions by \eqref{eq:prop:eq5}. We conclude that $\tilde{f} \in V(n_{c_W}+ n_{c_U}, \beta, C_{\textrm{prod}} L_f, C_{\tilde{f}})$ for some set of functions $V$ satisfying Assumption \ref{assumption:Main:assumptions:S4} and where $C_{\tilde{f}} > 0$ is a constant depending on $\tilde f$. We can therefore apply \cite[Theorem 5]{liu2024neuralscalinglawsdeep}: specifically, there exists a constant $C$ depending on $\beta_W$, $\beta_U$, $C_{\textrm{prod}}$, $L_f$ so that the following holds. Let $H := 2 C \sqrt{n_{c_W} + n_{c_U}} \eps^{-1}$ and consider the network class $\cF_{\rm NN}(n_{c_W}+n_{c_W},1,L,p,K,\kappa,R)$ whose parameters scale as \begin{align*}
    &L=\mathcal{O}\left((n_{c_W} + n_{c_U})^2\log(n_{c_W} + n_{c_U})+(n_{c_W} + n_{c_U})^2[\log(\varepsilon^{-1}) + \log(2)]\right), \qquad p = \mathcal{O}(1),\\
    &K = \mathcal{O}\left((n_{c_W} + n_{c_U})^2\log(n_{c_W} + n_{c_U})+(n_{c_W} + n_{c_U})^2[\log(\varepsilon^{-1}) + \log(2)]\right), \\ 
    &\kappa=\mathcal{O}((n_{c_W} + n_{c_U})^{(n_{c_W} + n_{c_U})/2+1}\varepsilon^{-(n_{c_W} + n_{c_U} +1)} 2^{n_{c_W} + n_{c_U} +1}),\qquad \, R=1
    \end{align*}
    where the constants hidden in $\mathcal{O}$ depend on $\beta_W, \beta_U$, $C_{\mathrm{prod}}$ and $L_f$. Then, there exists:
\begin{itemize}
    \item networks $\{b_k\}_{k=1}^{H^{n_{c_W} + n_{c_U}}} \subset \cF_{\rm NN}(n_{c_W}+n_{c_U},1,L,p,K,\kappa,R)$
    \item points $\{s_k\}_{k=1}^{H^{n_{c_W} + n_{c_U}}} \subset [-\beta, \beta]^{n_{c_W} + n_{c_U}}$
\end{itemize}
such that
\begin{equation} \label{eq:prop:eq6}
    \sup_{(\xi,\zeta)\in[-\beta,\beta]^{n_{c_W}+n_{c_U}}}
\left|
\tilde f(\xi,\zeta)-\sum_{k=1}^{H^{n_{c_W}+n_{c_U}}}\tilde f(s_k)\,b_k(\xi,\zeta)
\right|
\le \frac{\varepsilon}{2}.
\end{equation}
We also note that \cite[Theorem 5]{liu2024neuralscalinglawsdeep} yields $0 \leq b_k \leq 1$.
For any $(\alpha,u) \in W \times U$, we define the inverse-rescaled coordinates
\[
(\xi_\alpha,\zeta_u):=T^{-1}(P_{\mathcal C_W}(\alpha),P_{\mathcal C_U}(u))
=
\left(
\frac{\beta}{\beta_W}P_{\mathcal C_W}(\alpha),\,
\frac{\beta}{\beta_U}P_{\mathcal C_U}(u)
\right) \in [-\beta,\beta]^{n_{c_W} + n_{c_U}},
\]
since $\Vert \alpha \Vert_{\Lp{\infty}(\Omega_W)} \leq \beta_W$ and $\Vert u \Vert_{\Lp{\infty}(\Omega_U)} \leq \beta_U$. This implies that $\hat{f}(P_{\mathcal C_W}(\alpha),P_{\mathcal C_U}(u))=\tilde f(\zeta_\alpha,\xi_u)$ and therefore, by \eqref{eq:prop:eq6},  
\begin{equation} \label{eq:productApprox:rescaledApprox}
\sup_{\alpha \in W} \sup_{u \in U} \left|
\hat f(P_{\mathcal C_W}(\alpha),P_{\mathcal C_U}(u))
-
\sum_{k=1}^{H^{n_{c_W}+n_{c_U}}}
\tilde f(s_k)\,b_k(\xi_\alpha,\zeta_u)
\right|
\le
\frac{\varepsilon}{2}.
\end{equation}

Combining \eqref{eq:productApprox:LipschitzBound} and \eqref{eq:productApprox:rescaledApprox}, we conclude as follows:
\begin{align*}
&\sup_{\alpha\in W}\sup_{u\in U}
\left|
f(\alpha,u)
-
\sum_{k=1}^{H^{n_{c_W}+n_{c_U}}}
\tilde f(s_k)\,b_k\!\left(
\frac{\beta}{\beta_W}P_{\mathcal C_W}(\alpha),\,
\frac{\beta}{\beta_U}P_{\mathcal C_U}(u)
\right)
\right| \\
&\le
\sup_{\alpha\in W}\sup_{u\in U}
|f(\alpha,u)-\hat f(P_{\mathcal C_W}(\alpha),P_{\mathcal C_U}(u))| \\
&+
\sup_{\alpha\in W}\sup_{u\in U}
\left|
\hat f(P_{\mathcal C_W}(\alpha),P_{\mathcal C_U}(u))
-
\sum_{k=1}^{H^{n_{c_W}+n_{c_U}}}
\tilde f(s_k)\,b_k\!\left(
\frac{\beta}{\beta_W}P_{\mathcal C_W}(\alpha),\,
\frac{\beta}{\beta_U}P_{\mathcal C_U}(u)
\right)
\right| \\
&\le \frac{\varepsilon}{2}+\frac{\varepsilon}{2}
=\varepsilon.
\end{align*}
Finally, for each sampling point 
\(
s_k=(z_k,w_k)\in[-\beta,\beta]^{n_{c_W}}\times[-\beta,\beta]^{n_{c_U}},
\)
we have \[
\tilde f(s_k) = \hat f\l  \frac{\beta_W}{\beta} z_k , \frac{\beta_U}{\beta} w_k \r  = f\l \frac{\beta_W}{\beta} I_{\mathcal C_W}[z_k] , \frac{\beta_U}{\beta} I_{\mathcal C_U}[w_k]  \r =: f(\alpha_k,u_k)
\]
where $\alpha_k \in \cB_{\beta_W,\Vert \cdot \Vert_{\Lp{\infty}},\Omega_W}(0)$ and $u_k \in \cB_{\beta_U,\Vert \cdot \Vert_{\Lp{\infty}},\Omega_U}(0)$, 
which yields the claimed representation.

\end{proof}

\begin{remark}[Uniform Functional Approximation in Product Spaces] \label{rem:uniformProduct}
    We can extend Proposition \ref{thm:back:functionalApproximationLinftyProductSpace} to families of functionals. The same argument as in \cite[Remark 3.7]{weihs2025MOL} shows that the construction is uniform over any family \(\{f_j\}_{j\in J}\) of functionals defined using the same spaces \(W\) and \(U\), provided the constants entering the proof are bounded uniformly in \(j\). More precisely, let us assume that
\[
\sup_{j\in J} L_{f_j}<\infty
\qquad\text{and}\qquad
\sup_{j\in J} C_{\hat f_j} = \sup_{j \in J} \sup_{\alpha \in W} \sup_{u \in U} \vert f_j(\alpha,u) \vert <\infty,
\]
where \(L_{f_j}\) is the Lipschitz constant of \(f_j\), and \(C_{\hat f_j}\) denotes the uniform bound placing the associated finite-dimensional surrogate \(\hat f_j\) in the same approximation class as in the proof of Proposition \ref{thm:back:functionalApproximationLinftyProductSpace}. Then there exist sampling pairs
\(
\{(\alpha_k,u_k)\}_{k=1}^{H^{n_{c_W}+n_{c_U}}}
\)
and neural networks
\(
\{b_k\}_{k=1}^{H^{n_{c_W}+n_{c_U}}}
\)
in the same \(\varepsilon\)-dependent network class as in Proposition \ref{thm:back:functionalApproximationLinftyProductSpace}, such that
\[
\sup_{j\in J}\sup_{\alpha\in W}\sup_{u\in U}
\left|
f_j(\alpha,u)
-
\sum_{k=1}^{H^{n_{c_W}+n_{c_U}}}
f_j(\alpha_k,u_k)\,b_k(\bm\alpha,\bm u)
\right|
\le \varepsilon.
\]
In particular, the conclusion always applies to any finite family of functionals.
\end{remark}

\begin{proof}[Proof of Proposition \ref{prop:multipleProductSpace}]
By Assumption, for all $(\alpha,u) \in W \times U$, the functions $x \mapsto G[\alpha][u](x)$ are in $V$. Consequently, we can apply \cite[Theorem 5]{liu2024neuralscalinglawsdeep}. Specifically, there exists a constant $C$ depending on $\gamma_V$, $L_V$ so that the following holds. For $\eps_0 > 0$, let $N := C \sqrt{d_V} \eps_0^{-1}$ and consider the network class $\cF_1 := \cF_{\rm NN}(d_V,1,L_1,p_1,K_1,\kappa_1,R_1)$ whose parameters scale as \begin{align*}
&L_1 = \mathcal{O}\left(d_V^2\log d_V+d_V^2\log(\varepsilon_0^{-1})\right),\quad p_1 = \mathcal{O}(1),\quad K_1 = \mathcal{O}\left(d_V^2\log d_V+d_V^2\log(\varepsilon_0^{-1})\right),\\
&\kappa_1=\mathcal{O}(d_V^{d_V/2+1}\varepsilon_0^{-(d_V+1)}),\qquad \qquad \quad R_1=1
    \end{align*}
where the constants hidden in $\mathcal{O}$ depend on $\gamma_V$ and $L_V$. Then, there exists
\begin{itemize}
    \item networks $\{\tau_\ell\}_{k=1}^{N^{d_V}} \subset \cF_1$
    \item points $\{v_\ell\}_{\ell=1}^{N^{d_V}} \subset \Omega_V$
\end{itemize}
such that
\begin{equation} \label{eq:prop:operator:eq1}
    \sup_{x \in \Omega_V}
\left|
G[\alpha][u](x)-\sum_{\ell=1}^{N^{d_V}} G[\alpha][u](v_\ell)\,\tau_\ell(x)
\right|
\le \eps_0.
\end{equation}
Notably, we also recall from the proof of \cite[Theorem 5]{liu2024neuralscalinglawsdeep} that $0\leq \tau_\ell(x) \leq 1$ (where the last inequality follows by definition of the network class $\cF_1$ with $R_2=1$).

Next, we consider the $N^{d_V}$ functionals $f_\ell: \cB_{\beta_W,\Vert \cdot \Vert_{\Lp{\infty}},\Omega_W}(0) \times \cB_{\beta_U,\Vert \cdot \Vert_{\Lp{\infty}},\Omega_U}(0) \mapsto \bbR$ defined by $f_\ell(\alpha,u) = G[\alpha][u](v_\ell)$. For $(\alpha_1,u_1), (\alpha_2,u_2) \in \cB_{\beta_W,\Vert \cdot \Vert_{\Lp{\infty}},\Omega_W}(0) \times \cB_{\beta_U,\Vert \cdot \Vert_{\Lp{\infty}},\Omega_U}(0)$, we estimate as follows: \begin{align}
    \vert f_\ell(\alpha_1,u_1) - f_\ell(\alpha_2,u_2) \vert &= \vert G[\alpha_1][u_1](v_\ell) - G[\alpha_2][u_2](v_\ell) \vert \notag \\
    &\leq \vert G[\alpha_1][u_1](v_\ell) - G[\alpha_1][u_2](v_\ell) \vert + \vert G[\alpha_1][u_2](v_\ell) - G[\alpha_2][u_2](v_\ell) \vert \notag \\
    &\leq \Vert G[\alpha_1][u_1] - G[\alpha_1][u_2] \Vert_{\Lp{\infty}(\Omega_V)} + \Vert G[\alpha_1] - G[\alpha_2] \Vert_{\Lp{\infty}(\{ u:\Omega_U \mapsto \bbR \spaceBar \Vert u \Vert_{\Lp{\infty}} \leq \beta_U \} \times \Omega_V)} \notag \\
    &\leq L_{\mathcal{G}} \Vert u_1 - u_2 \Vert_{\Lp{r_\mathcal{G}}(\Omega_U)} + L_G \Vert \alpha_1 - \alpha_2 \Vert_{\Lp{r_G}(\Omega_W)} \label{eq:prop:operator:eq2} \\
    &\leq \max\{L_{\mathcal{G}}\vert \Omega_U \vert^{1/r_\mathcal{G}},L_{G}\vert \Omega_W \vert^{1/r_G}\} \max\{ \Vert u_1 - u_2 \Vert_{\Lp{\infty}(\Omega_U)},\Vert \alpha_1 - \alpha_2 \Vert_{\Lp{\infty}(\Omega_W)}\} \label{eq:prop:operator:eq3}
\end{align}
where we used Assumptions \ref{assumption:Main:assumptions:O1} and \ref{assumption:Main:assumptions:O2} for \eqref{eq:prop:operator:eq2}. We note that the condition \eqref{eq:prop:operator:eq3} (instead of Lipschitz continuity with respect to the product norm \eqref{eq:prop:productNorm}---this implies that we replace $L_f C_{\textrm{prod}}$ by $C_{\textrm{prod}}$) is sufficient in the proof of Proposition \ref{thm:back:functionalApproximationLinftyProductSpace}. We therefore apply the latter (in conjunction with Remark \ref{rem:uniformProduct}). Specifically, there exist constants \begin{itemize}
    \item $C'$ depending on $\beta_W$, $\beta_U$, $L_\mathcal{G}$, $\gamma_U$, $r_\mathcal{G}$, $L_G$, $\gamma_W$, $r_G$
    \item  $C_{W}$ depending on $L_\mathcal{G}$, $\gamma_U$, $r_\mathcal{G}$, $L_G$, $\gamma_W$, $r_G$, $L_W$
    \item $C_U$ depending on $L_\mathcal{G}$, $\gamma_U$, $r_\mathcal{G}$, $L_G$, $\gamma_W$, $r_G$, $L_U$
\end{itemize}
such that the following holds. For any $\eps_1 > 0$, \begin{itemize}
        \item let $\delta_W=C_{W}\varepsilon_1$ and let $\{a_i\}_{i=1}^{n_{c_W}}\subset \Omega_W$ be points so that $\{\mathcal{B}_{\delta_W}(a_i) \}_{ i  = 1}^{n_{c_W}}$ is a cover of $\Omega_W$ for some $n_{c_W}$;
        \item let $\delta_U=C_{U}\varepsilon_1$ and let $\{c_i\}_{i=1}^{n_{c_U}}\subset \Omega_U$ be points so that $\{\mathcal{B}_{\delta_U}(c_i) \}_{ i  = 1}^{n_{c_U}}$ is a cover of $\Omega_U$ for some $n_{c_U}$;
        \item let $H = 2C' \sqrt{n_{c_W} + n_{c_U}} \eps_1^{-1}$ and consider the network class $\cF_2 := \cF_{\rm NN}(n_{c_W} + n_{c_U}, 1 , L_2, p_2, K_2, \kappa_2, R_2)$ 
     with parameters scaling as
     \begin{align*}
    &L_2=\mathcal{O}\left((n_{c_W} + n_{c_U})^2\log(n_{c_W} + n_{c_U})+(n_{c_W} + n_{c_U})^2[\log(\varepsilon_1^{-1}) + \log(2)]\right), \qquad p_2 = \mathcal{O}(1),\\
    &K_2 = \mathcal{O}\left((n_{c_W} + n_{c_U})^2\log(n_{c_W} + n_{c_U})+(n_{c_W} + n_{c_U})^2[\log(\varepsilon_1^{-1}) + \log(2)]\right), \\ 
    &\kappa_2=\mathcal{O}((n_{c_W} + n_{c_U})^{(n_{c_W} + n_{c_U})/2+1}\varepsilon_1^{-(n_{c_W} + n_{c_U} +1)} 2^{n_{c_W} + n_{c_U} +1}),\qquad \, R_2=1
    \end{align*}
    where the constants hidden in $\mathcal{O}$ depend on $\beta_W, \beta_U$, $L_\mathcal{G}$, $\gamma_U$, $r_\mathcal{G}$, $L_G$, $\gamma_W$, $r_G$.
    \end{itemize} 
    Then, there exists networks
     $\{b_k\}_{k=1}^{H^{n_{c_W} + n_{c_U}}} \subset  \cF_2$ and functions $$\{\alpha_k\}_{k=1}^{H^{n_{c_W} + n_{c_U}}} \subset \cB_{\beta_W,\Vert \cdot \Vert_{\Lp{\infty}},\Omega_W}(0), \, \{u_k\}_{k=1}^{H^{n_{c_W} + n_{c_U}}} \subset \cB_{\beta_U,\Vert \cdot \Vert_{\Lp{\infty}},\Omega_U}(0)$$ such that 
     \begin{align}
       \sup_{1 \leq \ell\leq N^{d_V}} \sup_{\alpha\in W} \sup_{u\in U}\left|f_\ell(\alpha,u)-\sum_{k=1}^{H^{n_{c_W}+n_{c_U}}} f_\ell(\alpha_k,u_k) b_k(\bm{\alpha},\bm{u})\right|\leq \varepsilon_1,
        \label{eq:prop:operator:eq4}
    \end{align}
    where $\bm{\alpha}=\frac{\max\{\beta_W,\beta_U\}}{\beta_W}(\alpha(a_1), \alpha(a_2),\dots,\alpha(a_{n_{c_W}}))^\top$, $\bm{u}=\frac{\max\{\beta_W,\beta_U\}}{\beta_U}(u(c_1), u(c_2),\dots,u(c_{n_{c_U}}))^\top$.

We continue by estimating as follows: \begin{align}
    &\sup_{\alpha \in W} \sup_{u \in U} \sup_{x \in \Omega_V} \left\vert G[\alpha][u](x) - \sum_{\ell=1}^{N^{d_V}} \sum_{k=1}^{H^{n_{c_W}+n_{c_U}}} G[\alpha_k][u_k](v_\ell) b_k(\bm{\alpha},\bm{u}) \tau_\ell(x) \right\vert \notag \\
    &\leq \sup_{\alpha \in W} \sup_{u \in U} \sup_{x \in \Omega_V} \left\vert G[\alpha][u](x) - \sum_{\ell=1}^{N^{d_V}} G[\alpha][u](v_\ell) \tau_\ell(x) \right\vert \notag \\
    &+ \sup_{\alpha \in W} \sup_{u \in U} \sup_{x \in \Omega_V} \left\vert \sum_{\ell=1}^{N^{d_V}} \ls G[\alpha][u](v_\ell) \tau_\ell(x) - \sum_{k=1}^{H^{n_{c_W}+n_{c_U}}} G[\alpha_k][u_k](v_\ell) b_k(\bm{\alpha},\bm{u}) \tau_\ell(x) \rs \right\vert \notag \\
    &\leq \eps_0 + \sup_{\alpha \in W} \sup_{u \in U} \sup_{x \in \Omega_V}  \sum_{\ell=1}^{N^{d_V}}  \tau_\ell(x) \left\vert G[\alpha][u](v_\ell) - \sum_{k=1}^{H^{n_{c_W}+n_{c_U}}} G[\alpha_k][u_k](v_\ell) b_k(\bm{\alpha},\bm{u}) \right\vert \label{eq:prop:operator:eq5} \\
    &\leq \eps_0 + \eps_1 \sup_{x\in \Omega_V} \sum_{\ell=1}^{N^{d_V}} \tau_\ell(x) \label{eq:prop:operator:eq6}
\end{align}
where we used \eqref{eq:prop:operator:eq1} and the fact that $0 \leq \tau_\ell(x) \leq 1$ for \eqref{eq:prop:operator:eq5} and \eqref{eq:prop:operator:eq4} for \eqref{eq:prop:operator:eq6}. For the last term, we note that for any $0 < \eta < \beta_V$, the constant function equal to $\eta$ is included in $V$. In particular, by \eqref{eq:prop:operator:eq1}, we have \begin{align}
    \eps_0 \geq \sup_{x \in \Omega_V}
\left|
\eta-\eta\sum_{\ell=1}^{N^{d_V}} \tau_\ell(x)
\right| = \eta \sup_{x \in \Omega_V}
\left|
1-\sum_{\ell=1}^{N^{d_V}} \tau_\ell(x)
\right| \notag
\end{align}
which implies that \begin{equation} \label{eq:prop:operator:eq7}
    \sum_{\ell=1}^{N^{d_V}}\tau_\ell(x) \leq 1 + \frac{\eps_0}{\eta}
\end{equation}
for all $x \in \Omega_V$.
Setting $\eps_0 = \frac{\eps}{2}$, $\eps_1 = \frac{\eps}{2(1 + \frac{\eps}{2\eta})}$ and inserting \eqref{eq:prop:operator:eq7} into \eqref{eq:prop:operator:eq6} yields: \begin{align}
    \sup_{\alpha \in W} \sup_{u \in U} \sup_{x \in \Omega_V} \left\vert G[\alpha][u](x) - \sum_{\ell=1}^{N^{d_V}} \sum_{k=1}^{H^{n_{c_W}+n_{c_U}}} G[\alpha_k][u_k](v_\ell) b_k(\bm{\alpha},\bm{u}) \tau_\ell(x) \right\vert &\leq \eps. \notag
\end{align}
The final network scalings for $\cF_1$ are:
\begin{align*}
&L_1 = \mathcal{O}\left(d_V^2\log d_V+d_V^2(\log(\varepsilon^{-1})+\log(2))\right),\qquad p_1 = \mathcal{O}(1),\\
&K_1 = \mathcal{O}\left(d_V^2\log d_V+d_V^2(\log(\varepsilon^{-1})+\log(2))\right),\qquad \kappa_1=\mathcal{O}(d_V^{d_V/2+1}\varepsilon^{-(d_V+1)}2^{d_V+1}),\\
&R_1=1, \qquad N= 2C\sqrt{d_V}\eps^{-1}.
    \end{align*}
Noting that $\eps_1 = \frac{\eps}{2(1 + \frac{\eps}{2\eta})} = \frac{\eps \eta}{2\eta + \eps} \asymp \frac{\eps}{2}$ for sufficiently small $\eps$, the final network scalings for $\cF_2$ are:
\begin{align*}
    &L_2=\mathcal{O}\left((n_{c_W} + n_{c_U})^2\log(n_{c_W} + n_{c_U})+(n_{c_W} + n_{c_U})^2[\log(\varepsilon^{-1}) + 2\log(2)]\right), \qquad p_2 = \mathcal{O}(1),\\
    &K_2 = \mathcal{O}\left((n_{c_W} + n_{c_U})^2\log(n_{c_W} + n_{c_U})+(n_{c_W} + n_{c_U})^2[\log(\varepsilon^{-1}) + 2\log(2)]\right), \\ 
    &\kappa_2=\mathcal{O}((n_{c_W} + n_{c_U})^{(n_{c_W} + n_{c_U})/2+1}\varepsilon^{-(n_{c_W} + n_{c_U} +1)} 2^{2(n_{c_W} + n_{c_U} +1)}),\, R_2=1, \, H = 4C' \sqrt{n_{c_W} + n_{c_U}} \eps^{-1}.
\end{align*}
\end{proof}

\begin{proof}[Proof of Theorem \ref{thm:minimaxDeepONet}]
    \begin{enumerate}
        \item From the proof of Corollary \ref{cor:codMNO}, we know that $W$ and $U$ are compact sets of $\Lp{r_G}(\Omega_W)$ and $\Lp{r_G}(\Omega_U)$, respectively. With $\eta > \min\left\{1 + \frac{1}{d_W},1+\frac{1}{d_U}\right\}$, by Lemma \ref{lem:boundedLipCube}, at least one of $W$ and $U$ will contain a hypercube $Q_\eta$. Without loss of generality, we assume that $W$ does so. By Lemma \ref{lem:productCubeFromSingleCube}, we can embed $Q_\eta$ into the product space, i.e. we obtain that $W \times \{0\} \subset W \times U$ will contain a hypercube $\tilde{Q}_\eta$. We note that $W \times U$ is also compact as product of compact spaces.

        Let $r \in \bbN$ and $\delta > 0$. By the above, all the conditions to apply \cite[Corollary 2.12]{lanthalerStuart} on the product Banach space $\Lp{r_G}(\Omega_W) \times \Lp{r_G}(\Omega_U)$ are satisfied. In particular, we deduce the existence of $\overline{\eps}$ and $c > 0$ such that for $0< \eps \leq \overline{\eps}$ and any operator of neural network type $\nn_\eps$ satisfying \eqref{eq:thm:minimax2:approximation}, we have $\cC(\nn_\eps) \geq \exp(c\eps^{-1/(\eta + 1 + \delta)r})$. 

        We conclude by noting that any network of the form \eqref{eq:separable2} is an operator of neural network type on the product space $\Lp{r_G}(\Omega_W) \times \Lp{r_G}(\Omega_U)$ by Lemma \ref{lem:complexityConcatenated}. The latter also implies that \[
        \Vert \Theta \Vert_0 + HK_2 + NK_1 \gtrsim \cC(\nn_\eps) \geq \exp(c\eps^{-1/(\eta + 1 + \delta)r}).
        \]
        
 \item From part 1, we may take $r=1$ to obtain $G:\Lp{r_G}(\Omega_W) \times \Lp{r_G}(\Omega_U) \to V$ which is Frechet differentiable on $\Lp{r_G}(\Omega_W) \times \Lp{r_G}(\Omega_U)$. Specifically, from the proof of \cite[Corollary 2.12]{lanthalerStuart}, we know that 
\[
G[\alpha][u](x)=F(\alpha,u)\phi(x),
\]
where $F:\Lp{r_G}(\Omega_W)\times \Lp{r_G}(\Omega_U)\to \mathbb R$ is the Frechet differentiable functional provided by \cite[Theorem 2.11]{lanthalerStuart}, and $\phi\in V$ is a fixed nontrivial function.

Next, the proof of \cite[Lemma A.7]{lanthalerStuart} shows that
\[
\sup_{\alpha\in \Lp{r_G}(\Omega_W)} \sup_{u\in \Lp{r_G}(\Omega_U)} \|DF(\alpha,u)\|_{(\Lp{r_G}(\Omega_W)\times \Lp{r_G}(\Omega_U))^*}<\infty.
\]
Hence, $F$ is Lipschitz on $\Lp{r_G}(\Omega_W) \times \Lp{r_G}(\Omega_U)$ with Lipschitz constant $\textrm{Lip}(F)$.
It remains to deduce the two required Lipschitz properties. First, for fixed $\alpha$ and any $u_1,u_2$,
we have \begin{align}
    \Vert G[\alpha][u_1] - G[\alpha][u_2]\Vert_{\Lp{\infty}(\Omega_V)} &= \Vert F(\alpha,u_1)\phi-F(\alpha,u_2)\phi \Vert_{\Lp{\infty}(\Omega_V)} \notag \\
    &\leq \mathrm{Lip}(F)\Vert \phi \Vert_{\Lp{\infty}(\Omega_V)} \Vert (\alpha,u_1) - (\alpha,u_2) \Vert_{\Lp{r_G}(\Omega_W) \times \Lp{r_G}(\Omega_U)} \notag \\
    &\leq C_{\textrm{prod}} \mathrm{Lip}(F) \beta_V \Vert u_1 - u_2 \Vert_{\Lp{r_G}(\Omega_U)} \label{eq:thm:minimax2:norm}
\end{align}
where we used Assumption \ref{assumption:Main:assumptions:N2} for \eqref{eq:thm:minimax2:norm}.
We can therefore define $L_{\mathcal G} = C_{\textrm{prod}} \textrm{Lip}(F) \beta_V >0 $ and $r_\mathcal{G} = r_G$. Second, for $\alpha_1,\alpha_2$ and any $u$, we estimate as follows:
\begin{align}
\|G[\alpha_1]-G[\alpha_2]\|_{\Lp{\infty}(\{u:\|u\|_{\Lp{\infty}(\Omega_U)}\le \beta_U\}\times \Omega_V)}
&=
\sup_{u \in U} |F(\alpha_1,u)-F(\alpha_2,u)|\,\|\phi\|_{\Lp{\infty}(\Omega_V)} \notag \\
&\le
\Lip(F) \beta_V \|(\alpha_1,u)-(\alpha_2,u)\|_{\Lp{r_G}(\Omega_W)\times \Lp{r_G}(\Omega_U)} \notag \\
&\le
\Lip(F) C_{\textrm{prod}} \beta_V \|\alpha_1-\alpha_2\|_{\Lp{r_G}(\Omega_W)} \label{eq:thm:minimax2:norm2}\\
&=: L_G \|\alpha_1-\alpha_2\|_{\Lp{r_G}(\Omega_W)} \notag
\end{align}
where we used Assumption \ref{assumption:Main:assumptions:N2} for \eqref{eq:thm:minimax2:norm2}. This concludes the proof. 
\item The lower bound in \eqref{eq:thm:minimax:minimax} is given by combining parts 1 and 2 of the theorem. The upper bound is a direct consequence of Proposition \ref{prop:multipleProductSpace} and Remark \ref{rem:count}.
    \end{enumerate}
\end{proof}

\end{document}